\documentclass{article}
\usepackage{ijcai26}

\usepackage{times}
\usepackage{soul}
\usepackage{url}
\usepackage[dvipsnames]{xcolor}
\definecolor{darkblue}{rgb}{0.0, 0.0, 0.55}
\usepackage[colorlinks,citecolor=darkblue]{hyperref}
\usepackage[utf8]{inputenc}
\usepackage[small]{caption}
\usepackage{graphicx}
\usepackage{amsmath}
\usepackage{amsthm}
\usepackage{booktabs}
\usepackage{algorithm}
\usepackage{algorithmic}
\usepackage[switch]{lineno}

\usepackage{tabularx}
\usepackage{makecell}
\usepackage{hyperref}   
\usepackage{cleveref}   \usepackage{array}

\newtheorem{theorem}{Theorem}

\title{LoCO: Low-rank Compositional Rotation Fine-tuning}

\author{
An Nguyen$^1$
\and
Jaesik Choi$^{2,3}$\and
Anh Tong$^{1}$\\
\affiliations
$^1$Korea University, 
$^2$KAIST,
$^3$INEEJI 
\emails
nguyenan@korea.ac.kr,
jaesik.choi@kaist.ac.kr,
anhtong12@korea.ac.kr
}

\usepackage{tikz}
\usepackage{tikz-3dplot}
\usetikzlibrary{arrows.meta,calc}
\usetikzlibrary{math}
\usetikzlibrary{angles, quotes}
\usetikzlibrary{3d,calc,perspective}

\usepackage{booktabs, multirow}

\usepackage{amsmath}      \usepackage{amssymb}      \usepackage{amsbsy}       \usepackage{amsfonts}

\newcommand{\vect}[1]{{\mathbf{#1}}}
\newcommand{\mat}[1]{\mathbf{#1}}

\newcommand{\vI}{\mat{I}}
\newcommand{\vX}{\mat{X}}
\newcommand{\vY}{\mat{Y}}

\newcommand{\vA}{\vect{A}}

\newcommand{\vR}{\mat{R}}

\newcommand{\vW}{\mat{W}}

\newcommand{\vV}{\mat{V}}
\newcommand{\vU}{\mat{U}}

\newcommand{\vx}{\vect{x}}

\newcommand{\vz}{\vect{z}}

\def\RR{\mathbb{R}}

\newcommand{\normx}[2]{\left\|#1\right\|_{#2}}

\DeclareMathOperator{\SO}{SO}
\DeclareMathOperator{\sso}{\mathfrak{so}}

\def\rmA{{\mathbf{A}}}
\def\rmB{{\mathbf{B}}}
\def\rmR{{\mathbf{R}}}
\def\rmI{{\mathbf{I}}}
\def\rmS{{\mathbf{S}}}
\def\rmX{{\mathbf{X}}}
\def\rmY{{\mathbf{Y}}}
\def\rmZ{{\mathbf{Z}}}

\crefname{appendix}{Appendix}{Appendices}
\Crefname{appendix}{Appendix}{Appendices} \usepackage[normalem]{ulem}
\usepackage{multirow}
\usepackage{float}
\usepackage{siunitx}
 \newif\ifarxiv
\arxivtrue 
\begin{document}

\maketitle

\begin{abstract}
Parameter-efficient fine-tuning (PEFT) has emerged as an critical technique for adapting large-scale foundation models across natural language processing and computer vision. While existing methods such as low-rank adaptations achieve parameter efficiency via low-rank weight updates, they are limited in their ability to preserve the geometric structure of pretrained representations. We introduce Low-rank Compositional Orthogonal fine-tuning (LoCO), a novel PEFT method that constructs orthogonal transformations through low-rank skew-symmetric matrices and compositional rotation chains. We propose an approximation scheme that enables fully parallel computation of compositional rotations, making the approach practical for high-dimensional feature spaces. Our method maintains low computational complexity while maintaining orthogonality with controlled approximation error. We validate LoCO across diverse domains, including diffusion transformer fine-tuning, vision transformer adaptation, and language model adaptation. Our method demonstrates superior or competitive performance compared to both existing orthogonal and non-orthogonal methods.
\end{abstract} \section{Introduction}
\label{sec:intro}

Large-scale pretrained foundation models~\cite{vaswani2017attention} have revolutionized computer vision and natural language processing, achieving impressive performance across diverse downstream tasks. Models such as large language models~\cite{brown2020language}, vision transformers~\cite{vit} and diffusion transformers~\cite{diffusion_models,peebles2023scalable} contain billions of parameters, encoding rich representations learned from massive datasets. Adapting these models to specific applications through full fine-tuning is often impractical: it demands substantial computational resources, risks catastrophic forgetting of pretrained knowledge, and requires storing separate model copies for each task. Parameter-efficient fine-tuning (PEFT) methods~\cite{adapter,hu2022lora} have emerged as an effective solution for task adaptation while training only a small fraction of parameters.

Orthogonal fine-tuning methods~\cite{qiu2023controlling,liu2024boft,yuan2024bridging} have attracted attention for their geometric properties. Unlike additive methods such as LoRA~\cite{hu2022lora} that injects low-rank adaptation to weight matrices, orthogonal methods apply multiplicative transformations, which are constrained to the special orthogonal group. This constraint offers a theoretical advantage in that orthogonal transformations preserve the pairwise angular relationships among neuron representations and maintain the hyperspherical energy structure of pretrained features~\cite{qiu2023controlling}. Such properties are particularly valuable when semantic information is encoded in angular relationships rather than magnitudes, as observed in contrastive learning~\cite{chen2020angular} and face recognition~\cite{liu2017deep}.

Despite these properties, existing orthogonal fine-tuning methods face a fundamental computational bottleneck. Constructing and applying a general orthogonal matrix requires cubic time complexity for matrix inversion or matrix exponential computation, which becomes prohibitively expensive for high-dimensional feature spaces typical in modern architectures. Prior work addresses this challenge through structural constraints. For example, OFT~\cite{qiu2023controlling} uses block-diagonal matrices that restrict rotations to independent subspaces, while BOFT~\cite{liu2024boft} uses butterfly factorizations for hierarchical decomposition. However, these approaches impose specific sparsity patterns that may limit expressivity, and butterfly structures may suffer from non-contiguous memory access patterns that underutilize hardware resources.

In this paper, we propose \textbf{Lo}w-rank \textbf{C}ompositional \textbf{O}rthogonal Fine-tuning (\textbf{LoCO}), a novel approach that constructs orthogonal transformations through \emph{low-rank} skew-symmetric matrices. Our key observation is that the tangent space to the special orthogonal group, which is the space of skew-symmetric matrices, admits a low-rank parameterization. By exploiting this structure with the Sherman--Morrison--Woodbury identity~\cite{max1950inverting}, we reduce the computational complexity of orthogonal fine-tuning, allowing adaptation in high-dimensional spaces. Additionally, we leverage chains of multiple low-rank rotations to enhance the expressiveness of orthogonal transformations. Rather than computing these rotations sequentially, we employ a first-order approximation that enables parallel computation across all rotation components. Our analysis confirms that this approximation introduces negligible error given the small-magnitude adaptations characteristic of fine-tuning. A practical advantage of LoCO is a test-time temperature scaling mechanism for controllable adaptation strength. By introducing a scalar temperature parameter, we can smoothly interpolate between pretrained and fully-adapted behaviors at inference time. 

We test LoCO across three domains: (1) fine-tuning language models, (2) adapting diffusion transformers for conditional generation, and (3) fine-tuning vision transformers. Across all settings, LoCO achieves competitive or superior performance compared to established PEFT methods including LoRA~\cite{hu2022lora}, OFT~\cite{qiu2023controlling}, BOFT~\cite{liu2024boft}, and HRA~\cite{yuan2024bridging}, while maintaining computational efficiency and a low memory footprint.

\noindent\textbf{Contributions.} Our main contributions are threefold:
\begin{itemize}
    \item We propose LoCO, a parameter-efficient fine-tuning method that constructs orthogonal transformations via low-rank skew-symmetric matrices. We further introduce an approximation for compositional rotation chains that enables parallel computation while preserving orthogonality for fine-tuning.
    
    \item We demonstrate a test-time temperature scaling mechanism that enables flexible control over adaptation strength without retraining, providing an alternative to costly regularization schemes used in prior orthogonal methods.
    
    \item We achieve state-of-the-art or competitive results across diverse benchmarks, including GLUE, mathematical reasoning, and controllable generation.
\end{itemize} \section{Related Work}

\paragraph{Low-rank Adaptation.}
Building on the low intrinsic rank hypothesis~\cite{aghajanyan2020intrinsic}, LoRA~\cite{hu2022lora} parameterizes weight updates as a product of two low-rank matrices. This avoids the inference overhead of adapters~\cite{adapter} and the optimization challenges in prompt tuning~\cite{liliang2021prefix,lesteretal2021power}. Subsequent works extend this framework through adaptive rank allocation~\cite{zhang2023adaptive,kalajdzievski2023rslora,jiang2024morahighrank}, alternative initializations~\cite{meng2024pissa,wang2024milora}, and parameter sharing~\cite{balazy2024lora,dong2024arcvision,kopiczko2024vera}. DoRA~\cite{liu2024dora} further decomposes updates into magnitude and direction components. Unlike this line of research, our approach applies multiplicative transformations, constructing rotation matrices from low-rank skew-symmetric parameterizations.

\paragraph{Orthogonal Fine-Tuning.}
OFT~\cite{qiu2023controlling,qiu2025reparameterized} parameterizes orthogonal matrices through block-diagonal, butterfly structures~\cite{liu2024boft}, structured matrices~\cite{gorbunov2024group}, exploiting sparsity for efficiency. While qGOFT~\cite{ma2024parameter} decomposes rotations into Givens rotations, its sequential composition limits parallelization. Householder-based methods such as~\cite{dong2024efficient,yuan2024bridging,arcas2025hoft} use reflection parameterizations. Recent work~\cite{liao20243} applies 2D rotations directly to representation subspaces. Our method differs by constructing rotations from low-rank skew-symmetric matrices, enabling efficient computation through first-order approximation while operating on higher-dimensional subspaces.

\paragraph{Spectral-based Methods.}
In another line of research, PiSSA~\cite{meng2024pissa} and MiLoRA~\cite{wang2024milora} decompose weights via singular value decomposition. These approaches adapt residual or minor singular components selectively. Compared to these spectral-based methods, our approach preserves angular structure through orthogonal transformations rather than modifying spectral components.

\paragraph{Geometric Approaches.}
Several works exploit geometric structures: hyperbolic networks~\cite{ganea2018hyperbolic,chen2022fully} utilize exponential maps for hierarchical representations, and Lie group methods~\cite{mironenco2024lie,finzi2020generalizing} construct equivariant architectures. Skew orthogonal convolutions~\cite{singla2021skew} use matrix exponentials for training stability. While building from similar mathematical tools, these works focus on designing new architectures rather than proposing new PEFT methods. \begin{figure*}[t]
	\centering
	\includegraphics[width=0.90\textwidth]{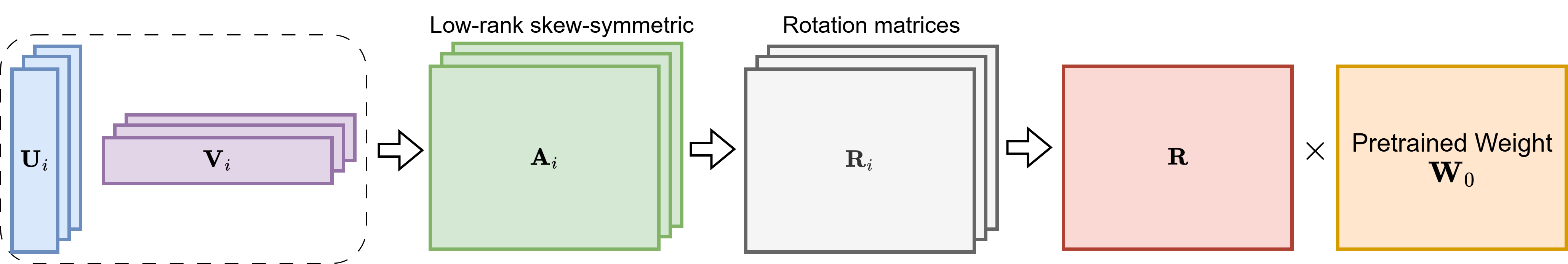}
	\caption{An overview of LoCO. We parameterize each rotation component $\vR_i$ using low-rank skew-symmetric matrices constructed from learnable factors $\vU_i, \vV_i$. We then compose these rotations to form the final orthogonal transformation to adapt pretrained weights.}
	\label{fig:main_figure}
\end{figure*}

\section{Low-rank Compositional Orthogonal Fine-tuning (LoCO)}

This section first reviews rotation groups and their connection to skew-symmetric matrices, and then presents the LoCO adaptation method.

\subsection{Rotation Groups and Skew-symmetric Matrices}

\paragraph{The special orthogonal group.} A rotation in $\RR^d$ is defined as an element of the special orthogonal group:
\begin{equation*}
\SO(d) = \left\{\vR \in \RR^{d \times d}: \vR^\top \vR = \vI, \det(\vR) = 1\right\}.
\end{equation*}
Here, the condition $\vR^\top \vR = \vI$ indicates that rotations preserve distance, while $\det(\vR) = 1$ distinguishes proper rotations from reflections.~\footnote{Reflection transformations have $\det(\vR) = -1$.}
\paragraph{Skew-symmetric matrices as tangent space.} The tangent space to $\SO(d)$ at the identity is given by:

\begin{equation}
	\sso(d) = \left\{\vA \in \RR^{d \times d}: \vA^\top = -\vA \right\}.
\end{equation}

These matrices represent infinitesimal rotations or rotation velocities. The dimension of $\sso(d)$ is $\frac{d (d -1)}{2}$, corresponding to the number of independent rotation axes in $d$ dimensions.

\paragraph{Matrix exponential map.} The exponential map provides a canonical way to generate rotations from skew-symmetric matrices. Every rotation can be obtained from some skew-symmetric matrix. That is, given $\vA \in \sso(d)$, we have
\begin{equation}
	\exp(\vA) = \sum_{n=0}^\infty \frac{\vA^n}{n!} \in \SO(d).
\end{equation}
This provides a \emph{smooth} parameterization of rotation manifold. However, computing the matrix exponential directly is computationally expensive and would be prohibitive for the task of fine-tuning foundation models.

\paragraph{Cayley transform.} The Cayley transform offers an alternative approach which is more tractable for generating rotations from skew-symmetric matrices. Given $\vA \in \sso(d)$, we have:
\begin{equation}
	\vR = (\vI - \vA)^{-1}(\vI + \vA) \in \SO(d).
\end{equation}

The Cayley transform is particularly well-suited in fine-tuning settings. During fine-tuning, we expect model adaptations to be relatively small, corresponding to rotations with limited angular magnitude. The Cayley transform provides an accurate and smooth parameterization of the rotation manifold near the identity. This approach has been successfully adopted in recent orthogonal fine-tuning methods~\cite{qiu2023controlling,liu2024boft}.

\subsection{LoCO Adaptations}
\label{sec:model_method}
This section presents the main specification of \textbf{Lo}w-rank \textbf{C}ompositional \textbf{O}rthogonal Fine-tuning (LoCO). 

Given a pretrained model with weight matrix $\vW_0 \in \mathbb{R}^{k \times d}$, the standard linear transformation can be adapted through orthogonal fine-tuning as follows.
\begin{equation*}
	f_{\textrm{pretrained}}(\vx) = \vW_0 \vx \quad \text{becomes} \quad f_{\textrm{finetune}}(\vx) = \vW_0 \vR \vx,
\end{equation*}
where $\vR \in \mathrm{SO}(d)$ is a special orthogonal matrix that performs rotation transformations on the input features $\vx \in \RR^d$. To enhance model expressivity, we utilize a compositional representation of $\vR$:
\begin{equation*}
	\vR = \vR_1 \vR_2 \cdots \vR_n,
\end{equation*}
where each component $\vR_i \in \mathrm{SO}(d)$ represents an elementary rotation. The composition of orthogonal matrices preserves orthogonality and ensures that $\vR \in \mathrm{SO}(d)$. Orthogonal fine-tuning takes inspiration from the idea that angular features capture the semantic gap well~\cite{chen2020angular,liu2017deep,liu2018decoupled,qiu2023controlling}.

In the following subsections, we detail how to parameterize each component $\vR_i$ using low-rank skew-symmetric matrices and describe the approximation scheme for the compositional rotation $\vR$. Figure~\ref{fig:main_figure} provides an overview of the LoCO framework.

\subsubsection{Skew-symmetric construction} We parameterize skew-symmetric matrices through a low-rank outer product form. Given learnable matrices $\vU, \vV \in \RR ^{d \times r}$ where $r \ll d$ is the rank dimension, we construct

\begin{equation*}
	\vA = \vU \vV^\top - \vV \vU^\top.
\end{equation*}

This formula guarantees skew-symmetry since we can verify that $\vA^\top = \left(\vU \vV^\top\right)^\top - (\vV \vU^\top)^\top = \vV \vU^\top - \vU \vV^\top = -\vA$. For notation convenience, we rewrite this construction using auxiliary matrices:
\begin{equation*}
	\vA = \vX \vY^\top, \text{where} \quad \vX = [\vU \mid -\vV], \quad \vY = [\vV \mid \vU],
\end{equation*}
with $\vX, \vY \in \RR^{d \times 2r}$. 

\paragraph{Efficient computation} The naive computation of the Cayley transform $\vR = (\vI - \vA)^{-1} (\vI + \vA)$ involves inverting a $d \times d$ matrix, which takes $\mathcal{O}(d^3)$ time complexity. This is prohibitively expensive for high-dimensional features. 

We now show how the low-rank structure of $\vA$ enables a reduction in computational cost. We apply the Woodbury matrix identity~\cite{max1950inverting}. For our setting that $\vA = \vX \vY^\top$, this identity leads us to
\begin{equation*}
	(\vI_d - \vX \vY^\top)^{-1} = \vI_{d} + \vX(\vI_{2r} - \vY^\top \vX)^{-1}\vY^\top.
\end{equation*}
In this case, we add subscripts for the identity matrix $\vI$ to indicate the dimension of the vector space. The inversion of a $d\times d$ matrix is now converted to the inversion of a $2r \times 2r$ matrix $(\vI_{2r} - \vY^\top \vX)$, which requires only $\mathcal{O}(r^3)$ operations. 

With a further algebraic manipulation, we arrive at:
\begin{equation}
	\label{eq:actual_adapter}
	\vR = (\vI - \vA)^{-1} (\vI + \vA) = \vI + 2 \vX (\vI - \vY^\top \vX)^{-1} \vY^\top.
\end{equation}

The main computational cost consists of (1) computing $\vY^\top \vX \in \mathbb{R}^{2r \times 2r}$ in $\mathcal{O}(dr^2)$ time, (2) inverting the $2r \times 2r$ matrix in $\mathcal{O}(r^3)$ time, and (3) matrix multiplications in $\mathcal{O}(dr^2)$ time. The total complexity is $\mathcal{O}(dr^2 + r^3)$. Since $r \ll d$ in practice (typically $r \in [1, 4]$), this simplifies to $\mathcal{O}(dr^2)$. {In practice, we adopt input-centric implementation by multiplying the input vector with low-rank matrices thus avoiding large intermediate matrices.}

\paragraph{Comparison with other orthogonal approaches.} Existing orthogonal fine-tuning methods employ alternative structural constraints to manage computational costs: block-diagonal parameterization~\cite{qiu2023controlling} partitions transformations into independent subspaces, while butterfly structure~\cite{liu2024boft} exploits hierarchical factorizations. Our low-rank approach instead enables explicit control over the rotation subspace dimension through the rank parameter $r$.

\subsubsection{Compositional Chain-of-Rotations}
\label{model:Compositional Chain-of-Rotation}

From a theoretical perspective, any rotation in $\RR^d$ can be decomposed into a sequence of Givens rotations in 2-dimensional planes. Specifically, any element of $\SO(d)$ can be written as a product of at most $\frac{d(d - 1)}{2}$ Givens rotations~\cite{press2007numerical}. Our chain-of-rotations framework can be viewed as a learnable, parameter-efficient approximation to this decomposition, where each component rotation $\vR_i$ operates on a distinct low-dimensional subspace determined by $\vU_i, \vV_i$.

We define a chain of $n$ rotations:
\begin{equation*}
	\vR = \vR_1 \vR_2 \cdots \vR_n = \prod_{i=1}^n \left(\vI + 2\vX_i (\vI - \vY_i^\top \vX_i)^{-1}\vY_i^\top\right),
\end{equation*}

where each $\vR_i$ is parameterized by its own pair $\vU_i \in \RR^{d \times r}, \vV_i \in \RR^{d \times r}$. Since the composition of orthogonal matrices remains orthogonal, $\vR \in \SO(d)$ is guaranteed.

To enable efficient parallel computation, we use a first-order approximation by expanding the product and retaining only linear terms:

\begin{equation}
	\vR \approx \vI + 2\sum_{i=1}^n \vX_i (\vI - \vY_i^\top \vX_i)^{-1}\vY_i^\top.
	\label{eq:first_order_approx}
\end{equation}
This approximation replaces sequential matrix multiplication with a sum of low-rank updates, which can be computed in parallel using standard GPU operations.

\paragraph{Approximation error analysis.} Consider rotations where $\normx{\vX_i}{F}, \normx{\vY_i}{F} = \mathcal{O}(\varepsilon)$ for small $\varepsilon > 0$. Each perturbation $\Delta_i = 2\vX_i (\vI - \vY_i^\top \vX_i)^{-1} \vY_i^\top$ has magnitude $\mathcal{O}(\varepsilon^2)$, since it involves the product of $\vX_i$ and $\vY_i^\top$. When expanding the product:
\begin{align}
	\vR &= \prod_{i=1}^n (\vI + \Delta_i) = \vI + \sum_{i=1}^n \Delta_i + \sum_{i < j} \Delta_i \Delta_j + \mathcal{O}(\varepsilon^6).
\end{align}
Here, the second-order cross terms $\Delta_i \Delta_j$ are $\mathcal{O}(\varepsilon^4)$ and high-order terms are at least $\mathcal{O}(\varepsilon^6)$. During fine-tuning, where parameters are initialized near zero and remain small throughout training, these higher-order terms contribute negligibly to the transformation. The orthogonality approximation error is bounded by \Cref{thm:upper_bound_dev} (see \Cref{sec:app_error} for the full proof).

\paragraph{Orthogonality.} We numerically evaluate the first-order approximation by measuring how well it preserves vector norms. For randomly initialized vectors $\vx$, we compute the relative error $|\|\vx\| - \|\vR\vx\|| / \|\vx\|$ across $\|\vX_i\|_F = \|\vY_i\|_F = \varepsilon$ for $\varepsilon \in [10^{-6}, 10^{-0.5}]$, a range typical of fine-tuning perturbations. Figure~\ref{fig:error_analysis} shows that our approximation effectively preserves vector magnitude (ideally $\|\vR\vx\| = \|\vx\|$) across the range of $\varepsilon$ values.

\begin{figure}[h]
	\centering
	\includegraphics[width=0.45\textwidth]{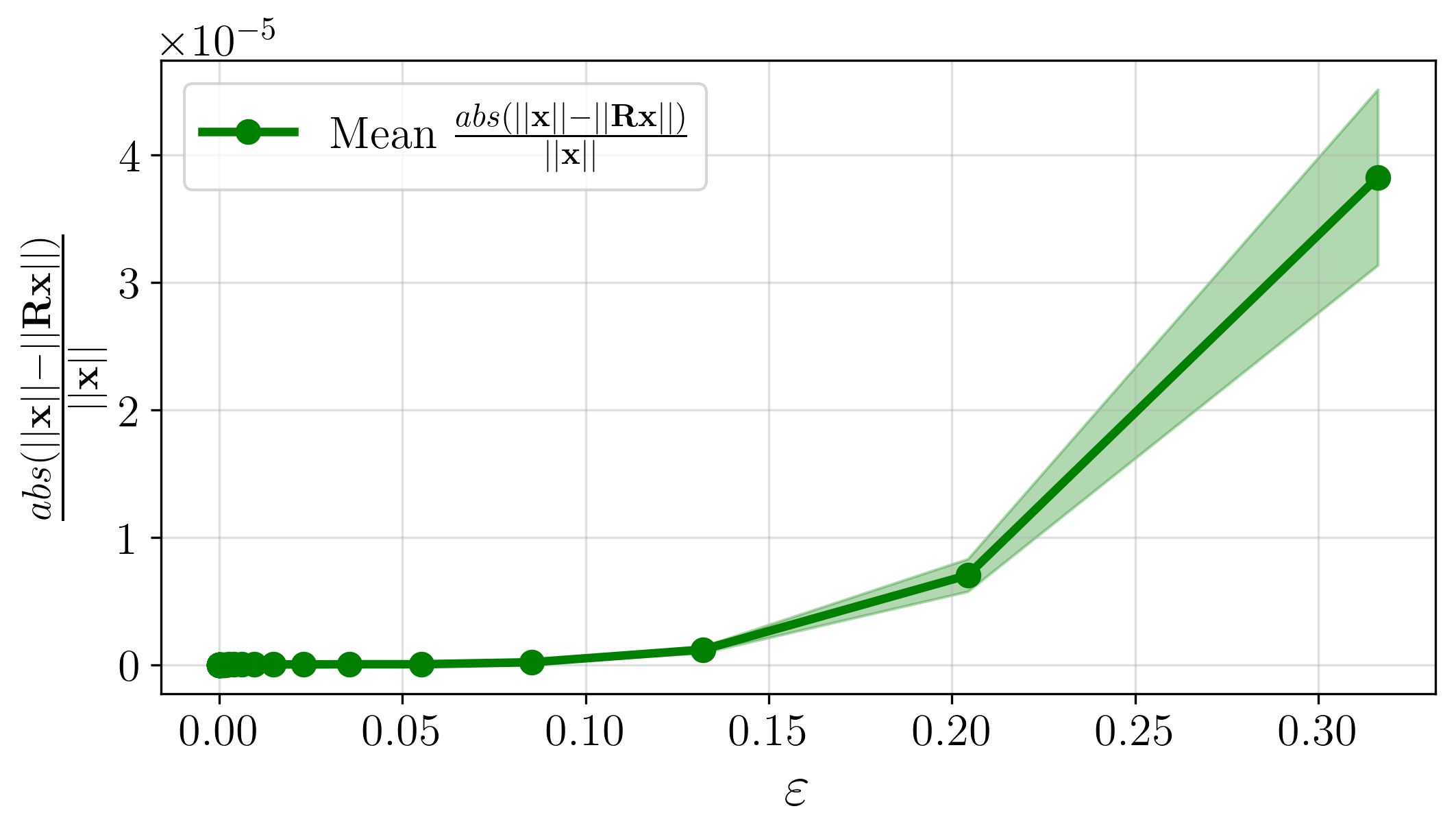}
	\caption{Relative error $|\normx{\vx}{} - \normx{\vR\vx}{} | / \normx{\vx}{}$ in first-order approximations in Equation~\eqref{eq:first_order_approx}. This illustrates that our first-order approximation can preserve vector magnitude under linear transformation $\vR$. In this experiment, we vary $\normx{\vX}{} = \normx{\vY}{}$ across a range of $\varepsilon \in [10^{-6}, 10^{-0.5}]$. The shaded region indicates the standard deviation across multiple random initializations of $\vX$ and $\vY$.}
	\label{fig:error_analysis}
\end{figure}

\paragraph{Remark.} The first-order approximation enables all $n$ rotations to be computed independently and summed, fully utilizing GPU parallelism, while simplifying backpropagation compared to sequential matrix products. While compositional orthogonal transformations appear in BOFT~\cite{liu2024boft} (butterfly factorization) and reflection-based methods~\cite{yuan2024bridging}, neither leverages parallel computation. BOFT uses sequential butterfly stages, and reflection chains require multiple compositions. Our outer-product parameterization with first-order approximation uniquely enables fully parallel execution. This is also reflected by our experimental validation in Section~\ref{sec:experimental_results}.

 \section{Experimental Results}
\label{sec:experimental_results}

We evaluate LoCO on three tasks: fine-tuning large language models, diffusion transformers, and vision transformers.

\subsection{Large Language Model Adaptation}
\label{sec_llm}
We examine LoCO on two tasks: natural language understanding (NLU) and mathematical reasoning.

\paragraph{Baselines.} We compare against orthogonal fine-tuning methods: (a) OFT~\cite{qiu2023controlling}, which uses block-diagonal orthogonal matrices; (b) BOFT~\cite{liu2024boft}, which employs butterfly factorizations to construct dense orthogonal adapters; and (c) HRA~\cite{yuan2024bridging}, which chains Householder reflections~\cite{householder1958unitary}.

\subsubsection{Natural Language Understanding}
\label{exp:llm:nlu}
Following~\cite{qiu2023controlling,liu2024boft,yuan2024bridging}, we fine-tune DeBERTa-V3-base~\cite{he2021deberta} on GLUE benchmark tasks.

\begin{table*}[htbp]
    \centering
    \footnotesize \setlength{\tabcolsep}{5pt} 
\begin{tabular}{llc|cccccccc|c}
        \toprule
        \textbf{Method} & \textbf{Config} & \textbf{\#Params} & \textbf{MNLI} & \textbf{SST-2} & \textbf{MRPC} & \textbf{CoLA} & \textbf{QNLI} & \textbf{QQP} & \textbf{RTE} & \textbf{STS-B} & \textbf{Mean} \\
        \midrule
        OFT  & b=24    & 0.954M & \textbf{90.27} & 95.83          & 87.93          & 67.63          & 94.30          & 89.77          & 84.27         & 91.00          & 87.63 \\
        BOFT & r=2,b=4 & 0.746M & \textbf{90.27} & 96.00          & 87.07          & \textbf{68.20} & 94.07          & 89.73          & 81.43         & 91.13          & 87.24 \\
        HRA  & r=8     & 0.663M & 90.20          & 95.70          & 88.03          & 51.67          & 93.57          & \textbf{90.20} & 84.50         & 91.25          & 85.64 \\
        \specialrule{0.01pt}{2pt}{2pt}
        \textbf{LoCO (ours)} & n=1,r=4 & 0.663M & 90.20          & \textbf{96.14} & \textbf{88.28} & 65.56          & \textbf{94.36} & 90.08          & \textbf{86.10} & \textbf{91.28} & {\textbf{{87.75}}} \\
        \bottomrule
    \end{tabular}
    \vspace{5pt} \caption{Experimental results on GLUE benchmark. The best results on each dataset are shown in \textbf{bold}. \textbf{\#Params} refers to the total number of trainable parameters in the backbone side - excluding the classification head layers which are trained in a standard way for each task.}
    \label{tab:glue_deberta}
\end{table*} \paragraph{Setup.} We adopt the experimental setup of BOFT and HRA, applying each method to all linear modules within transformer layers. 
We increase OFT's block size from $b{=}16$ to $b{=}24$ to match parameter counts, and use HRA's best configuration ($\lambda{=}0$). Additional details are in Appendix~\ref{appendix:nlu_exps}.

\paragraph{Results.} Table~\ref{tab:glue_deberta} reports matched accuracy for MNLI, Matthews correlation for CoLA, Pearson correlation for STS-B, and accuracy for other tasks. LoCO achieves the best average score (87.75), outperforming HRA (85.64) while using 30\% and 10\% fewer parameters than OFT and BOFT, respectively. LoCO ranks first on 5 out of 8 tasks (SST-2, MRPC, QNLI, RTE, STS-B). Notably, HRA shows a substantial drop on CoLA, suggesting Householder reflections may be less stable across diverse linguistic tasks than orthogonal rotations.

\subsubsection{Mathematical Reasoning}
\label{exp:llm_math}
\paragraph{Setup.} We fine-tune LLaMA2-7B~\cite{touvron2023llama2openfoundation} on MetaMathQA-40K~\cite{yu2023metamath}, adapting only the query and value projections ($W_q$, $W_v$). All baselines are reproduced under identical settings~\cite{yuan2024bridging,liu2024boft} and evaluated on the GSM8K~\cite{cobbe2021gsm8k} and MATH~\cite{hendryckstest2021} validation sets.

\begin{table}[H]
    \centering
\resizebox{0.95\columnwidth}{!}{\begin{tabular}{llccc}
        \toprule
        \textbf{Method} & \textbf{Config} & \textbf{\%Param} & \textbf{GSM8K} & \textbf{MATH} \\
        \midrule
        Llama-2-7B & -- & -- & 14.6 & 2.5 \\
        \midrule
        OFT  & $b{=}64$       & 0.12 & 49.35 & 8.39 \\
        BOFT & $m{=}2, b{=}8$ & 0.13 & 45.58 & 6.85 \\        
        HRA  & $r{=}32$       & 0.12 & 50.04 & 7.45 \\
        \midrule
        \textbf{LoCO (ours)} & $n{=}1, r{=}16$ & 0.12 & \textbf{50.19} & \textbf{8.40} \\
        \bottomrule
    \end{tabular}}
    \vspace{1pt}
    \caption{Experimental results on mathematical reasoning tasks using Llama-2-7B. We report the accuracy (\%) on GSM8K and MATH benchmarks. The best results are highlighted in \textbf{bold}.}
    \label{tab:math_reasoning}
\end{table}
 
\paragraph{Results.} As shown in Table~\ref{tab:math_reasoning}, LoCO achieves the highest GSM8K score (50.19), outperforming OFT (49.35), HRA (49.65), and BOFT (45.58) with only 0.12\% trainable parameters. On MATH, LoCO matches OFT while surpassing BOFT and HRA. These results suggest that low-rank skew-symmetric parameterization captures mathematical reasoning patterns more effectively than block-diagonal (OFT/BOFT) or Householder (HRA) structures.

\subsubsection{Computational Efficiency}

We compare training efficiency by measuring time per step and peak memory after 100 warm-up steps, averaged over 300 iterations across 3 trials. Experiments are conducted on DeBERTa-V3 (184M) and LLaMA2-7B with matched parameter counts. Details are provided in Appendix~\ref{appendix:computation_efficiency}.

\begin{figure}[htbp]
	\centering
	\includegraphics[width=0.95\linewidth]{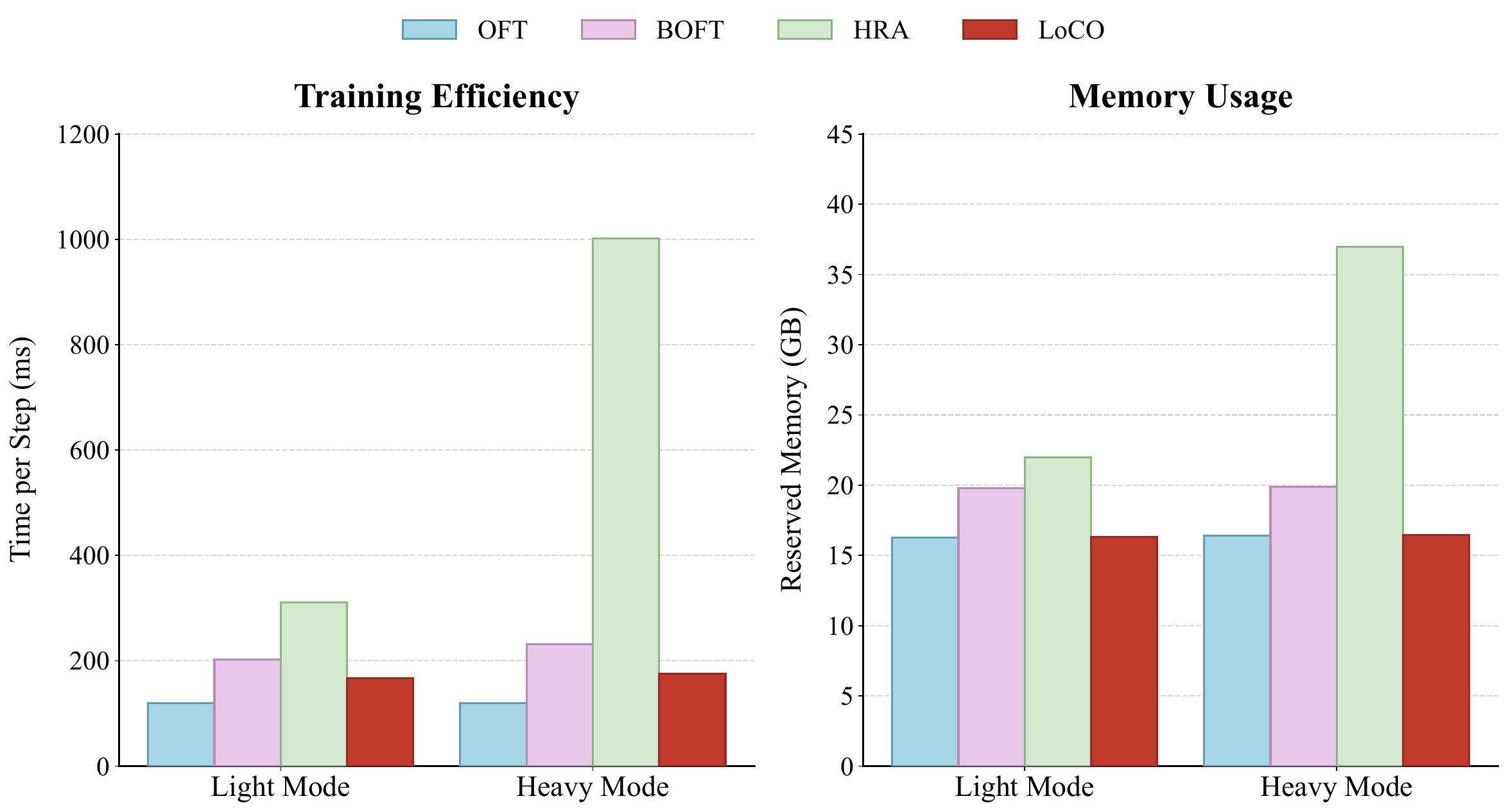}
	\caption{\label{fig:deberta}Training efficiency comparison on DeBERTA-V3.}
\end{figure}

As shown in Figure~\ref{fig:deberta}, LoCO matches the efficiency of sparse OFT while outperforming dense methods (HRA, BOFT). Under heavy configurations, HRA incurs 160\% longer step time and 20\% higher memory usage, whereas LoCO maintains stable throughput. This efficiency stems from two factors: (1) the Sherman–Morrison–Woodbury formula reduces matrix inversion to $O(r^3)$ complexity, and (2) all rotation components compute in parallel, unlike HRA's sequential backward pass or BOFT's non-contiguous memory access patterns~\cite{dao2022monarch}.

\paragraph{Summary.} Across NLU and mathematical reasoning benchmarks, LoCO achieves competitive or superior performance while maintaining both parameter efficiency and computational scalability.

\subsection{Fine-tuning Diffusion Transformers}
\label{sec:diffusion}
\paragraph{Setup.} 
We follow the setup of~\cite{tan2025ominicontrol} focusing on conditional generation tasks. Similar to the OminiControl setup, we use FLUX.1~\cite{flux2024} as the pre-trained model. We conduct all experiments on 4 NVIDIA H200 GPUs. To match the total samples of OminiControl~\cite{tan2025ominicontrol}, we train our models for 25,000 iterations (half of the setting of~\cite{tan2025ominicontrol} when running on 2 GPUs).

\textbf{Baselines.} We compare our approach against two representative baseline methods: LoRA~\cite{hu2022lora}---the default adapter of OminiControl~\cite{tan2025ominicontrol}, serving as a non-orthogonal adapter---and butterfly orthogonal fine-tuning (BOFT)~\cite{liu2024boft} as an orthogonal adapter. All baseline experiments utilize the official OminiControl implementation.
We extend this framework to incorporate BOFT~\cite{liu2024boft} into the FLUX pipeline. Our LoCO adapter is implemented within the HuggingFace PEFT library to ensure seamless integration with the existing codebase.

We evaluate performance using established metrics: Fréchet Inception Distance (FID)~\cite{fid}, Structural Similarity Index (SSIM), CLIP-IQA~\cite{clip_iqa}, MANIQA~\cite{maniqa}, MUSIQ~\cite{musiq}, and Peak Signal-to-Noise Ratio (PSNR)~\cite{psnr}. For controllability assessment, we adopt the OminiControl approach, measuring the F1 score between extracted and input edge maps for edge-conditioned generation, and Mean Squared Error (MSE) between extracted and original condition maps for other tasks. To evaluate consistency between prompts and generated images, we use CLIP Text-Image similarity. Following OminiControl~\cite{tan2025ominicontrol}, we evaluate on 5,000 images from the COCO dataset, with prompts extracted from corresponding image captions.

\begin{table*}[t]
	\centering
	
	\resizebox{\textwidth}{!}{
		\begin{tabular}{ll|c|cccccc|cc}
			\toprule
			\multirow{2}{*}{\textbf{Task}} & \multirow{2}{*}{\textbf{Methods}} &
			{\textbf{Controllability}} &
			\multicolumn{6}{c|}{\textbf{Image Quality}} &
			\multicolumn{2}{c}{\textbf{Alignment}} \\
			\cmidrule(lr){3-3} \cmidrule(lr){4-9} \cmidrule(lr){10-11}
			& & F1↑ / MSE↓ & FID↓ & SSIM↑ & CLIP-IQA↑ & MAN-IQA↑ & MUSIQ↑ & PSNR↑ & CLIP Text↑ & CLIP Image↑ \\
			\midrule
			
			\multirow{3}{*}{Canny}
			& LoRA & \textbf{0.52} & 25.12 & \textbf{0.45}& 0.62 & 0.57 &74.22 & \textbf{11.44} & 0.331 & 0.798\\
			& BOFT  & 0.4 & 28.71 & 0.39 & 0.65 & 0.58 & 73.82 & 10.89 & 0.308 & \textbf{0.799} \\
			
& \textbf{LoCO (ours)}  & 0.49 & \textbf{25.00} & 0.44 & \textbf{0.66} & \textbf{0.61} & \textbf{74.24} & 11.17 & \textbf{0.336} & 0.788 \\
			\midrule
			
			\multirow{4}{*}{Depth}
			& LoRA  & \textbf{624} & \textbf{29.98} & 0.31& 0.62 & 0.52 & 71.65 & 11.33 & 0.315 & 0.729 \\
			& BOFT  & 752 & 31.29 & 0.32 & 0.59 & 0.50 & 71.99 & \textbf{11.35} & 0.304 & 0.726 \\
			
& \textbf{LoCO (ours)}  & 723 & 32.01 & \textbf{0.36} & \textbf{0.63} & \textbf{0.53} & \textbf{72.50} & 9.85 & \textbf{0.319} & \textbf{0.741} \\
			\midrule
			
			\multirow{3}{*}{Mask}
			& LoRA & \textbf{6402} & 10.11 & 0.62 & 0.58 & 0.47 & 68.81 & \textbf{18.31} & \textbf{0.292} & 0.812 \\
			& BOFT  & 7482 & \textbf{10.06} & 0.55 & 0.51 & 0.49 & 66.31 & 18.01 & 0.283 & 0.801 \\
& \textbf{LoCO (ours)}  & 6931 & 10.14 & \textbf{0.63} & \textbf{0.59} & \textbf{0.50} & \textbf{69.80} & 18.08 & 0.290 & \textbf{0.824} \\
			\midrule

			\multirow{3}{*}{Colorization}
			& LoRA  & \textbf{103} & 10.51 & 0.91 & \textbf{0.54 }& \textbf{0.45} & \textbf{67.33}  & \textbf{22.30} & 0.298 & 0.826 \\
			& BOFT  & 121 & \textbf{10.01} & \textbf{0.92} & 0.51 & 0.42 & 66.32 & 22.19 & 0.297 & 0.812 \\
			
& L\textbf{LoCO (ours)} & 106 & 10.33 & 0.90 & 0.50 & 0.44 & 66.44 & 21.26 & \textbf{0.310} & 0.831 \\

			\midrule
			\multirow{3}{*}{Deblur}
			& LoRA  & \textbf{79} & 20.27 & \textbf{0.65} & 0.41 & 0.35 & 58.75 & \textbf{23.95} & 0.290 & 0.812 \\
			& BOFT  & 87 & 22.94 & 0.60 & 0.49 & 0.40 & 57.95 & 20.78 & 0.304 & 0.803 \\
			
& \textbf{LoCO (ours)}  & 83 & \textbf{20.06} & 0.61 & \textbf{0.54} & \textbf{0.48} & \textbf{65.85} & 21.79 & \textbf{0.318} & \textbf{0.832} \\
			
			\bottomrule
	\end{tabular}}
	\caption{Quantitative comparison across different controllable generation tasks. 
		$\uparrow$ indicates higher is better, $\downarrow$ indicates lower is better.}
	\label{tab:spatialy_align}
\end{table*}

\paragraph{Results.} Table~\ref{tab:spatialy_align} presents quantitative results across spatially-aligned controllable generation tasks. Our method demonstrates competitive performance across the evaluation landscape. Specifically, on image quality metrics (CLIP-IQA, MAN-IQA, MUSIQ), our approach consistently achieves performance comparable to or exceeding the baselines. More significantly, our method exhibits superior alignment fidelity as measured by CLIP Image scores across the majority of tasks. For instance, on the Mask task, our method achieves a CLIP Image score of 0.824, outperforming BOFT (0.801) and LoRA (0.812). Similarly, on the Deblur task, we obtain 0.832 compared to 0.803 for BOFT. These results indicate that our low-rank compositional rotation approach effectively preserves semantic alignment. Qualitative examples of these methods are provided in Appendix~\ref{appendix:diffusion}.

\paragraph{Test-time adjustable generation.}

We introduce a temperature parameter $t$ for the skew-symmetric matrix $\vA$, whereby the rotation matrix is computed as $\vR = (\vI - \vA t)^{-1}(\vI + \vA t)$. During training, we use $t=1$. At inference time, $t$ can be adjusted to control the degree of adaptation. Since the scaled matrix $\vA t$ remains within the tangent space $\sso(d)$ for varying values of $t$, the resulting transformations continue to yield semantically meaningful outputs in the generated images. This property is preserved because the Cayley transformation provides a good approximation to $\exp(\vA t)$ given a small magnitude of $\vA$. 

This adjustable inference mechanism described here has not been explored by prior orthogonal fine-tuning approaches~\cite{qiu2023controlling,liu2024boft}. Furthermore, orthogonal fine-tuning methods based on reflection transformations, such as Householder-based approaches~\cite{yuan2024bridging}, do not possess this capability due to their different geometric structure.

Figure~\ref{fig:varying_t_main} illustrates the visual quality of generated images across a range of temperature values $t \in [0, 2]$. Additional figures demonstrating generated images beyond this range using LoCO are provided in the Appendix. Compared with images generated by varying the scaling parameter $\alpha$ in LoRA for OminiControl (the two bottom rows in Figure~\ref{fig:varying_t_main}), our method exhibits robustness and stability across different scaling factors. This is verified quantitatively by image quality metrics in Figure~\ref{fig:vary_t_2}. The generation quality when varying $t$ is more robust and less brittle than LoRA.
\newcommand{\bigarrow}[6]{
\shade[left color=#5, right color=#6, middle color=cyan!45, rounded corners=2pt]
	(0,-#2) -- (0,#2) -- (#1,#2) -- (#1,#4) --
	({#1+#3},0) -- (#1,-#4) -- (#1,-#2) -- cycle;
}

\begin{figure}[h]
	\centering
\newcommand{\imgspacing}{0.8}
	{
	\begin{tikzpicture}
	\begin{scope}
\node (img1) at (0*\imgspacing, 0) {\includegraphics[width=0.04\textwidth]{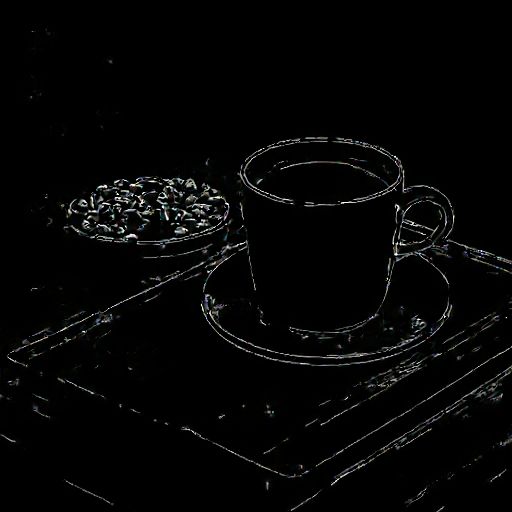}};
	\node (img2) at (1*\imgspacing, 0) {\includegraphics[width=0.04\textwidth]{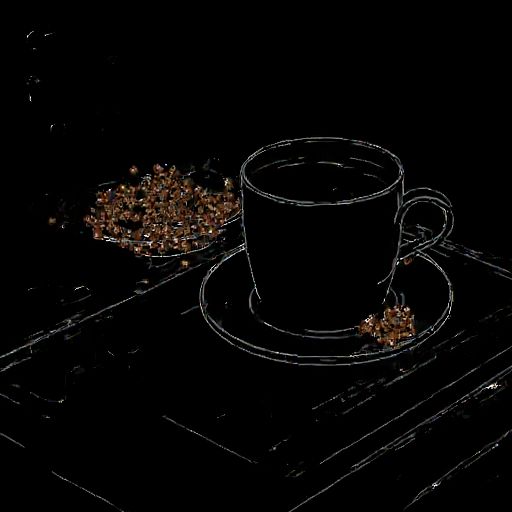}};
	\node (img3) at (2*\imgspacing, 0) {\includegraphics[width=0.04\textwidth]{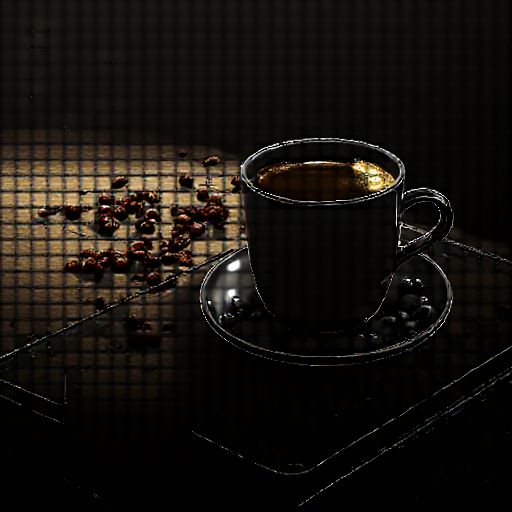}};
	\node (img4) at (3*\imgspacing, 0) {\includegraphics[width=0.04\textwidth]{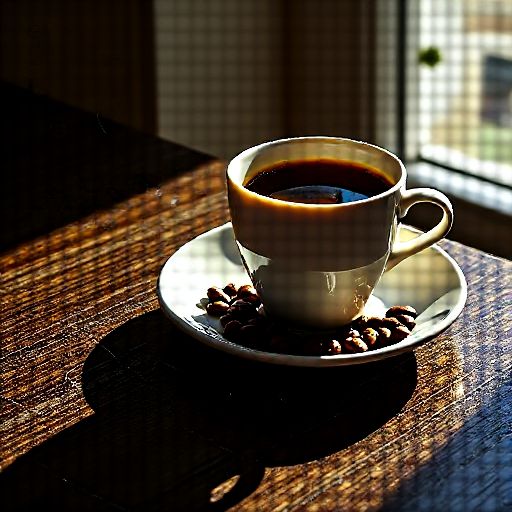}};
	\node (img5) at (4*\imgspacing, 0) {\includegraphics[width=0.04\textwidth]{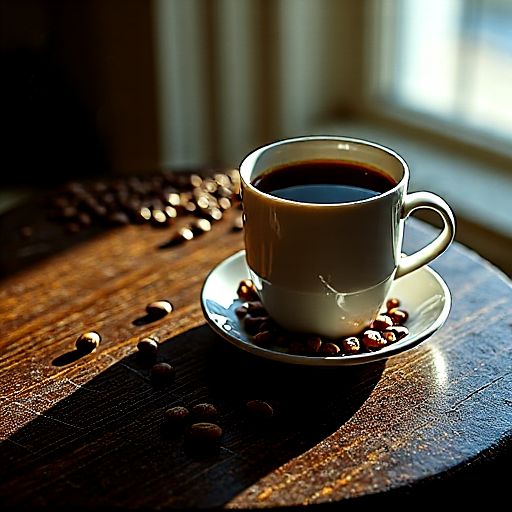}};
	\node (img6) at (5*\imgspacing, 0) {\includegraphics[width=0.04\textwidth]{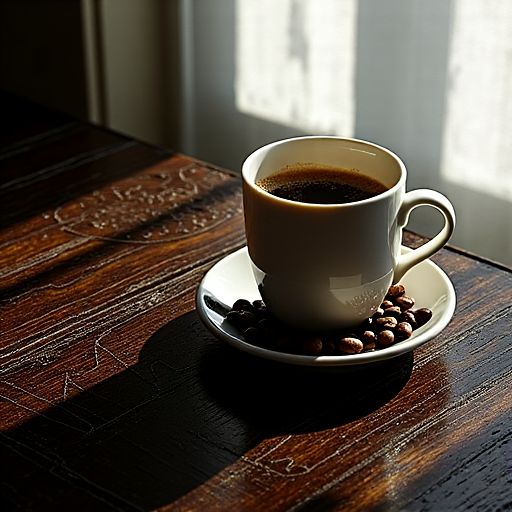}};
	\node (img7) at (6*\imgspacing, 0) {\includegraphics[width=0.04\textwidth]{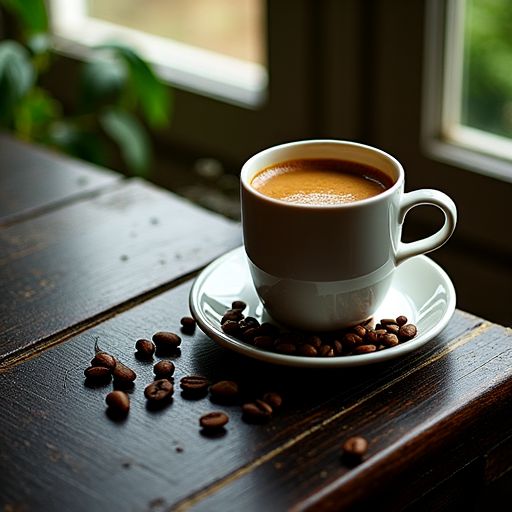}};
	\node (img8) at (7*\imgspacing, 0) {\includegraphics[width=0.04\textwidth]{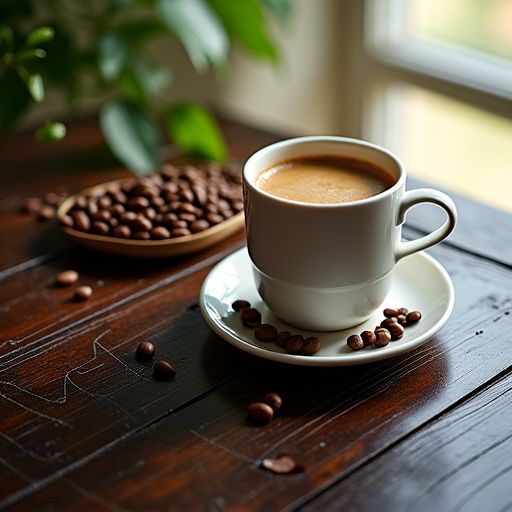}};
	\node (img9) at (8*\imgspacing, 0) {\includegraphics[width=0.04\textwidth]{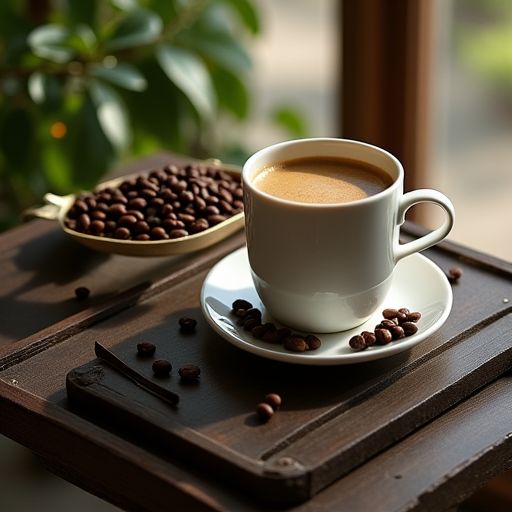}};
	\node (img10) at (9*\imgspacing, 0) {\includegraphics[width=0.04\textwidth]{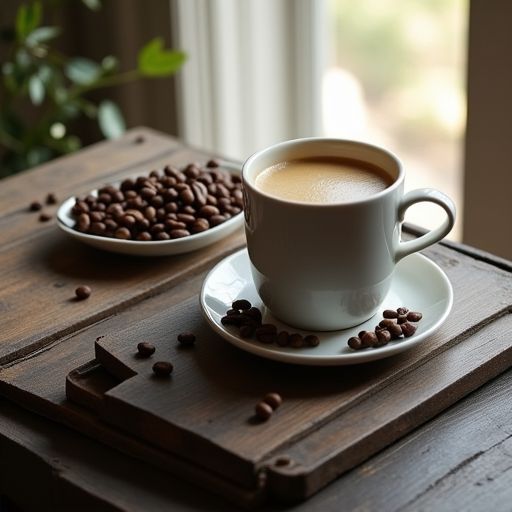}};

	\begin{scope}[shift={(-0.4, -0.5)}]
	\node at (0.35, -0.3) {$t=0$};
	\bigarrow{7.6}{0.1}{0.4}{0.15}{cyan!15}{cyan!75}
	\node at (7.6, -0.3) {$t=1$};
	\end{scope}
	\end{scope}
	
	\begin{scope}[shift={(0,-1.4)}]
\node (img1) at (0*\imgspacing, 0) {\includegraphics[width=0.04\textwidth]{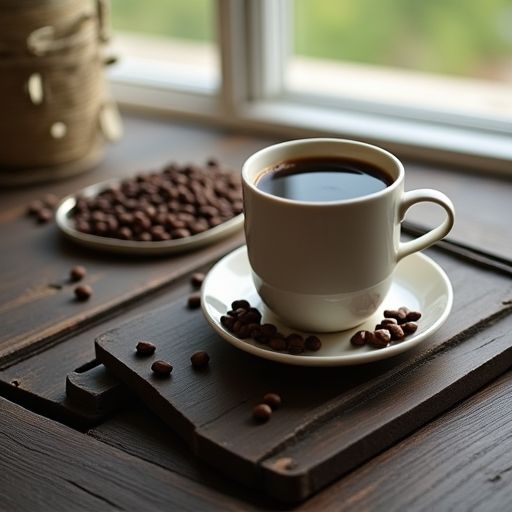}};
	\node (img2) at (1*\imgspacing, 0) {\includegraphics[width=0.04\textwidth]{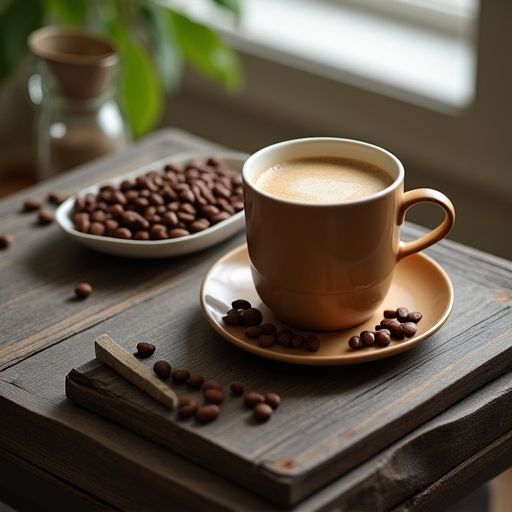}};
	\node (img3) at (2*\imgspacing, 0) {\includegraphics[width=0.04\textwidth]{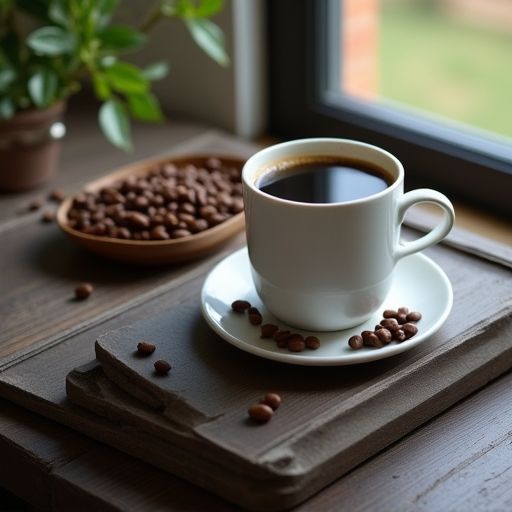}};
	\node (img4) at (3*\imgspacing, 0) {\includegraphics[width=0.04\textwidth]{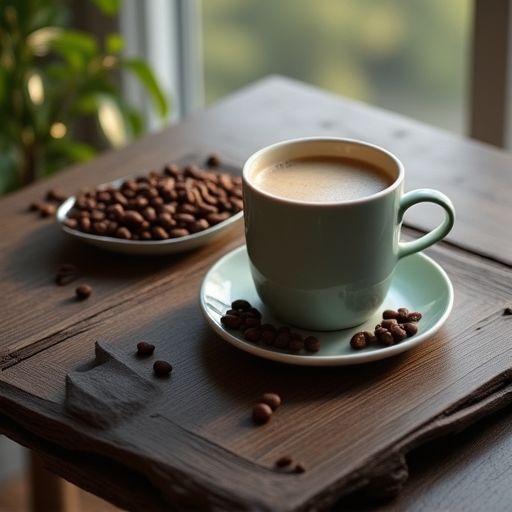}};
	\node (img5) at (4*\imgspacing, 0) {\includegraphics[width=0.04\textwidth]{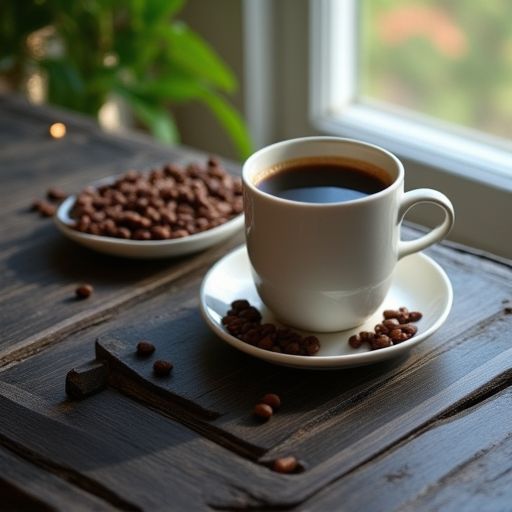}};
	\node (img6) at (5*\imgspacing, 0) {\includegraphics[width=0.04\textwidth]{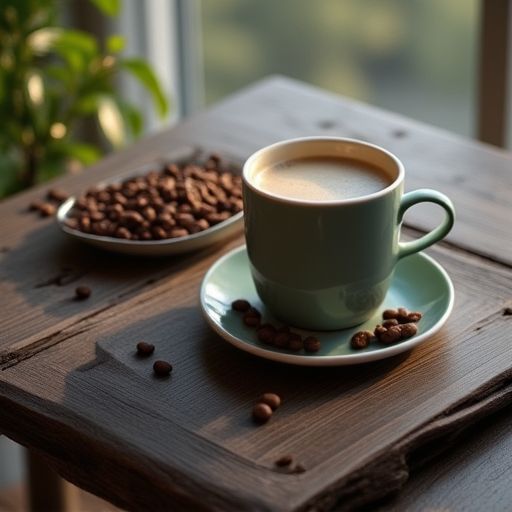}};
	\node (img7) at (6*\imgspacing, 0) {\includegraphics[width=0.04\textwidth]{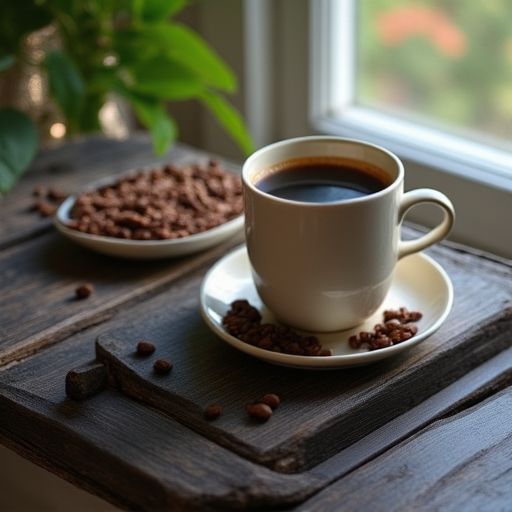}};
	\node (img8) at (7*\imgspacing, 0) {\includegraphics[width=0.04\textwidth]{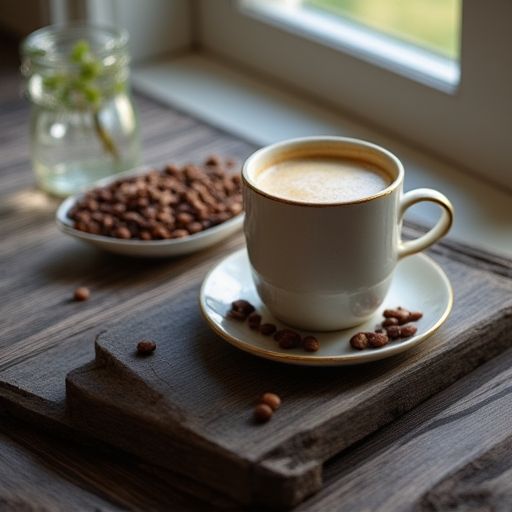}};
	\node (img9) at (8*\imgspacing, 0) {\includegraphics[width=0.04\textwidth]{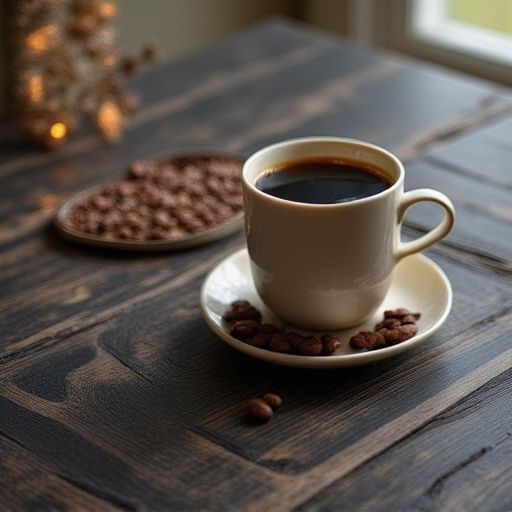}};
	\node (img10) at (9*\imgspacing, 0) {\includegraphics[width=0.04\textwidth]{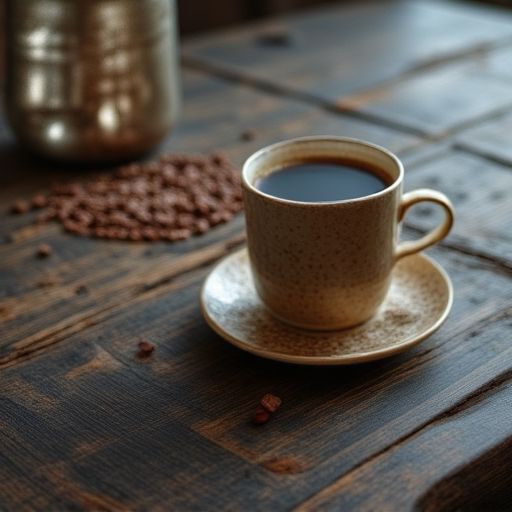}};

	\begin{scope}[shift={(-0.4, -0.5)}]
	\node at (0.35, -0.3) {$t=1$};
	\bigarrow{7.6}{0.1}{0.4}{0.15}{cyan!15}{cyan!75}
	\node at (7.6, -0.3) {$t=2$};
	\end{scope}
	\end{scope}

\begin{scope}[shift={(0,-2.8)}]
\node (img1) at (0*\imgspacing, 0) {\includegraphics[width=0.04\textwidth]{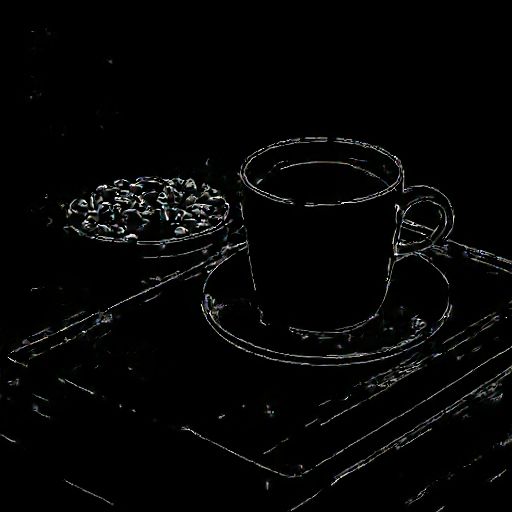}};
	\node (img2) at (1*\imgspacing, 0) {\includegraphics[width=0.04\textwidth]{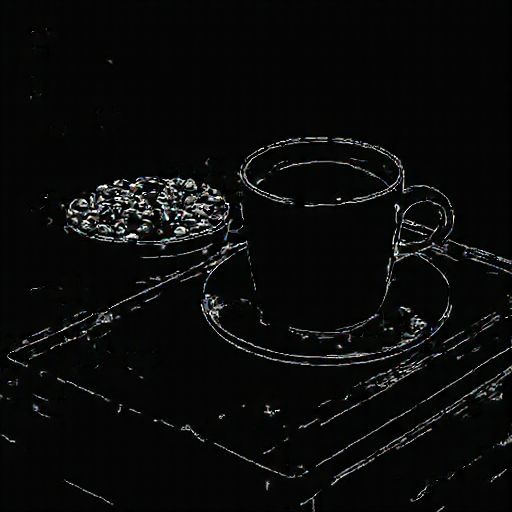}};
	\node (img3) at (2*\imgspacing, 0) {\includegraphics[width=0.04\textwidth]{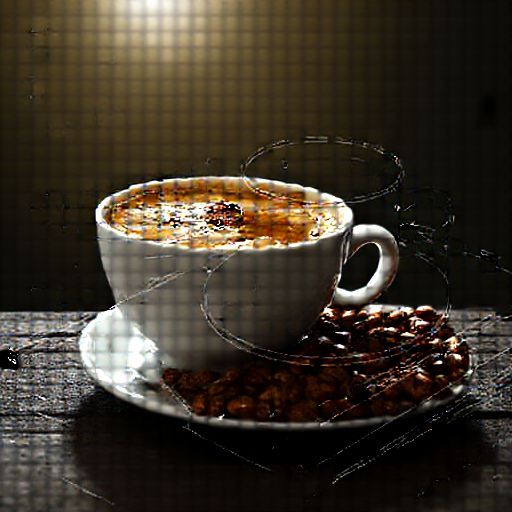}};
	\node (img4) at (3*\imgspacing, 0) {\includegraphics[width=0.04\textwidth]{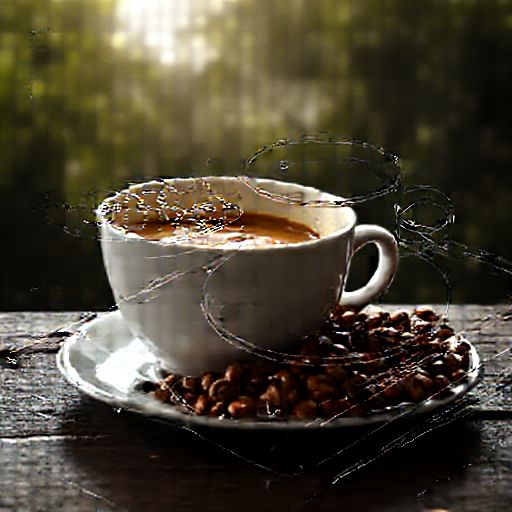}};
	\node (img5) at (4*\imgspacing, 0) {\includegraphics[width=0.04\textwidth]{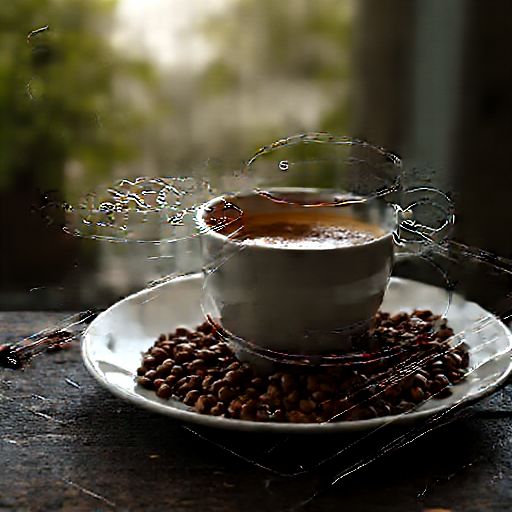}};
	\node (img6) at (5*\imgspacing, 0) {\includegraphics[width=0.04\textwidth]{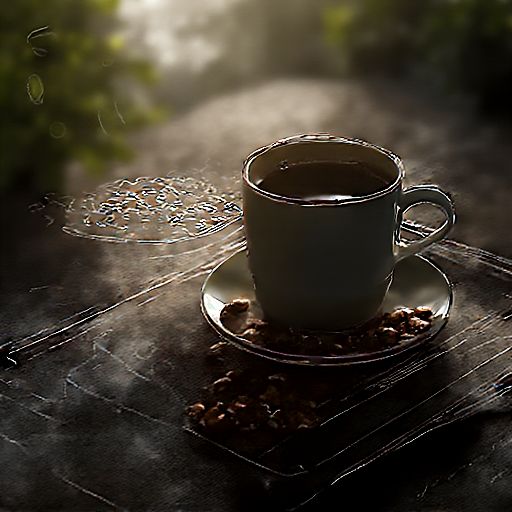}};
	\node (img7) at (6*\imgspacing, 0) {\includegraphics[width=0.04\textwidth]{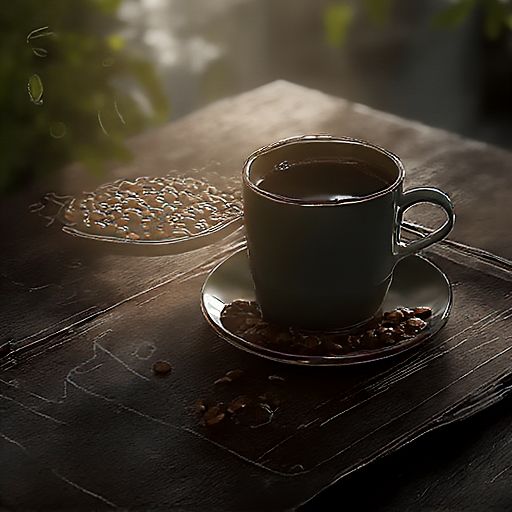}};
	\node (img8) at (7*\imgspacing, 0) {\includegraphics[width=0.04\textwidth]{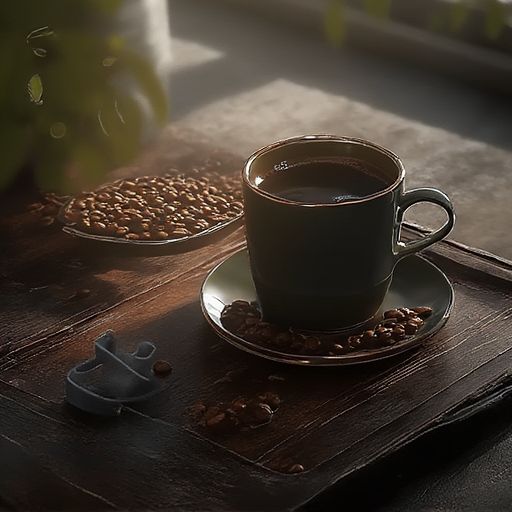}};
	\node (img9) at (8*\imgspacing, 0) {\includegraphics[width=0.04\textwidth]{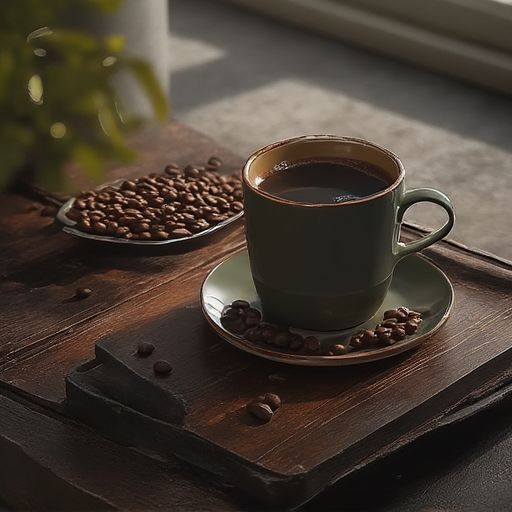}};
	\node (img10) at (9*\imgspacing, 0) {\includegraphics[width=0.04\textwidth]{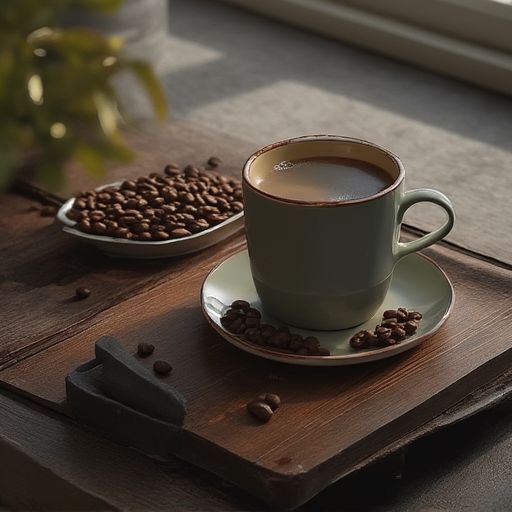}};

	\begin{scope}[shift={(-0.4, -0.5)}]
	\node at (0.35, -0.3) {$\alpha=0$};
	\bigarrow{7.6}{0.1}{0.4}{0.15}{cyan!15}{cyan!75}
	\node at (7.6, -0.3) {$\alpha=1$};
	\end{scope}
	\end{scope}
	
	\begin{scope}[shift={(0,-4.2)}]
\node (img1) at (0*\imgspacing, 0) {\includegraphics[width=0.04\textwidth]{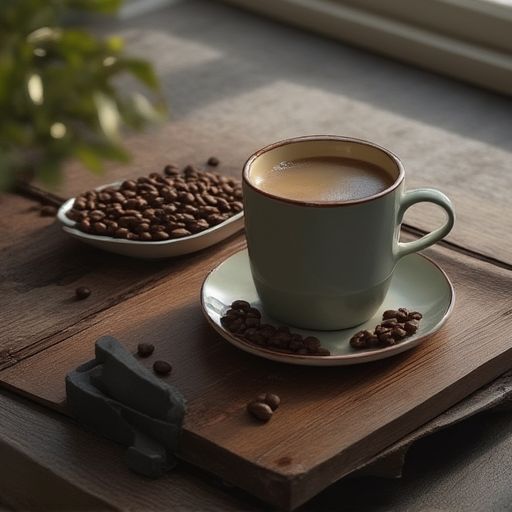}};
	\node (img2) at (1*\imgspacing, 0) {\includegraphics[width=0.04\textwidth]{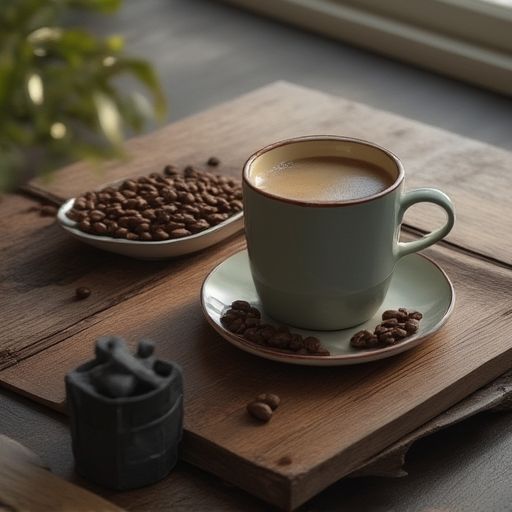}};
	\node (img3) at (2*\imgspacing, 0) {\includegraphics[width=0.04\textwidth]{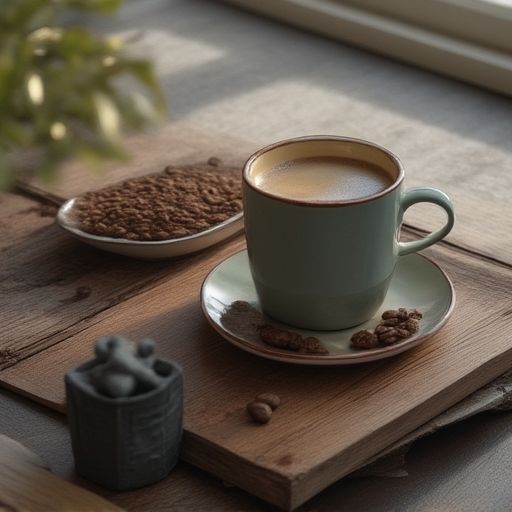}};
	\node (img4) at (3*\imgspacing, 0) {\includegraphics[width=0.04\textwidth]{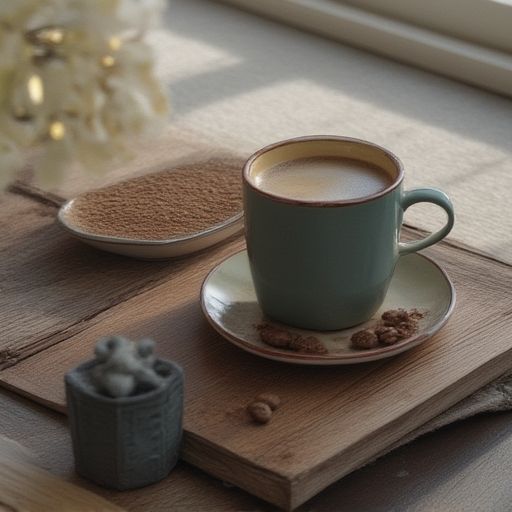}};
	\node (img5) at (4*\imgspacing, 0) {\includegraphics[width=0.04\textwidth]{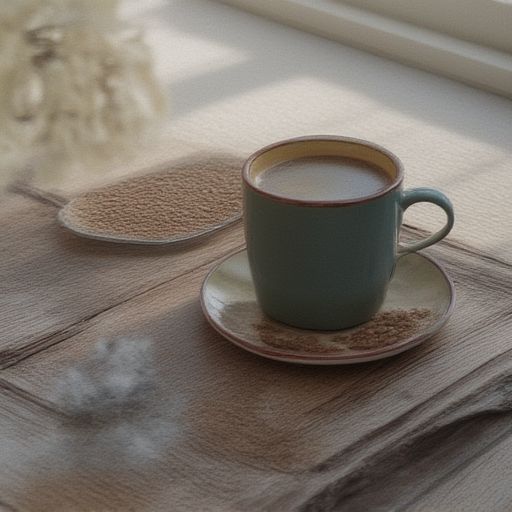}};
	\node (img6) at (5*\imgspacing, 0) {\includegraphics[width=0.04\textwidth]{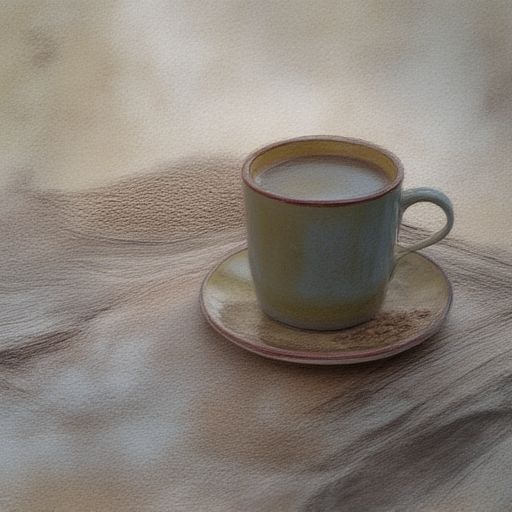}};
	\node (img7) at (6*\imgspacing, 0) {\includegraphics[width=0.04\textwidth]{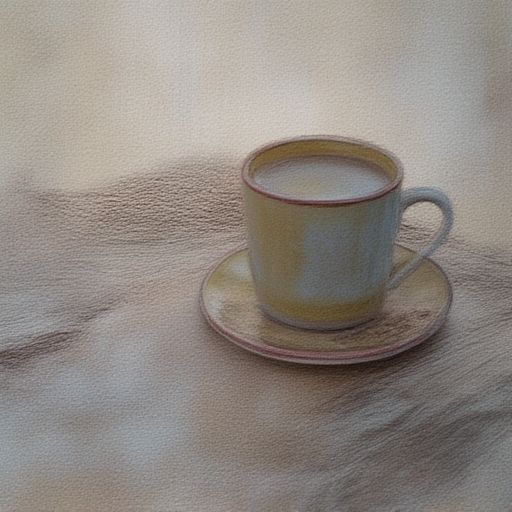}};
	\node (img8) at (7*\imgspacing, 0) {\includegraphics[width=0.04\textwidth]{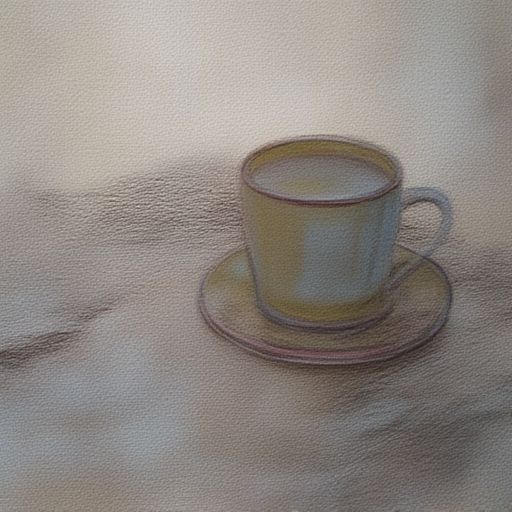}};
	\node (img9) at (8*\imgspacing, 0) {\includegraphics[width=0.04\textwidth]{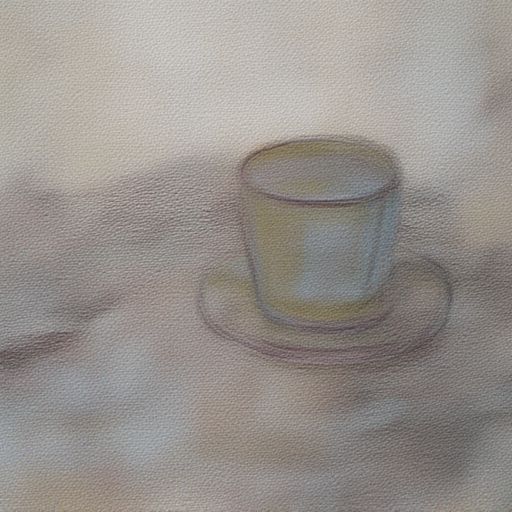}};
	\node (img10) at (9*\imgspacing, 0) {\includegraphics[width=0.04\textwidth]{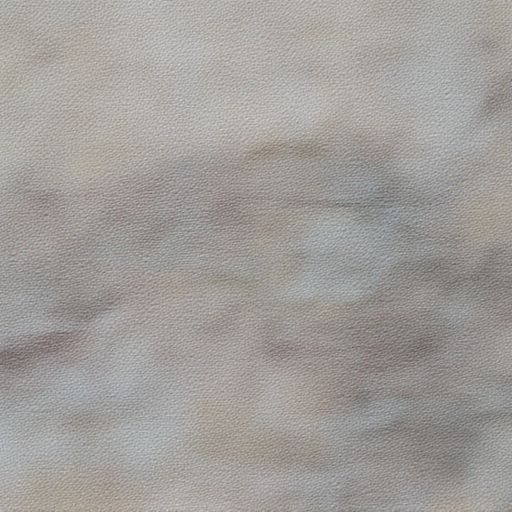}};

	\begin{scope}[shift={(-0.4, -0.5)}]
	\node at (0.35, -0.3) {$\alpha=1$};
	\bigarrow{7.6}{0.1}{0.4}{0.15}{cyan!15}{cyan!75}
	\node at (7.6, -0.3) {$\alpha=2$};
	\end{scope}
	\end{scope}
	
	\node[rotate=90] at (-0.7, -0.8) {\large LoCo};
	
	\node[rotate=90] at (-0.7, -3.3) {\large LoRA};
	
	\end{tikzpicture}	
}
	\caption{Effect of varying the temperature parameter $t$ for LoCo (top two rows) and scaling parameter $\alpha$ for LoRA (bottom two rows) on image generation quality. LoCo was trained using $\vR=(\vI - \vA t)^{-1}(\vI + \vA t)$. During inference, we vary $t$ within the range $[0, 1]$ (first row) and $[1, 2]$ (second row). We observe that for $t \in [1 - l, 1+u]$ for some $l, u$, the generated images maintain high visual quality. Zoom for better view quality. See Figure~\ref{fig:varying_t} and Figure~\ref{fig:varying_t_2_3} in Appendix for better visualization.}
	\label{fig:varying_t_main}
\end{figure}
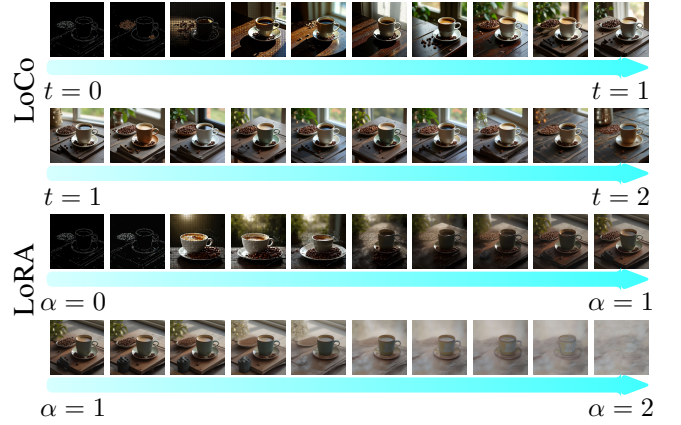

\begin{figure}[h]
	\begin{tikzpicture}
	\node at (0, 0) {\includegraphics[width=0.45\textwidth, height=0.1\textwidth]{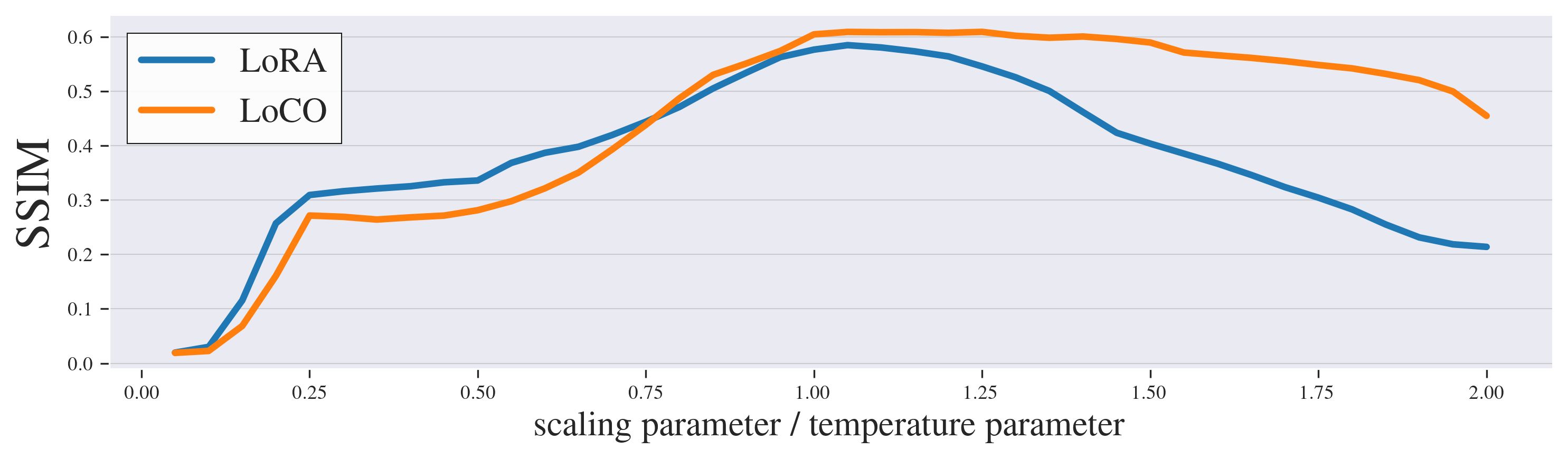}};
	
	\node at (0, -1.8) {\includegraphics[width=0.45\textwidth, height=0.1\textwidth]{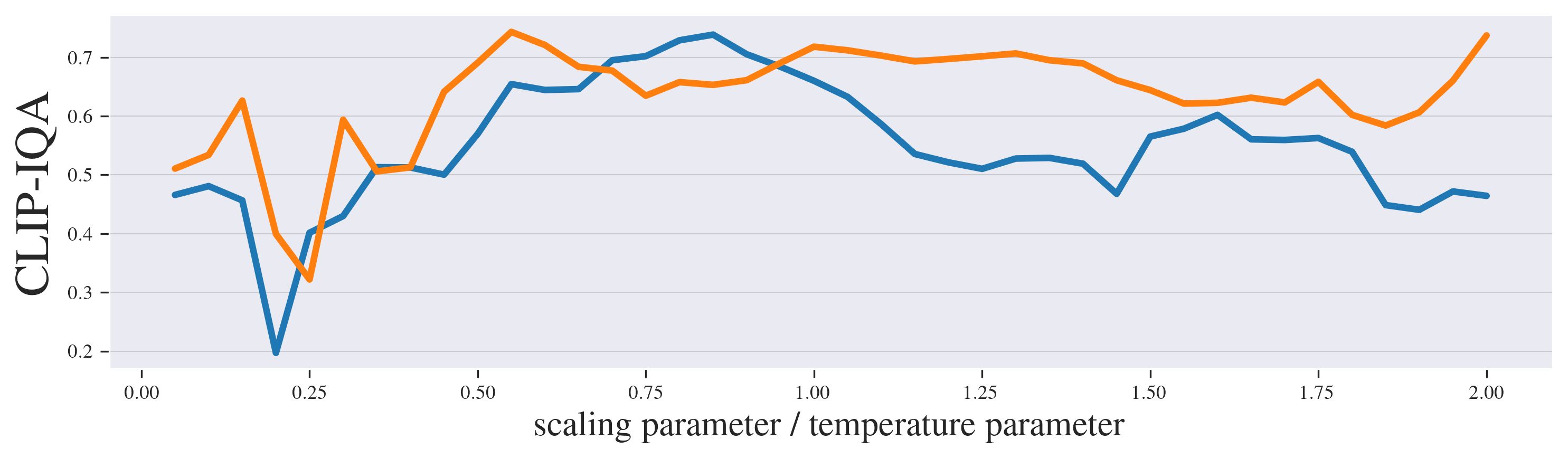}};
	
	\node at (0, -3.6) {\includegraphics[width=0.45\textwidth, height=0.1\textwidth]{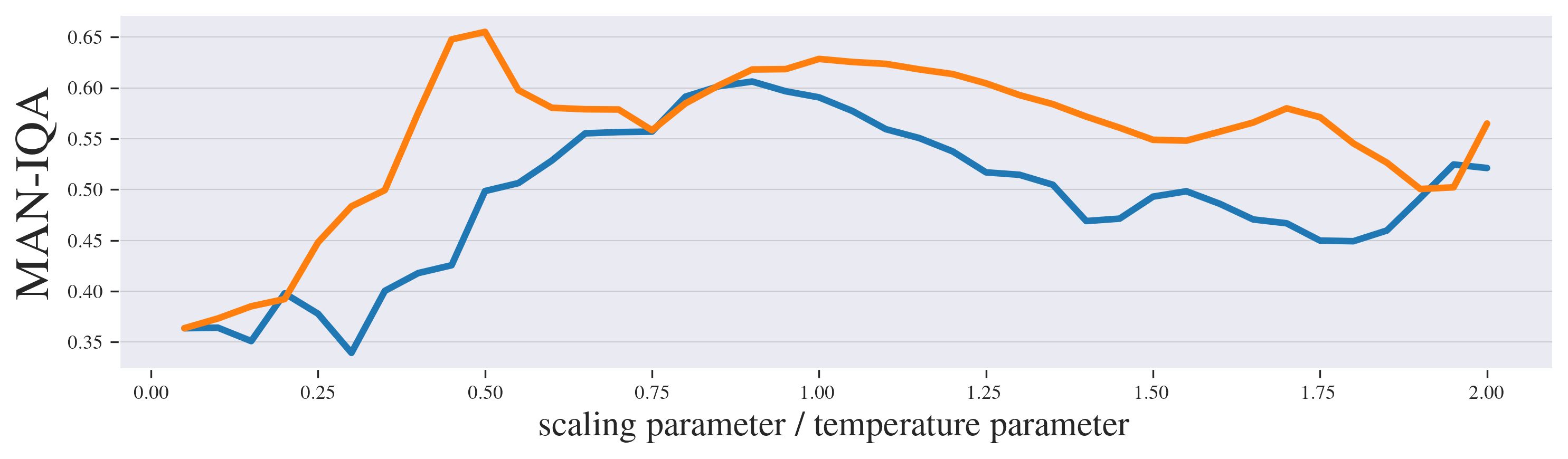}};
	
	\node at (0, -5.4) {\includegraphics[width=0.45\textwidth, height=0.1\textwidth]{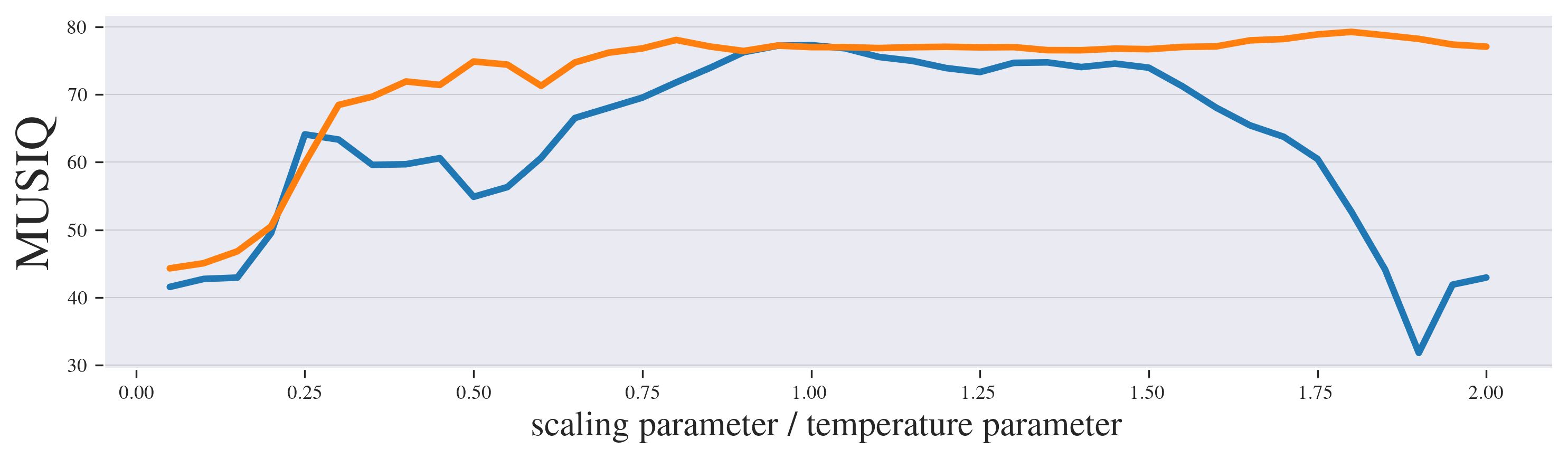}};

	\end{tikzpicture}
	\caption{Image quality indicated by SSIM, CLIP-IQA, MAN-IQA, and MUSIQ criteria of the ``coffee'' sample as the scaling parameter $\alpha$ or temperature parameter $t$ varies. We observe that the generated images by LoCO maintain high quality over a wider range than those by LoRA.}
	\label{fig:vary_t_2}
\end{figure}

\paragraph{Effect of the number of rotations $n$.} We ablate the number of rotations (using values of 2, 4, 6, 8) on the Canny edge detection setting. Table~\ref{tab:ablate_num_rotations} shows the results. We observe that the model benefits from a higher number of rotations, which increases the expressiveness of the LoCO adapter. However, since each rotation approximation can be computed in parallel, GPU memory increases linearly with the number of rotations.

\begin{table}[h]
	\centering
	\renewcommand{\arraystretch}{1.2}
	\resizebox{0.45\textwidth}{!}{
	\begin{tabular}{c |cccc}
		\hline
		 \#Rotations & FID $\downarrow$ & SSIM $\uparrow$ & CLIP-IQA $\uparrow$ & MAN-IQA $\uparrow$ \\
		\hline
		 2 & 33.66 & 0.37 & 0.38 & 0.51 \\
		 4 & 32.91 & 0.36 & 0.42 & 0.55 \\
		 6 & 27.12 & 0.42 & 0.59 & 0.59 \\
		 8 & 25.00 & 0.44 & 0.66 & 0.61 \\
		\hline
	\end{tabular}
	}
	\caption{Impact of the number of rotations $n$ on LoCO performance given a fixed rank $r$.}
	\label{tab:ablate_num_rotations}
\end{table} 
\subsection{Fine-tuning Vision Transformers (ViTs)}

\paragraph{Dataset.} We conduct experiments fine-tuning vision models on the Visual Task Adaptation Benchmark (VTAB). Specifically, we employ VTAB-1k~\cite{vtab}, which comprises 19 diverse visual classification tasks organized into three distinct groups: Natural (tasks featuring natural images), Specialized (medical and satellite imagery), and Structured (tasks requiring geometric understanding). In the fine-tuning setting, the number of training samples for each task is constrained to 1,000 examples. The test dataset for evaluation remains consistent with the original VTAB test set.

\paragraph{Pretrained model.} We use ViT-B/16~\cite{vit} as our main backbone, which has been pretrained on the large-scale ImageNet-21K dataset.

\paragraph{Baselines.} We compare our method against two distinct categories of baselines: non-orthogonal transformation methods and orthogonal adaptation methods. For non-orthogonal baselines, we include LoRA~\cite{hu2022lora}, which decomposes weight updates using low-rank matrices; Adapter~\cite{adapter}, which inserts bottleneck modules; Scale-and-Shift Feature (SSF)~\cite{ssf}, which performs lightweight feature transformations; and Visual Prompt Tuning (VPT)~\cite{vpt}, which learns task-specific prompt embeddings. For orthogonal fine-tuning baselines, we compare with BOFT~\cite{liu2024boft} and HRA~\cite{yuan2024bridging}. We perform a hyperparameter search for the learning rate of all models from the set $[\num{5e-5}, \num{1e-4}, \num{5e-4}, \num{1e-3}]$ and report the best test performance for each method.

\paragraph{Results.} 
Figure~\ref{fig:vtab} shows performance across models. LoCO achieves competitive accuracy across all three VTAB-1k benchmark categories, with the highest average among orthogonal PEFT methods on structured tasks (63.4\%) while remaining competitive on specialized tasks (86.4\%) and natural tasks (82.3\%).
On natural images, LoCO obtains strong results on Caltech101 (95.6\%), DTD (71.4\%), and Flowers102 (99.4\%). Notably, LoCO shows substantial gains on out-of-distribution tasks: Patch Camelyon reaches 88.1\% versus BOFT's 83.2\% (+4.9\%), demonstrating effective transfer to medical imaging. Performance on EuroSAT (95.6\%) and RESISC45 (85.4\%) further confirms robustness on specialized domains. Detailed per-dataset results are provided in Table~\ref{tab:vtab} (Appendix~\ref{appendix:vit}).

\begin{figure}
	\centering
	\includegraphics[width=1.1\linewidth]{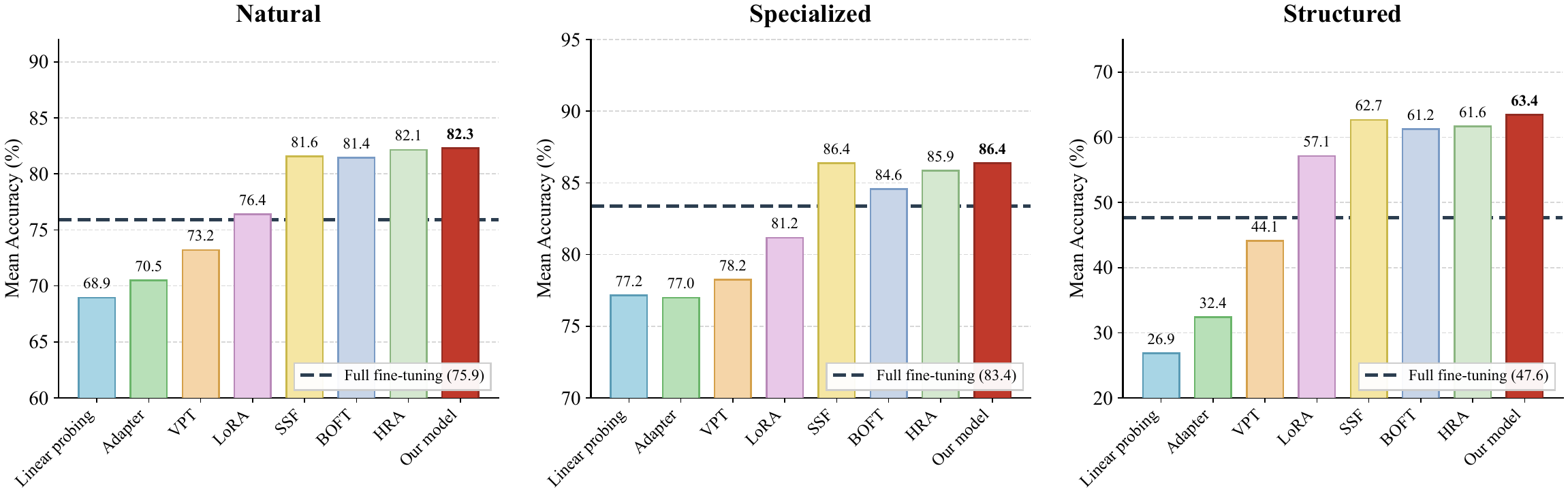}
	\caption{Performance comparison on VTAB-1k benchmark across Natural, Specialized, and Structured task groups. LoCO consistently achieves competitive or superior accuracy compared to both non-orthogonal and orthogonal fine-tuning baselines.}
	\label{fig:vtab}
\end{figure} \section{Conclusion}

We introduce a low-rank compositional orthogonal fine-tuning (LoCO), a parameter-efficient method that adapts foundation models through orthogonal transformations. By constructing orthogonal matrices via low-rank skew-symmetric parameterization and leveraging the Cayley transform with the Sherman-Morrison-Woodbury identity, LoCO reduces computational complexity while allowing fully parallel computation of compositional rotations. Experiments across natural language understanding, mathematical reasoning, diffusion transformers, and vision transformers demonstrate consistent improvements over existing orthogonal fine-tuning methods.

\paragraph{Limitations and future work.} While LoCO's low-rank formulation addresses scalability for typical fine-tuning scenarios, high-dimensional settings may benefit from further optimization. Future directions include theoretical analysis of optimal rotation decompositions, extending temperature scaling to advanced applications such as combining multiple adaptations with multiple scaling factors, and application to multimodal models.

\newpage
\newpage
\section*{Acknowledgements}
This work was supported by the Institute of Information \& Communications Technology Planning \& Evaluation (IITP) grant funded by the Korea government (MSIT) (No. RS-2019-II190079, Artificial Intelligence Graduate School Program (Korea University)), the National Research Foundation of Korea (NRF) grant funded by the Korea government (MSIT) (No. RS-2025-23525509), and the High-Performance Computing Support Project, funded by the Government of the Republic of Korea (Ministry of Science and ICT) (No. RQT-25-070263).

\bibliographystyle{named}

\newpage
\appendix

\crefalias{section}{appendix} \crefalias{subsection}{appendix}

\ifarxiv
    \setcounter{page}{1}

\section{Theoretical Analysis of LoCO}
\subsection{Time complexity among orthogonal approaches}
This section aims to provide a comprehensive comparison of orthogonal fine-tuning methods based on their structural properties, time and space complexity. In particular, we compare LoCO with several contemporary methods, namely: OFT~\cite{qiu2023controlling}, HRA~\cite{yuan2024bridging}, BOFT~\cite{liu2024boft} -- with respect to their theoretical methodology and actual implementations. While all the approaches fundamentally target to construct an orthogonal $d \times d$ transformation matrix to preserve the geometric structure of pretrained representations, they diverge significantly in the parameterization strategies and computational trade-offs. Here we concentrate on the rotation only: $\vz  = \vx \vR , \vx \in \mathbb{R}^{N \times d}$.

\paragraph{OFT:} uses a block-diagonal matrix to represent the orthogonal transformation, each with block size $b \times b$. Applying the Cayley transformation on each block requires $\mathcal{O}(b^3)$, then the total complexity to build a tranformation given input $\vx$ is: $\mathcal{O}(db^2+Ndb)$. Furthermore, the small-block operations on GPUs often suffer from underutilized hardware regarding the large-scale GEneral Matrix to Matrix Multiplication (GEMM) operations. The activation and gradient memory cost is $\mathcal{O}(N \cdot d + d \cdot b)$.

\paragraph{HRA:} parameterizes orthogonal updates as a product of $r$ Householder reflections. Theoretically, by utilizing the WY representation $R = I - VTV^\top$, the input transformation can be optimized to $\mathcal{O}(N \cdot d \cdot r + N \cdot r^2)$, though this entails an additional construction cost of $\mathcal{O}(d \cdot r^2)$ for the matrix $T$. In contrast, LoCO's additive formulation bypasses this $d$-dependent construction overhead. While LoCO incurs a minimal inversion cost of $\mathcal{O}(n \cdot r^3)$, this term is asymptotically smaller than HRA's construction overhead since $r \ll d$. The gradient and activation memory with the WY representation are $\mathcal{O}(r^2)$, and $\mathcal{O}(N \cdot r)$, respectively, resulting to final complexity of $\mathcal{O}(Nr + r^2)$.

More crucially, practical implementations of HRA (e.g., in the \textit{peft} library) often default to a weight-centric approach, sequentially materializing a dense $d \times d$ matrix via $r$ matrix-vector products, resulting in a practical complexity of $\mathcal{O}(r \cdot d^2)$. This materialization further leads to a prohibitive training memory footprint of $\mathcal{O}(d^2 + N \cdot d)$, where $\mathcal{O}(d^2)$ denotes the dense materialization of the reflection matrix within the computational graph, and $\mathcal{O}(Nd)$ is for the input $\vx$ activation.
Conversely, LoCO’s parallelized input-centric design avoids $d \times d$ instantiation entirely, strictly bounding its training memory to $\mathcal{O}(n \cdot r \cdot (N+d))$. This makes LoCO significantly more scalable for high-dimensional models and long-context training where HRA’s quadratic dependency on $d$ becomes a critical bottleneck.

\paragraph{BOFT:} Theoretically, BOFT parameterizes the orthogonal transformation using sparse butterfly factors, suggesting an optimal parameter memory complexity of $\mathcal{O}(m \cdot d \cdot b)$  and a theoretical total work (FLOPs) $\mathcal{O}(m \cdot N \cdot d \cdot b )$ due to the cost of $\mathcal{O}(db)$ for a single matrix multiplication. However, \textit{in practice}, preserving such fine-grained sparsity and sequential permutations leads to scattered memory accesses, which are inefficient on modern GPU architectures optimized for dense matrix multiplications~\cite{dao2022monarch}. Forced by this hardware limitation, the official implementation explicitly abandons sparse execution in favor of a weight-centric strategy. It sequentially accumulates these $m$ factors ($R = B_1 B_2 \dots B_m$) into a dense identity matrix to construct the full $d \times d$ rotation matrix before applying it to the input $x$. Consequently, the practical forward time complexity (Total Work) degrades to $\mathcal{O}(m \cdot d^2 \cdot b + N \cdot d^2)$. Specifically, the term $\mathcal{O}(m \cdot d^2 \cdot b)$ arises from $m$ sequential sparse-dense matrix multiplications to materialize the $d \times d$ matrix, while $\mathcal{O}(N \cdot d^2)$ accounts for the final dense projection of $N$ tokens. 
More critically, this dense materialization inflates the training memory footprint to $\mathcal{O}(m \cdot d^2 + N \cdot d)$. While $\mathcal{O}(N \cdot d)$ denotes the standard activation memory for the input features, the prohibitive $\mathcal{O}(m \cdot d^2)$ term arises from caching the intermediate dense matrices in the computational graph for backward gradients, imposing severe VRAM bottlenecks.

Furthermore, because modern Tensor Cores are strictly optimized for contiguous, large-block dense GEMMs, executing batched operations on exceptionally small block sizes ($b$) suffers from severe memory bandwidth bottlenecks and hardware padding overheads. Consequently, the theoretical computational savings of $b \ll d$ fail to translate into practical latency reductions, rendering the approach computationally intensive in real-world deployments~\cite{dao2022monarch}.

\paragraph{LoCO:} adopts input-centric transformation and circumvents prohibitive $d \times d$ dense matrices. Serial compositional adapters theoretically require a total computational work of $\mathcal{O}(n \cdot N \cdot d \cdot r + n \cdot r^3)$, where $\mathcal{O}(nNdr)$ and $\mathcal{O}(nr^3)$ denotes the cost of matrix-vector multiplication and matrix inversion, respectively (See~\Cref{eq:actual_adapter}). By virtue of full parallelism, LoCO reduces its sequential span to a depth of $\mathcal{O}(Ndr+r^3)$. Regarding activation memory, the static memory for parameter space is $\mathcal{O}(n \cdot d \cdot r)$. In addition, the framework only needs to cache $n$ independent projection tensors of size $N \times r$, without any $d \times d$ intermediate matrices, resulting in a memory footprint of $\mathcal{O}(n \cdot (d+N) \cdot r)$, where $\mathcal{O}(ndr)$ and $\mathcal{O}(nNr)$ represent gradients and activations memory, respectively.

In addition, in the naive parallel input-centric approach, we need to store $n$ intermidate tensors of size $N \times d$. The total memory for activations and gradients becomes: $\mathcal{O}(n \cdot N \cdot d) + \mathcal{O}(n \cdot d \cdot r)$. This significant increase in training overhead is \emph{eliminated} by serial compositional LoCO.

We further summarize the comparison results among 4 orthogonal fine-tuning methods in \Cref{tab:comprehensive_comparison}.

\begin{table*}[t]
\centering
\caption{Comprehensive Complexity and Structural Comparison of Orthogonal Fine-Tuning Methods, with respect to the analysis and notations above. Time complexity is decoupled into Total Work (FLOPs) and Sequential Span (Critical path length -- Latency). For sequential methods (HRA, BOFT), Span scales with the number of components, whereas LoCO achieve a constant Span relative to $n$ due its parallel additive structure. For HRA, we include WY-representation as the theoretically optimal strategy.}
\label{tab:comprehensive_comparison}
\renewcommand{\arraystretch}{1.8}
\resizebox{\textwidth}{!}{\begin{tabular}{@{} >{\raggedright\arraybackslash}m{1.2cm} >{\raggedright\arraybackslash}m{2.2cm} l l l l >{\raggedright\arraybackslash}m{5.7cm} @{}}
\toprule
\textbf{Method} & \textbf{Strategy} & \textbf{\#Params} & \textbf{Total Work (FLOPs)} & \textbf{Sequential Span} & \makecell[l]{\textbf{Memory} \\ \textbf{(Act. + Grad.)}} & \textbf{Remarks} \\ \midrule

\textbf{OFT} & Input-centric & $\mathcal{O}(db)$ & $\mathcal{O}(Ndb + db^2)$ & $\mathcal{O}(Ndb + db^2)$ & $\mathcal{O}(Nd + db)$ & Sparse block-diagonal; single-layer; local subspace interaction only. \\ \cmidrule(lr){1-7}

\textbf{BOFT} & \makecell[l]{Weight-centric \\ (Practical)} & $\mathcal{O}(mdb)$ & $\mathcal{O}(md^2b + Nd^2)$ & $\mathcal{O}(md^2b + Nd^2)$ & $\mathcal{O}(md^2 + Nd)$ & Sequential chain; $d^2$ bottleneck during dense materialization, negating theoretical FLOPs. \\ \cmidrule(lr){1-7}

\textbf{HRA} & WY-Refined & $\mathcal{O}(dr)$ & $\mathcal{O}(Ndr + dr^2 + Nr^2)$ & $\mathcal{O}(rd + Ndr + r^2)$ & $\mathcal{O}(Nr + dr)$ & High-depth Householder chain due to vector-based reflection; In their practical execution, sequential transformations impose $\mathcal{O}(d^2)$ overhead. \\ \cmidrule(lr){1-7}

\textbf{LoCO (ours)} & \makecell[l]{{Parallel} \\ {Input-centric}} & ${\mathcal{O}(ndr)}$ & $\mathcal{O}(nNdr + nr^3)$ & $\mathbf{\mathcal{O}(Ndr + r^3)}$ & ${\mathcal{O}(nr(N + d))}$ & Orthogonal approximation; parallel additive structure; shallow depth $\left( n \in \{1, 2, 4\} \right)$. \\ \bottomrule
\end{tabular}}
\end{table*}

\subsection{Theoretical analysis on approximation error}
\label{sec:app_error}
In this section, we provide detailed proofs of the upper bound of approximation error in \Cref{sec:model_method}.

\begin{theorem}[Upper Bound of Approximation Error]
    \label{thm:upper_bound_dev}
    Let $\rmZ_i = \Delta_i = 2\rmX_i(\rmI-\rmY_i^\top \rmX_i)^{-1}\rmY_i^\top$ denote the update components in \Cref{sec:model_method}. 
    
    Consider the generic expansion of a rotation chain $\mathbf{R} = \prod_{i=1}^{n} (\rmI + \rmZ_i)$, where $(\rmI + \rmZ_i) \in \textnormal{SO}(d)$. Its first-order parallel approximation is defined as $\tilde{\rmR} = \rmI + \sum_{i=1}^{n} \rmZ_i$. 
    
    If the Frobenius norm of each component is bounded by $\left\lVert \rmZ_i \right\rVert_F \leq \gamma$, then the orthogonal deviation of the parallel approximation is strictly bounded by:
    \begin{equation}
        \left\lVert \rmI - \tilde{\rmR}^\top \tilde{\rmR} \right\rVert_F \leq n(n-1) \gamma^2
    \end{equation}
\end{theorem}

\begin{proof}
    Consider the orthogonal deviation: $\Delta \rmR_O = \rmI - \tilde{\rmR}^\top \tilde{\rmR}$.
    
Since $\rmR_i = \rmI + \rmZ_i$ is orthogonal, we have $(\rmI + \rmZ_i)^\top (\rmI + \rmZ_i) = \rmI$, which yields:
    \begin{equation}
        \label{eq:zi_identity}
        \rmZ_i^\top + \rmZ_i = -\rmZ_i^\top \rmZ_i
    \end{equation}
    
Let $\rmS = \sum_{i=1}^n \rmZ_i$. The Gram matrix of the approximation expands as:
    \begin{gather}
        \tilde{\rmR}^\top \tilde{\rmR} = (\rmI +\rmS^\top)(\rmI +\rmS) = \rmI + \rmS^\top + \rmS + \rmS^\top \rmS \notag \\[1ex]
        \Rightarrow \Delta \rmR_O = -[\rmS^\top \rmS + (\rmS^\top + \rmS)]
    \end{gather}
    
By substituting \Cref{eq:zi_identity} into the linear term $(\rmS^\top + \rmS)$, we obtain:
    \begin{align}
        \Delta \rmR_O &= - \left( \sum_{i=1}^n \sum_{j=1}^n \rmZ_i^\top \rmZ_j - \sum_{i=1}^n\rmZ_i^\top \rmZ_i \right) \notag\\[1ex]
         &= - \sum_{i \neq j} \rmZ_i^\top \rmZ_j
    \end{align}
    
Applying the triangle inequality and the submultiplicativity of the Frobenius norm ($\lVert \rmA\rmB \rVert_F \leq \lVert \rmA \rVert_F \lVert \rmB \rVert_F$), we have:
    \begin{align}
        \left\lVert \Delta \rmR_O \right\rVert_F &= \left\lVert \sum_{i \neq j} \rmZ_i^\top \rmZ_j \right\rVert_F \leq \sum_{i \neq j} \left\lVert  \rmZ_i^\top \rmZ_j \right\rVert_F\notag \\[1ex]
        \Rightarrow \left\lVert \Delta \rmR_O \right\rVert_F &\leq \sum_{i \neq j} \left\lVert \rmZ_i^\top \right\rVert_F \left\lVert \rmZ_j \right\rVert_F 
    \end{align}
    
Since $\left\lVert \rmZ_i^\top \right\rVert_F = \left\lVert \rmZ_i \right\rVert_F \leq \gamma$, and there are exactly $n(n-1)$ terms where $i \neq j$, we conclude:
    \begin{equation}
        \left\lVert \rmI - \tilde{\rmR}^\top \tilde{\rmR} \right\rVert_F \leq n(n-1) \gamma^2
    \end{equation}
    This completes the proof.
\end{proof}
\Cref{thm:upper_bound_dev} shows the upper bound of deviation error in orthogonal approximation, providing the formal justification for LoCO's parallelized architecture. It mainly includes these properties below to advocate the effectiveness of the parallel mechanism:
\begin{itemize}
	\item \textbf{Mathematical Validation of Parallelism:} The theorem justifies replacing sequential rotation products with a parallel additive structure by proving the error remains bounded.
	\item \textbf{Guaranteed Orthogonality Approximation:} The upper bound $n^2\gamma^2$ ensures ``quasi-orthogonality'': for small update magnitude $\gamma$, the orthogonal properties are preserved.
	\item \textbf{Optimized Design Space:} The quadratic relationship with $n$ suggests that for small $n$, the deviation is minimal, providing a theoretical "safe zone" for the shallow configurations -- $n \in \{1,2,4\}$ -- of the LoCO architecture.
\end{itemize}

In addition, we provide empirical measurements of orthogonality deviation, and Frobenius norm of $\rmZ_i$ from the fine-tuned checkpoints in \Cref{fig:orthogonal_deviation}. The orthogonality deviation (~\Cref{fig:orthogonal_deviation}a) remains relatively small across all layers, and the perturbation norm $\|\Delta_i\|_F$ (~\Cref{fig:orthogonal_deviation}b) confirms that learned parameters stay within the regime where the approximation is valid. 

\begin{figure}[h]
	\centering
	\begin{tikzpicture}
	\node at (0,0) {\includegraphics[width=0.18\textwidth]{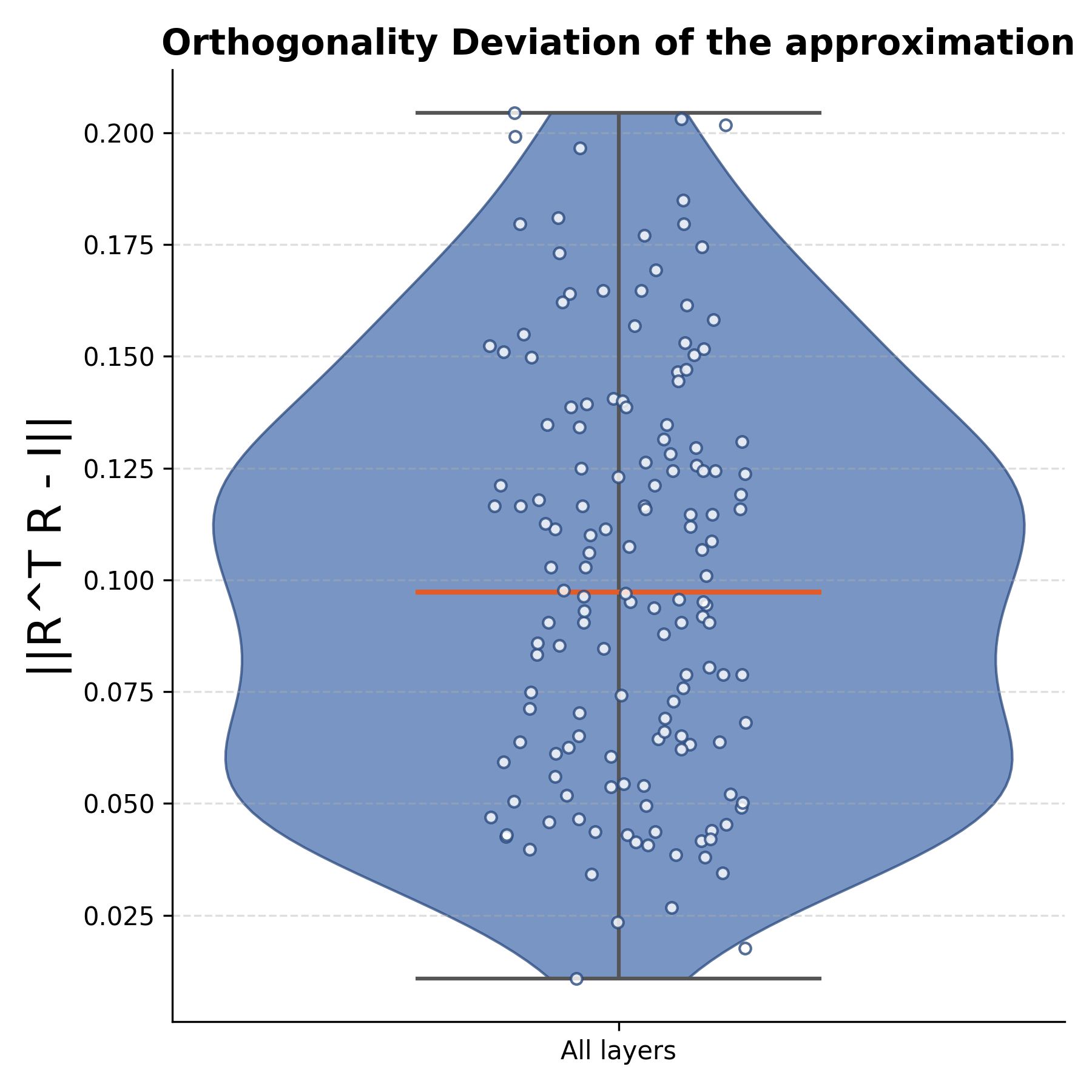}};
	\node at (4,0) {\includegraphics[width=0.18\textwidth]{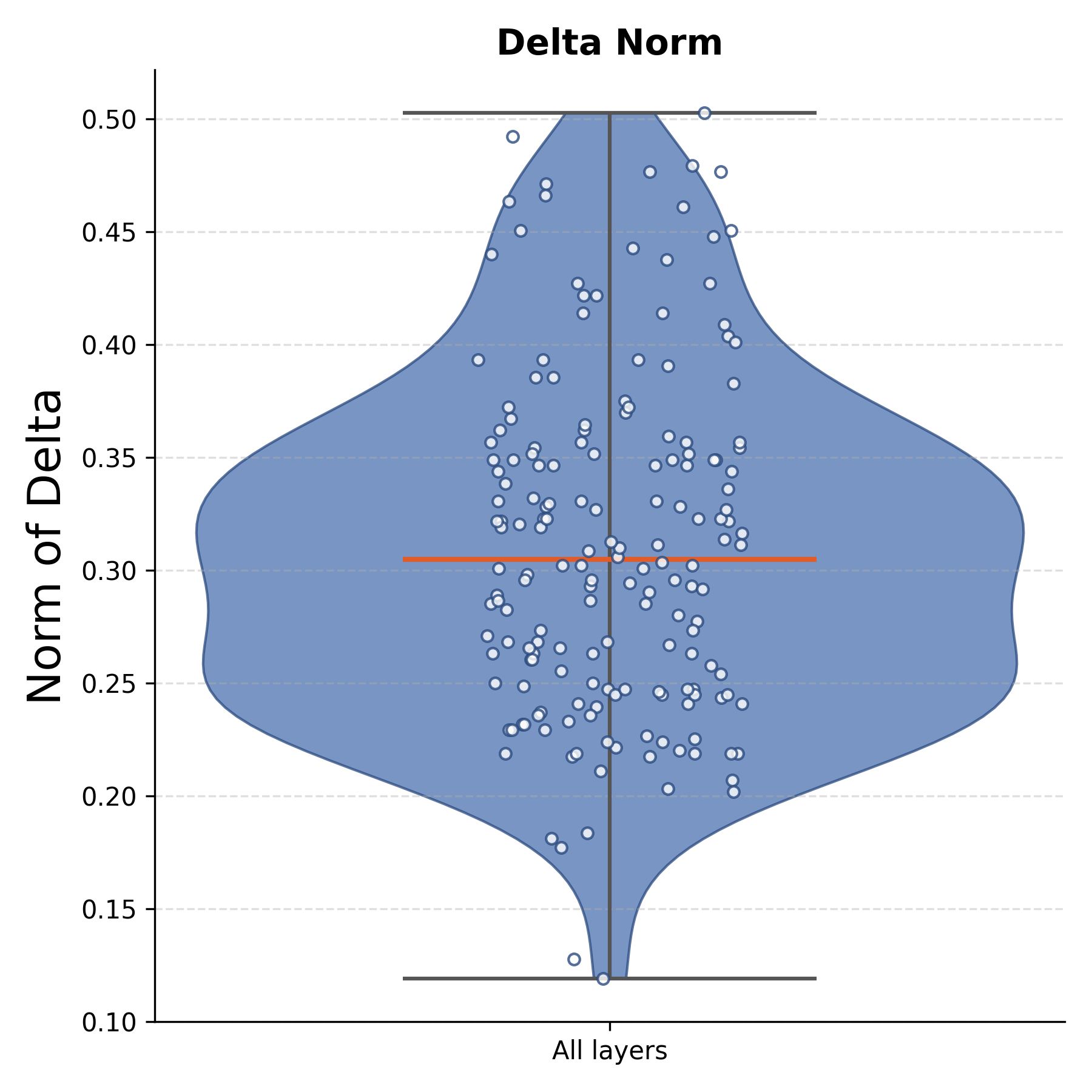}};
	\node at (0, -1.7) {(a)};
	\node at (4, -1.7) {(b)};
	\end{tikzpicture}
	\caption{(a) Orthogonality deviation $\|\tilde{\mathbf{R}}^\top \tilde{\mathbf{R}} - \mathbf{I}\|_F$ extracted from a trained diffusion model. (b) Distribution of perturbation norms $\left\lVert \rmZ_i \right\rVert_F$ or $\left\lVert \Delta_i \right\rVert_F$ from the same checkpoint. Here, $\tilde{\mathbf{R}} \in \mathbb{R}^{3072 \times 3072}$ is a high-dimensional matrix. The deviation norm remains low (mean $\approx 0.1$), indicating that the approximation roughly preserves vector norms under transformation.}
	\label{fig:orthogonal_deviation}
\end{figure}

\section{Experimental Details}

\subsection{Fine-tuning Large language models}
In this section regarding LLMs, we further describe experimental details and illustrate additional outputs on language tasks at~\ref{sec_llm}. They include experiments on natural language understanding/generation, computation efficiency evaluation, and ablation study on the impact of LoCO configurations.

\subsubsection{Natural Language Understanding}
\label{appendix:nlu_exps}

In regard to our experiments on the GLUE benchmark, we only tune the learning rate, and the number of training epochs. We adapt pre-trained DeBERTA-V3~\cite{he2021deberta} as the base model, then apply orthogonal methods to every linear module in each transformer block while freezing the pre-trained weights.

Particularly, we choose 8 datasets from the benchmark, including CoLA, SST-2, MRPC, m-MNLI, QNLI, RTE, STS-B, QQP.

For each specific task, the pooling layers and classification head are trained from scratch with individual classification learning rate. The settings for each dataset and other common factors are provided in Table~\ref{tab:app_llms:loco_nlu_common} and Table~\ref{tab:app_llms:loco_nlu_datasets}, respectively.

\begin{table}[H]
    \centering
    \resizebox{0.70\columnwidth}{!}{\begin{tabular}{l c}
        \toprule
        \textbf{Hyperparameter} & \textbf{Value} \\
        \midrule
        Optimizer & AdamW \\
        Optimizer ($\beta_1, \beta_2$) & (0.9, 0.99) \\
        Learning rate scheduler & Cosine \\
        Classifier learning rate & 2e-3 \\
        LoCO dropout & 0.0 \\
        \bottomrule
        \end{tabular}}
    \caption{Common hyperparameters shared across all GLUE tasks.}
    \label{tab:app_llms:loco_nlu_common}
\end{table} 
In order to do replication of HRA (taking regularization coefficient $\lambda{=}0$) and OFT/BOFT, we follow the settings in their papers~\cite{qiu2023controlling,liu2024boft}, and use a single learning rate value for all trainable parameters. The corresponding configurations are also mentioned as a part of Table~\ref{tab:glue_deberta}.

\subsubsection {Evaluation on mathematical reasoning tasks}
\label{appendix:nlg_math}

In the fine-tuning experiments on the LLaMA2 model, we fix the max sequence length as 512, which is sufficient for these tasks.
Following BOFT~\cite{liu2024boft}, we adapt the LLaMA2-7B model~\cite{touvron2023llama2openfoundation} on the first 512 tokens of MetaMath-40K~\cite{yu2023metamath}. The details are given in Table~\ref{tab:app_llms:metamath}

For the ablation study of different LoCO configurations in~\ref{exp:llm_abation} on LLMs, we only change combinations of $n,r$ and use the same values for other training/inference factors to ensure a fair comparison.
\begin{table}[H]
    \centering
    \setlength{\tabcolsep}{6pt}
    \resizebox{0.65\columnwidth}{!}{\begin{tabular}{l c}
        \toprule
        \textbf{Hyperparameter} & \textbf{Value} \\
        \midrule
        \multicolumn{2}{c}{\textbf{\textit{Training settings}}} \\
        \midrule
        Learning rate & 1e-3 \\
        Optimizer & AdamW \\
        Optimizer ($\beta_1, \beta_2$) & (0.9, 0.99) \\
        Learning rate scheduler & Cosine \\
        LoCO dropout & 0.1 \\
        Batch size & 32 \\
        Epoch & 3 \\
        Applied modules & \textit{Q, V} \\
        \midrule
        \multicolumn{2}{c}{\textbf{\textit{Inference settings}}} \\
        \midrule
        LLM max sequence length & default \\
        GSM new tokens & 1024 \\
        MATH new tokens & 1408 \\
        \bottomrule
        \end{tabular}}
    \caption{Hyperparameters for mathematical reasoning tasks.}
    \label{tab:app_llms:metamath}
\end{table} 
\subsubsection{Ablation study of LoCO configurations on LLMs}
\label{exp:llm_abation}
In this section, we explicitly evaluate influence of LoCo configurations on the final results, concentrating the two aspects: how rank $r$ of skew matrices affect LoCO performace, and what is the trade-off of parallel transformations $n$ in~\ref{model:Compositional Chain-of-Rotation}

\begin{table}[H]
    \centering
    \setlength{\tabcolsep}{6pt}
    \resizebox{0.80\columnwidth}{!}{
    \begin{tabular}{c c c c}
        \toprule
        \textbf{Params (\%)} & \textbf{Config ($n, r$)} & \textbf{GSM8K} & \textbf{MATH} \\
        \midrule
0.016 & $n{=}1, r{=}2$ & 41.75 & 5.80 \\
        0.031 & $n{=}1, r{=}4$ & 45.44 & 6.45 \\
        0.062 & $n{=}1, r{=}8$ & 48.40 & 7.37 \\
        \midrule
\multirow{3}{*}{0.124} 
              & $n{=}1, r{=}16$ & 50.19 & \textbf{8.40} \\
              & $n{=}2, r{=}8$  & \textbf{50.26} & 8.12 \\
              & $n{=}4, r{=}4$  & 49.46 & 7.85 \\
        \bottomrule
    \end{tabular}
    }
    \caption{Ablation study on LoCO configurations.  The first block shows the scaling effect by increasing rank $r$. The second block compares different factorization strategies ($n, r$) under a fixed parameter budget.}
    \label{tab:ablation_llm_nr}
\end{table} 
For the first question, we follow the experiment setting in~\ref{exp:llm_math}, vary the rank $r$. For the second, given a parameter budget equivalent to this experiment, $\approx$ 0.124\%, we apply and compare different configurations of $(n,r)$. The averaged results are reported in Table~\ref{tab:ablation_llm_nr}.

Our results indicate that increasing the rank $r$ consistently yields better performance. Furthermore, under a fixed parameter budget, approximating the transformation via parallel computation for low-rank skew matrices maintains competitive efficacy, demonstrating that computation efficiency benefits do not come at the cost of considerable degradation in performance.

\subsubsection{Computational Efficiency}
\label{appendix:computation_efficiency}

\begin{figure}[htbp]
	\centering
	\includegraphics[width=0.95\linewidth]{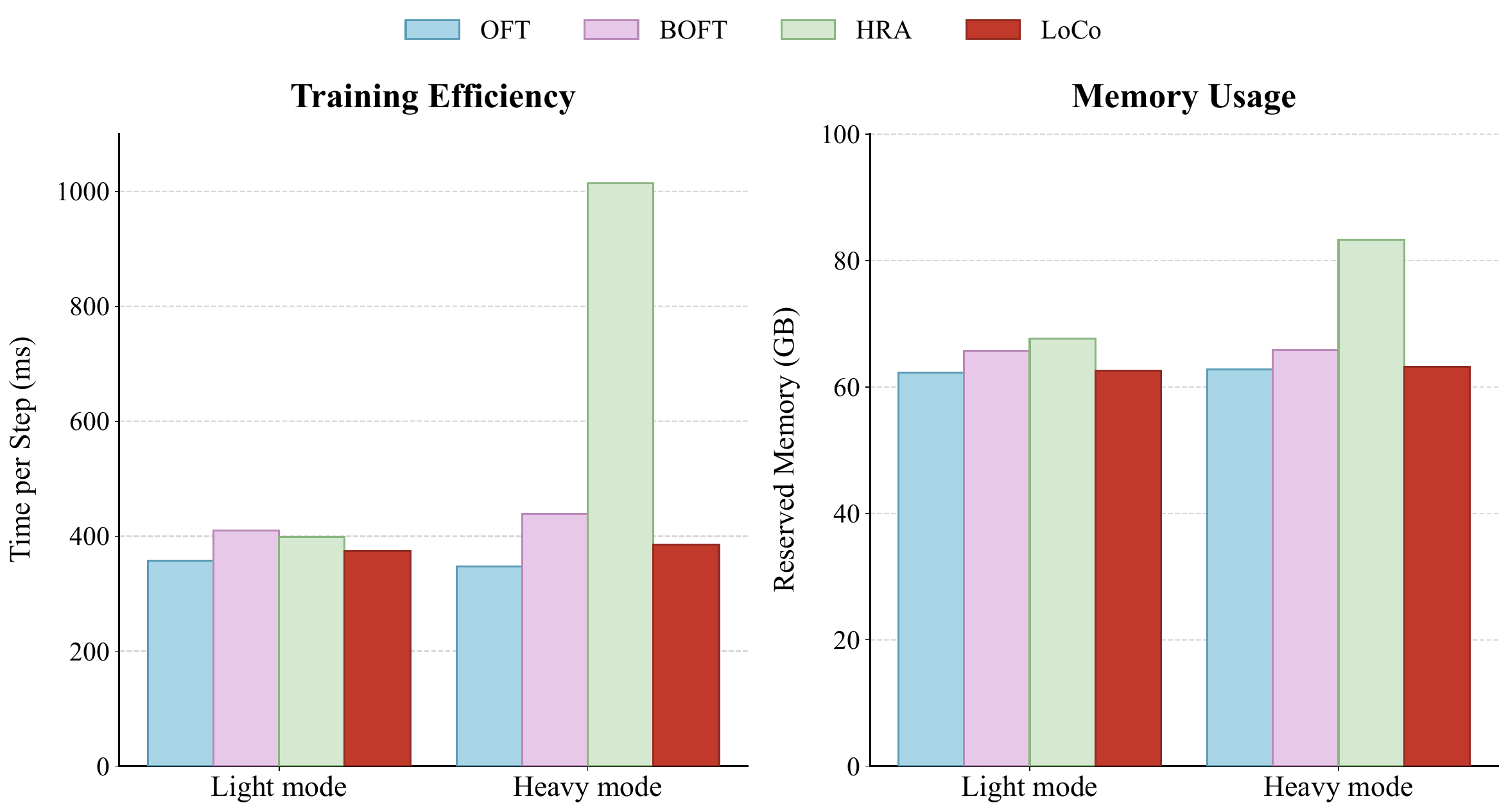}
	\caption{\label{fig:deberta64}Training efficiency comparison on DeBERTA-V3, batch size of 64.}
	
	\vspace{0.4cm}
	\centering
	\includegraphics[width=0.95\linewidth]{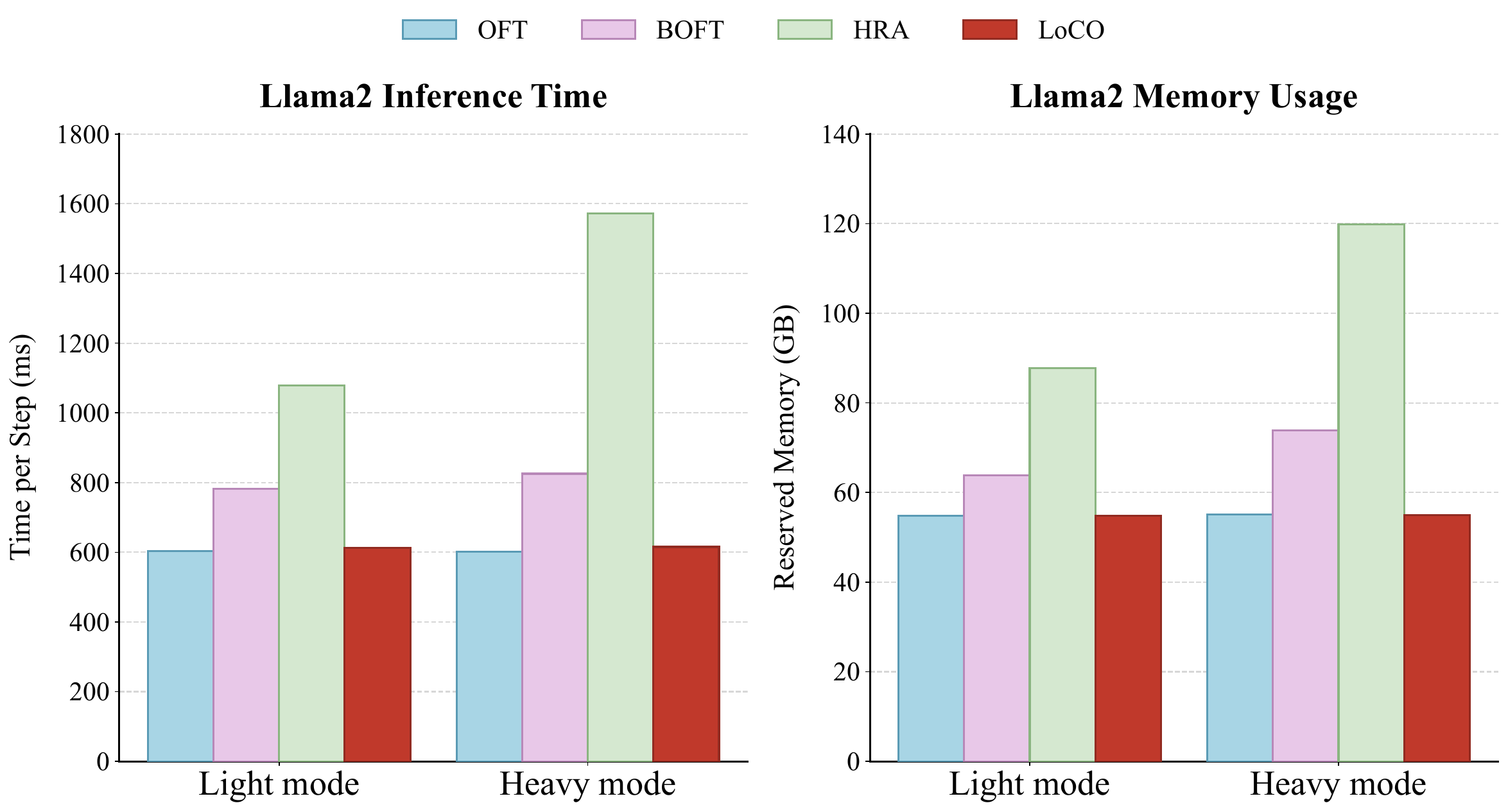}
	\caption{\label{fig:llama2b16}Training efficiency comparison on LLaMA2-7B batch size of 16.}

	\vspace{0.4cm}
	\centering
	\includegraphics[width=0.95\linewidth]{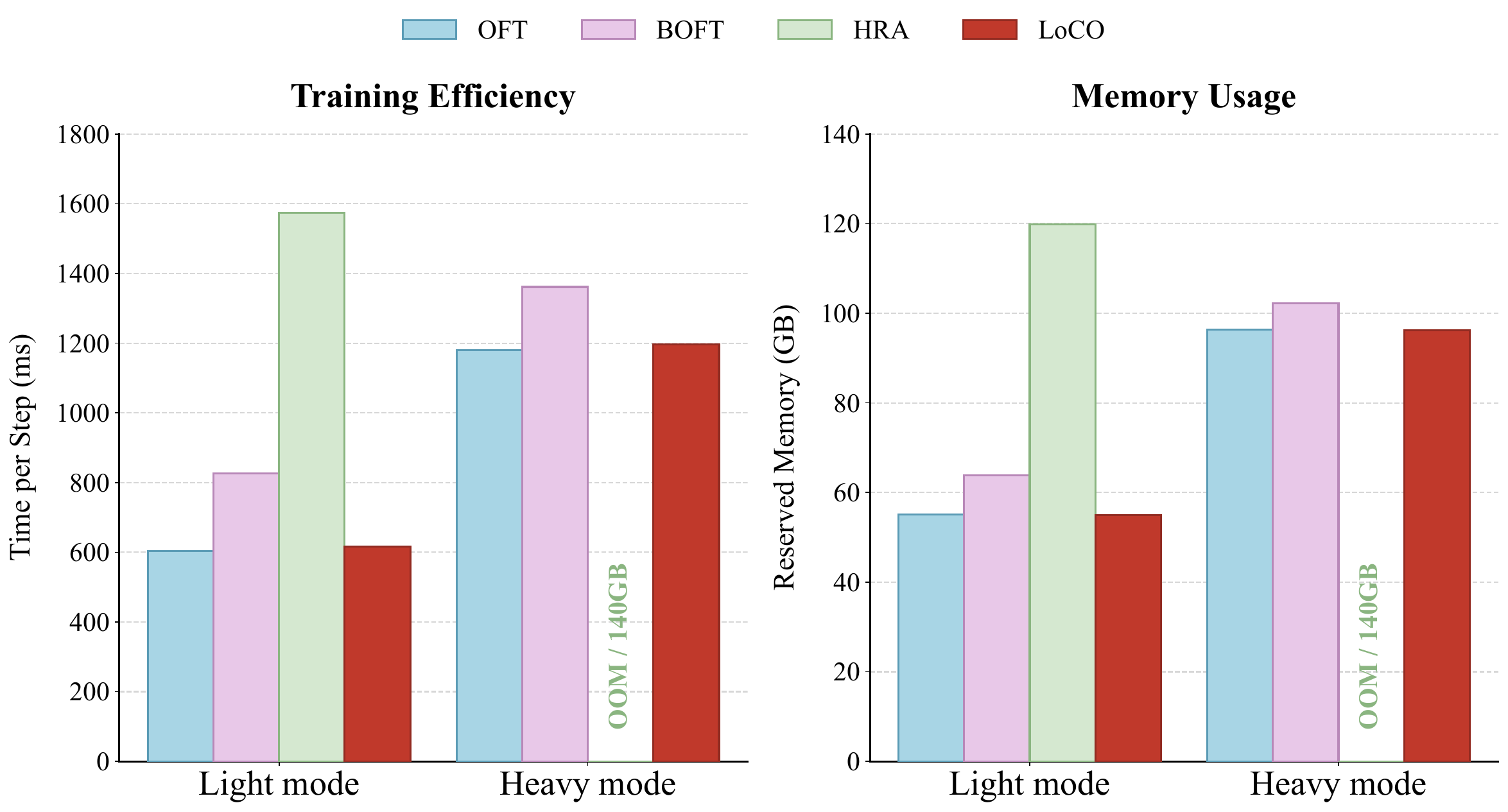}
	\caption{\label{fig:llama2b32} Efficiency comparison during training time on LLaMA2-7B batch size of 32. HRA suffers from Out of memory (OOM) issue.}
\end{figure}

To compare computational efficiency, we align the configurations of all methods to match the number of trainable parameters. We define two settings: \textit{light mode} and \textit{heavy mode}, as computational costs vary significantly depending on specific hyperparameter sets. For instance, the rank $r$ impacts the matrix inversion cost, while the number of Householder reflections determines the recursion depth in HRA.

In each mode, we test with two batch size values to get a comprehensive conclusion. For Natural language understanding tasks we use the QNLI dataset~\cite{rajpurkar2016squad} on the DeBERTA-V3~\cite{he2021deberta}, batch sizes of 16 and 64, with \textit{torch.compile()}; for Natural language generation tasks, we use the MetaMath-40K~\cite{yu2023metamath} dataset on the LLaMA2~\cite{touvron2023llama2openfoundation}, batch sizes of 16 and 32, without \textit{torch.compile()}. These models and datasets are similar to the setting of the primary experiments.

For each model, the detailed descriptions for \textit{light mode} and \textit{heavy mode} are provided in Table~\ref{tab:app_llms:eff_deberta} and Table~\ref{tab:app_llms:eff_llama}. All experiments are conducted on a single NVIDIA H200 GPU with 140GB VRAM.

\begin{table}[h]
    \centering
    \setlength{\tabcolsep}{6pt}
    \resizebox{0.85\columnwidth}{!}{\begin{tabular}{l l l c}
        \toprule
        \textbf{Setting} & \textbf{Method} & \textbf{Config} & \textbf{\#Params (M)} \\
        \midrule
        \multirow{4}{*}{Light mode}
          & OFT  & $b{=}16$      & 0.622 \\
          & BOFT & $m{=}2, b{=}4$  & 0.746 \\
          & HRA  & $r{=}8$       & 0.663 \\
          & LoCo & $n{=}1, r{=}4$  & 0.663 \\
        \midrule
        \multirow{4}{*}{Heavy mode}
          & OFT  & $b{=}64$      & 2.612 \\
          & BOFT & $m{=}2, b{=}1$  & 2.737 \\
          & HRA  & $r{=}32$      & 2.654 \\
          & LoCo & $n{=}1, r{=}16$ & 2.654 \\
        \bottomrule
        \end{tabular}}
    \caption{Detailed configurations for efficiency comparison on DeBERTa-v3.}
    \label{tab:app_llms:eff_deberta}
\end{table}

\begin{table}[h]
    \centering
    \setlength{\tabcolsep}{6pt}
    \resizebox{0.85\columnwidth}{!}{\begin{tabular}{l l l c}
        \toprule
        \textbf{Setting} & \textbf{Method} & \textbf{Config} & \textbf{\#Params (M)} \\
        \midrule
        \multirow{4}{*}{Light mode} 
          & OFT  & $b{=}32$      & 4.063 \\
          & BOFT & $m{=}2, b{=}8$  & 4.456 \\
          & HRA  & $r{=}16$      & 4.194 \\
          & LoCo & $n{=}1, r{=}8$  & 4.194 \\
        \midrule
        \multirow{4}{*}{Heavy mode} 
          & OFT  & $b{=}64$      & 8.257 \\
          & BOFT & $m{=}2, b{=}1$  & 8.651 \\
          & HRA  & $r{=}32$      & 8.388 \\
          & LoCo & $n{=}1, r{=}16$ & 8.388 \\
        \bottomrule
        \end{tabular}}
    \caption{Detailed configurations for efficiency comparison on LLaMA2-7B.}
    \label{tab:app_llms:eff_llama}
\end{table} 
\begin{table}[ht]
    \centering
    \resizebox{0.90\columnwidth}{!}{\begin{tabular}{l c c c c}
    \toprule
    \textbf{Method} & \textbf{Config} & \textbf{\#Params} & \textbf{Time (ms)} & \textbf{Mem (GB)} \\
    \midrule
    OFT & $b{=}128$ & 16.65 M & 151.3 & 15.97 \\
    \midrule
    
    \multirow{3}{*}{BOFT} 
        & $m{=}2,b{=}32$ & \multirow{3}{*}{17.04 M} & 303.7 & 57.93 \\
        & $m{=}4,b{=}16$ &                         & 466.0 & 37.68 \\
        & $m{=}8,b{=}8$  &                         & 879.8 & 98.37 \\
    \midrule
    
    HRA & $r{=}64$ & 16.78 M & \textbf{OOM} & \textbf{OOM} \\
    \midrule
    
    \multirow{3}{*}{\textbf{LoCO}} 
        & $n{=}1,r{=}32$ & \multirow{3}{*}{16.78M} & 173.3 & 15.65 \\
        & $n{=}4,r{=}8$  &                         & 171.0 & 15.71 \\
        & $n{=}2,r{=}16$ &                         & 175.3 & 15.74 \\
    \bottomrule
    \end{tabular}
    }
\caption{Efficiency comparison on an extreme configuration that is equivalent to $r=64$ vectors in $\mathbb{R}^d$. \textbf{OOM}: Out Of Memory.}
    \label{tab:efficiency_compact}
\end{table} 
Moreover, we also illustrate additional results beside Figure~\ref{fig:deberta}, in the setting of DeBERTA-V3 batch 64, LLaMA2-7B batch 16 and 32, respectively in Figure~\ref{fig:deberta64},~\ref{fig:llama2b16},~\ref{fig:llama2b32}.

Similar to Figure~\ref{fig:deberta} in~\ref{exp:llm:nlu}, we can observe consistency in requirements of adapters in regard to computational resources. LoCO and OFT are more reliable and robust compared to BOFT and HRA when the models are fine-tuned with more trainable parameters, correspondingly imposed by a large value of adapter configuration.

\begin{table*}[htbp]
    \centering
    \setlength{\tabcolsep}{6pt}
    \resizebox{0.75\textwidth}{!}{\begin{tabular}{l c c c c c c c c}
        \toprule
        \textbf{Hyperparameter} & \textbf{MNLI} & \textbf{SST-2} & \textbf{CoLA} & \textbf{QQP} & \textbf{QNLI} & \textbf{RTE} & \textbf{MRPC} & \textbf{STS-B} \\
        \midrule
        Learning Rate & 1e-4 & 1e-4 & 8e-4 & 2e-4 & 1e-4 & 5e-4 & 5e-4 & 5e-4 \\
        Batch Size & 64 & 32 & 32 & 64 & 64 & 32 & 32 & 32 \\
        Epochs & 10 & 8 & 12 & 10 & 11 & 30 & 30 & 13 \\
        Max Seq Length & 256 & 128 & 64 & 320 & 512 & 320 & 320 & 128 \\
        Eval Interval & 400 & 200 & 100 & 400 & 300 & 100 & 100 & 100 \\
        \bottomrule
        \end{tabular}}
    \caption{LoCO training hyperparameters for each dataset on GLUE benchmark.}
    \label{tab:app_llms:loco_nlu_datasets}
\end{table*}
 
This pattern is more obvious when we consider a small model, small batch sizes, then most of GPU memory is allocated to activation memory of the adapters. We also have another experiments for a more extreme configuration, equivalent to LoCO $n=1, r=32$ on LLaMA2-7B, batch size of $1$, demonstrated in ~Table~\ref{tab:efficiency_compact}.

\subsection{Fine-tuning Diffusion Transformers}
\label{appendix:diffusion}
This section provides additional details regarding our experiments on fine-tuning diffusion transformers for controllable image generation, as discussed in Section~\ref{sec:diffusion}.

\paragraph{Optimizer.} Similar to OminiControl setup, we adopt the Prodigy optimizer, an adaptive learning rate method that automatically adjusts step sizes based on gradient statistics. The optimizer is configured with:
\begin{itemize}
    \item Base learning rate: $\eta = 1$
    \item Weight decay: $\lambda = 0.01$
\end{itemize}

\paragraph*{Visualization examples} Figure~\ref{fig:canny_comparison},~\ref{fig:deblurring_comparison},~\ref{fig:depth_comparison} and~\ref{fig:fill_comparison} provide qualitative comparisons of generated images across different controllable generation tasks, including Canny edge-to-image, deblurring, depth-to-image, and impanting. Our method consistently produces high-quality images that effectively adhere to the provided control signals, demonstrating the efficacy of our method in preserving semantic alignment.

\subsection{Fine-tuning Vision Transformers (ViTs)}
\label{appendix:vit}

We provide the detailed information of VTAB-1k benchmark datasets and results in Table~\ref{tab:vtab_statistics} and Table~\ref{tab:vtab}.
\begin{table*}[h]
	\centering

	\small
	\setlength{\tabcolsep}{2pt}
	\scalebox{1}{
		
\begin{tabular}{l|ccccccc|cccc|cccccccc|cc}

			& \multicolumn{7}{c|}{\textbf{Natural}} & \multicolumn{4}{c|}{\textbf{Specialized}} & \multicolumn{8}{c}{\textbf{Structured}} \\
			Method & \rotatebox{90}{CIFAR-100} & \rotatebox{90}{Caltech101} & \rotatebox{90}{DTD} & \rotatebox{90}{Flowers102} & \rotatebox{90}{Pets} & \rotatebox{90}{SVHN} & \rotatebox{90}{Sun397} & \rotatebox{90}{Patch Camelyon} & \rotatebox{90}{EuroSAT} & \rotatebox{90}{Resisc45} & \rotatebox{90}{Retinopathy} & \rotatebox{90}{Clevr/count} & \rotatebox{90}{Clevr/distance} & \rotatebox{90}{DMLab} & \rotatebox{90}{KITTI/distance} & \rotatebox{90}{dSprites/loc} & \rotatebox{90}{dSprites/ori} & \rotatebox{90}{SmallNORB/azi} & \rotatebox{90}{SmallNORB/ele} & \rotatebox{90}{Mean} & \rotatebox{90}{Params (M)} \\
			\hline\hline
			Full fine-tuning & 68.9 & 87.7 & 64.3 & 97.2 & 86.9 & 87.4 & 38.8 & 79.7 & 95.7 & 84.2 & 73.9 & 56.3 & 58.6 & 41.7 & 65.5 & 57.5 & 46.7 & 25.7 & 29.1 & 65.57 & 85.84 \\
			Linear probing & 63.4 & 85.0 & 63.2 & 97.0 & 86.3 & 36.6 & 51.0 & 78.5 & 87.5 & 68.6 & 74.0 & 34.3 & 30.6 & 33.2 & 55.4 & 12.5 & 20.0 & 9.6 & 19.2 & 52.94 & 0.04 \\
			\hline
			Adapter & 74.1 & 86.1 & 63.2 & 97.7 & 87.0 & 34.6 & 50.8 & 76.3 & 88.0 & 73.1 & 70.5 & 45.7 & 37.4 & 31.2 & 53.2 & 30.3 & 25.4 & 13.8 & 22.1 & 55.82 & 0.27 \\
			VPT & 77.8 & 87.9 & 63.2 & 97.5 & 85.3 & 49.4 & 51.4 & 78.7 & 91.6 & 72.9 & 69.8 & 61.5 & 55.6 & 32.4 & 55.9 & 66.6 & 40.0 & 15.7 & 25.1 & 62.05 & 0.14 \\

			LoRA & \textbf{78.1} & 90.3 & 64.2 & 95.1 & 83.3 & 75.1 & 48.6 & 83.8 & 91.1 & 82.4 & 67.4 & 61.5 & 61.2 & 44.7 & 72.2 & 71.3 & \textbf{61.7} & 30.3 & 54.1 & 65.45 & 0.10 \\
			SSF & 69.0 & 92.6 & \textbf{75.1} & 99.4 & \textbf{91.8 }& 90.2 & 52.9 & 87.4 & 95.2 & 87.4 & 75.5 & 75.9 & 62.3 & 53.3 & 80.6 & 77.3 & 54.9 & 39.2 & \textbf{57.9} & 73.10 & 0.24 \\
\end{tabular}

	}
	\caption{Performance comparison across different methods and datasets}
	\label{tab:vtab}
\end{table*}

\begin{table*}[t]
	\centering
	\caption{Statistics of VTAB-1k benchmark datasets.}
	\label{tab:vtab_statistics}
	\begin{tabular}{l|l|c|c|c|c}
		\toprule
		\textbf{Dataset} & \textbf{Description} & \textbf{Classes} & \textbf{Train size} & \textbf{Val size} & \textbf{Test size} \\
		\midrule
		CIFAR-100 & \multirow{7}{*}{Natural} & 100 & \multirow{7}{*}{800/1,000} & \multirow{7}{*}{200} & 10,000 \\
		Caltech101 & & 102 & & & 6,084 \\
		DTD & & 47 & & & 1,880 \\
		Flowers102 & & 102 & & & 6,149 \\
		Pets & & 37 & & & 3,669 \\
		SVHN & & 10 & & & 26,032 \\
		Sun397 & & 397 & & & 21,750 \\
		\midrule
		Patch Camelyon & \multirow{4}{*}{Specialized} & 2 & \multirow{4}{*}{800/1,000} & \multirow{4}{*}{200} & 32,768 \\
		EuroSAT & & 10 & & & 5,400 \\
		Resisc45 & & 45 & & & 6,300 \\
		Retinopathy & & 5 & & & 42,670 \\
		\midrule
		Clevr/count & \multirow{8}{*}{Structured} & 8 & \multirow{8}{*}{800/1,000} & \multirow{8}{*}{200} & 15,000 \\
		Clevr/distance & & 6 & & & 15,000 \\
		DMLab & & 6 & & & 22,735 \\
		KITTI/distance & & 4 & & & 71,127 \\
		dSprites/location & & 16 & & & 73,728 \\
		dSprites/orientation & & 16 & & & 73,728 \\
		SmallNORB/azimuth & & 18 & & & 12,150 \\
		SmallNORB/elevation & & 9 & & & 12,150 \\
		\bottomrule
	\end{tabular}
\end{table*}

\subsection{Additional Figures}
\label{sec:appendix_additional_images}

\begin{figure*}[h]
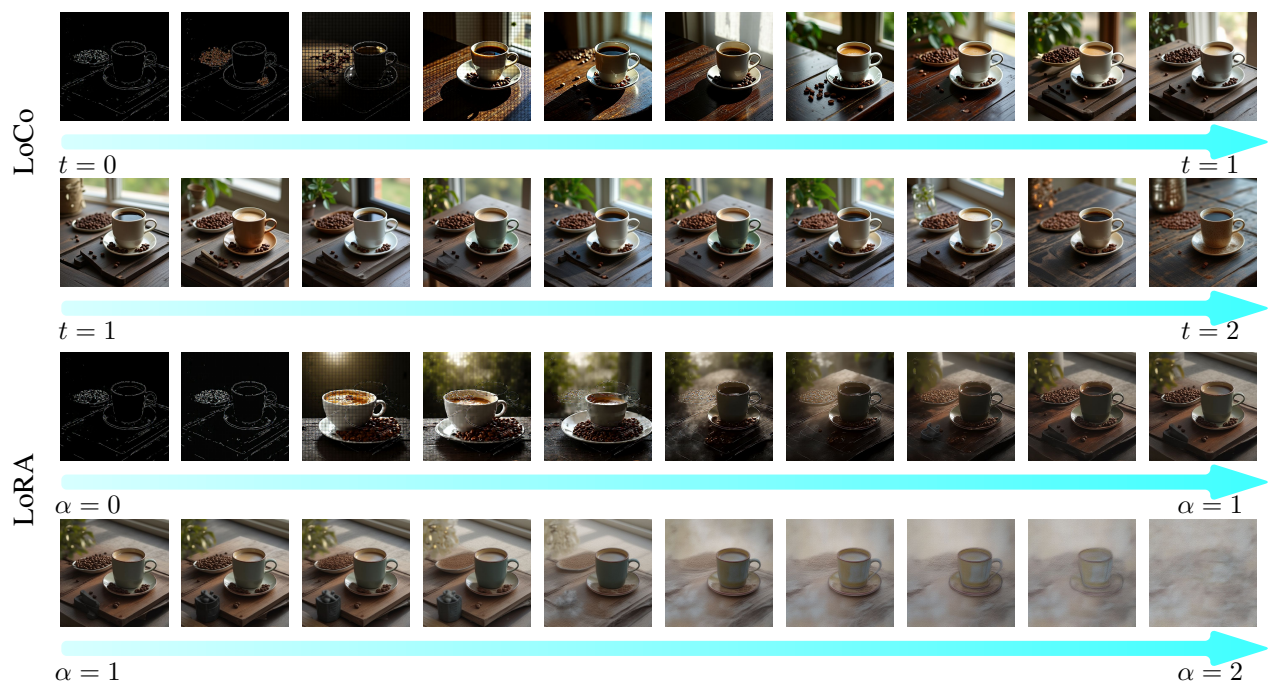

	\centering
\newcommand{\imgspacing}{1.6}
	{
	%
	
}
	\caption{Effect of varying the temperature parameter $t$ for LoCo.}
	\label{fig:varying_t}
\end{figure*}

\begin{figure*}[h]
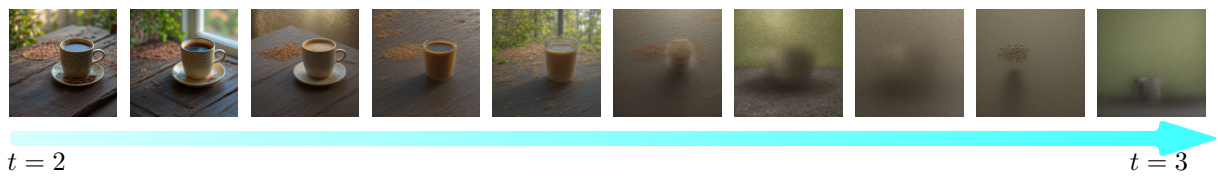

	\centering
\newcommand{\imgspacing}{1.6}
	%
	\caption{Effect of varying the temperature parameter $t$ on image generation quality within the range $[2, 3]$. The generated images start loosing visual quality. }
	\label{fig:varying_t_2_3}
\end{figure*}

\begin{figure*}[t]
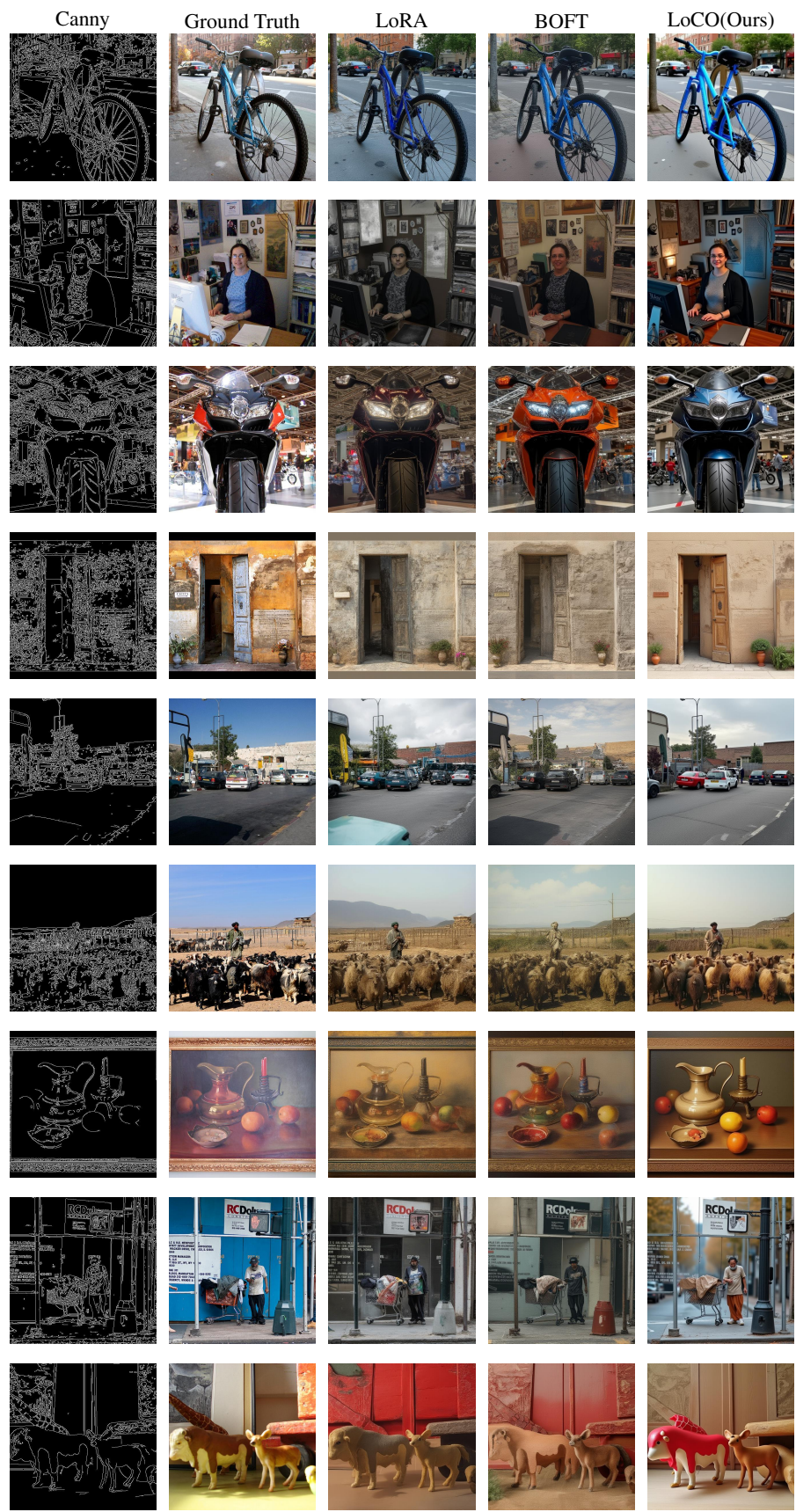

	\centering
	%
	\caption{Qualitative comparison of Canny edge-to-image generation. Columns show: input canny edges, ground truth, and outputs from LoRA, BOFT, and our method (Ours).}
	\label{fig:canny_comparison}
\end{figure*}

\begin{figure*}[t]
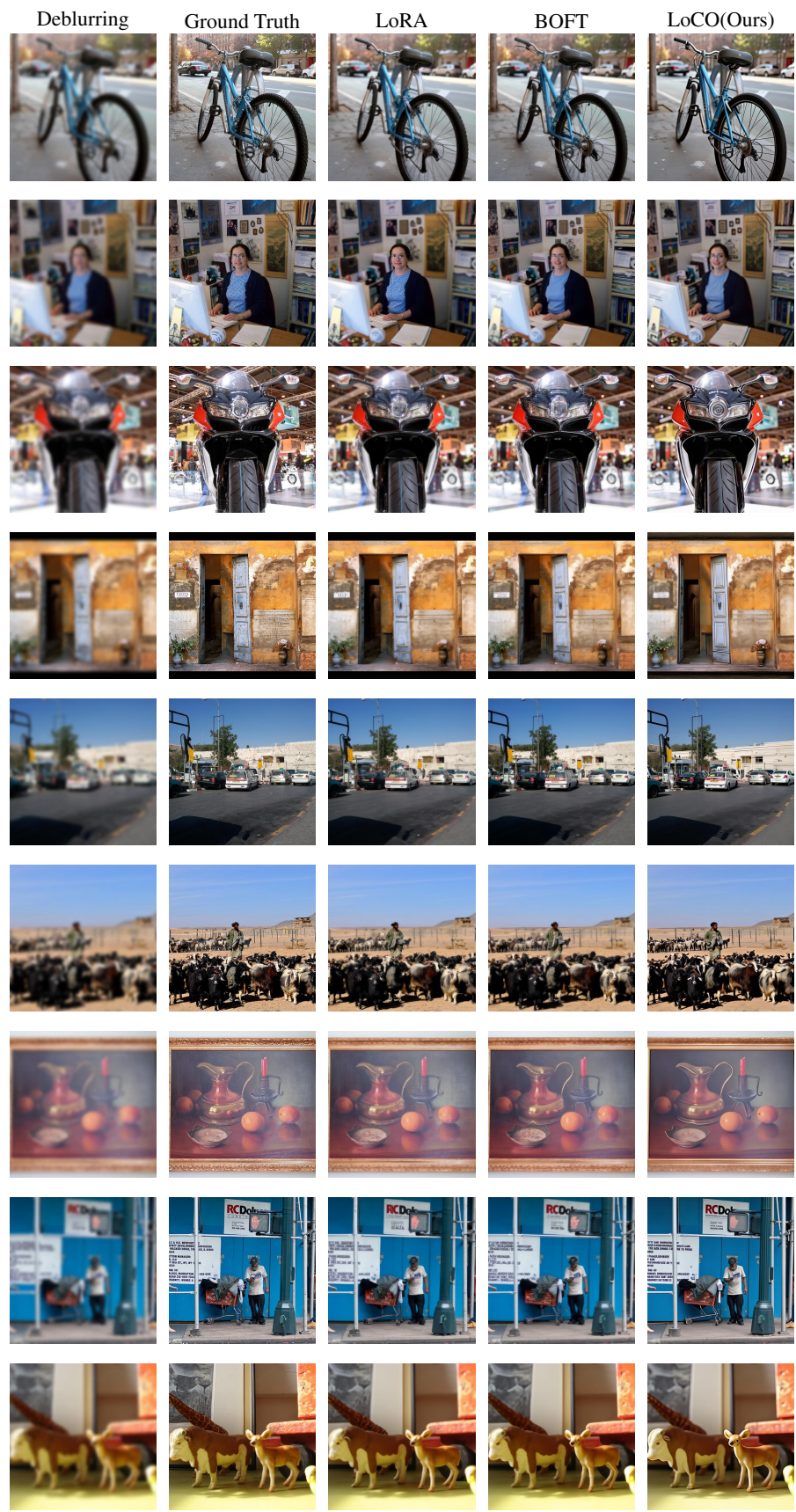

	\centering
	%
	\caption{Qualitative comparison of deblurring image generation. Columns show: blurred input, ground truth, and outputs from LoRA, BOFT, and our method (Ours).}
	\label{fig:deblurring_comparison}
\end{figure*}

\begin{figure*}[t]
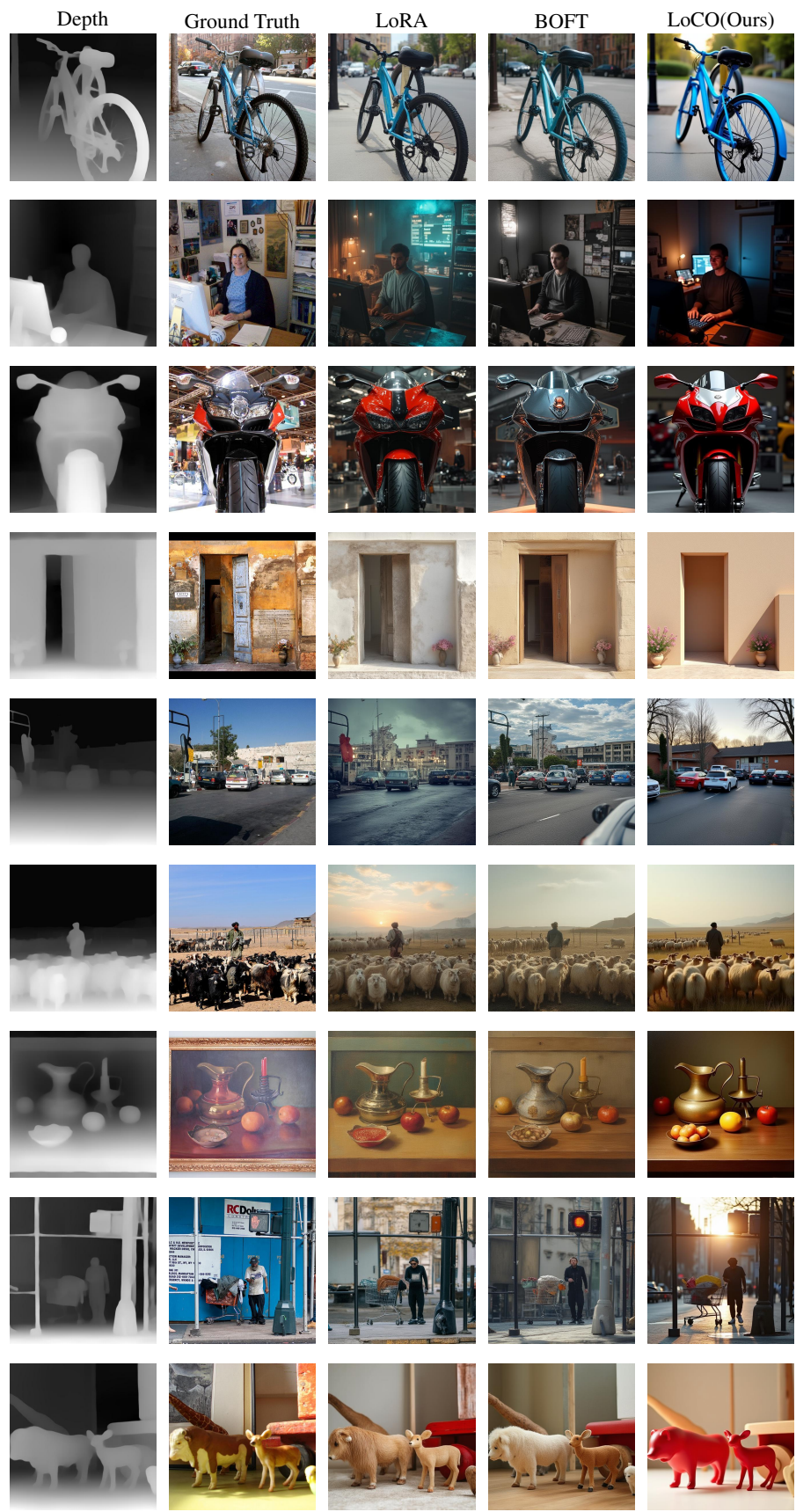

	\centering
	%
	\caption{Qualitative comparison of depth-to-image generation. Columns show: depth map input, ground truth, and outputs from LoRA, BOFT, and our method (Ours).}
	\label{fig:depth_comparison}
\end{figure*}

\begin{figure*}[t]
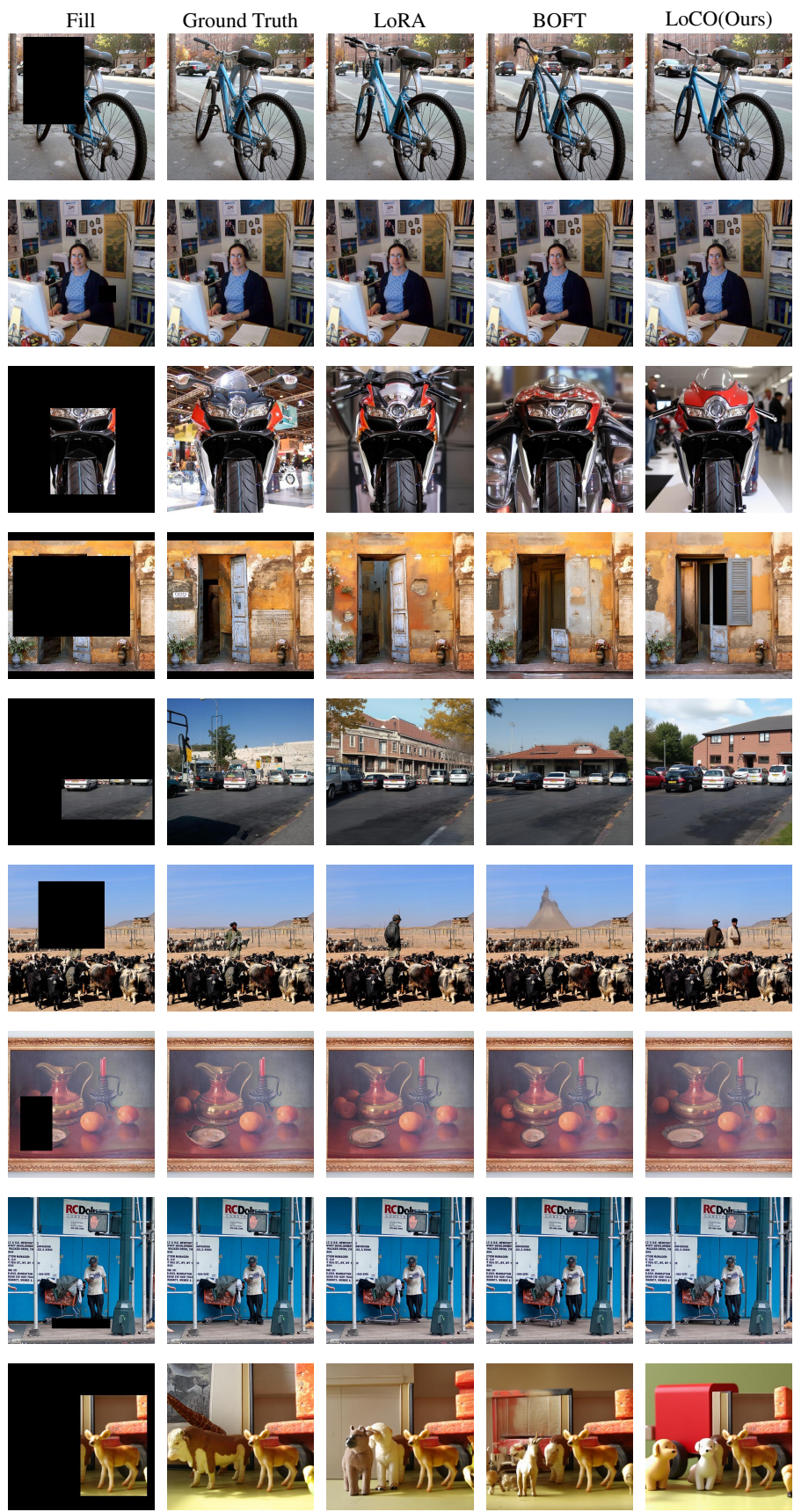

	\centering
	%
	\caption{Qualitative comparison of inpainting (fill) image generation. Columns show: masked input, ground truth, and outputs from LoRA, BOFT, and our method (Ours).}
	\label{fig:fill_comparison}
\end{figure*}

 \else
    \clearpage{}\setcounter{page}{1}

\section{Theoretical Analysis of LoCO}
\subsection{Time complexity among orthogonal approaches}
This section aims to provide a comprehensive comparison of orthogonal fine-tuning methods based on their structural properties, time and space complexity. In particular, we compare LoCO with several contemporary methods, namely: OFT~\cite{qiu2023controlling}, HRA~\cite{yuan2024bridging}, BOFT~\cite{liu2024boft} -- with respect to their theoretical methodology and actual implementations. While all the approaches fundamentally target to construct an orthogonal $d \times d$ transformation matrix to preserve the geometric structure of pretrained representations, they diverge significantly in the parameterization strategies and computational trade-offs. Here we concentrate on the rotation only: $\vz  = \vx \vR , \vx \in \mathbb{R}^{N \times d}$.

\paragraph{OFT:} uses a block-diagonal matrix to represent the orthogonal transformation, each with block size $b \times b$. Applying the Cayley transformation on each block requires $\mathcal{O}(b^3)$, then the total complexity to build a tranformation given input $\vx$ is: $\mathcal{O}(db^2+Ndb)$. Furthermore, the small-block operations on GPUs often suffer from underutilized hardware regarding the large-scale GEneral Matrix to Matrix Multiplication (GEMM) operations. The activation and gradient memory cost is $\mathcal{O}(N \cdot d + d \cdot b)$.

\paragraph{HRA:} parameterizes orthogonal updates as a product of $r$ Householder reflections. Theoretically, by utilizing the WY representation $R = I - VTV^\top$, the input transformation can be optimized to $\mathcal{O}(N \cdot d \cdot r + N \cdot r^2)$, though this entails an additional construction cost of $\mathcal{O}(d \cdot r^2)$ for the matrix $T$. In contrast, LoCO's additive formulation bypasses this $d$-dependent construction overhead. While LoCO incurs a minimal inversion cost of $\mathcal{O}(n \cdot r^3)$, this term is asymptotically smaller than HRA's construction overhead since $r \ll d$. The gradient and activation memory with the WY representation are $\mathcal{O}(r^2)$, and $\mathcal{O}(N \cdot r)$, respectively, resulting to final complexity of $\mathcal{O}(Nr + r^2)$.

More crucially, practical implementations of HRA (e.g., in the \textit{peft} library) often default to a weight-centric approach, sequentially materializing a dense $d \times d$ matrix via $r$ matrix-vector products, resulting in a practical complexity of $\mathcal{O}(r \cdot d^2)$. This materialization further leads to a prohibitive training memory footprint of $\mathcal{O}(d^2 + N \cdot d)$, where $\mathcal{O}(d^2)$ denotes the dense materialization of the reflection matrix within the computational graph, and $\mathcal{O}(Nd)$ is for the input $\vx$ activation.
Conversely, LoCO’s parallelized input-centric design avoids $d \times d$ instantiation entirely, strictly bounding its training memory to $\mathcal{O}(n \cdot r \cdot (N+d))$. This makes LoCO significantly more scalable for high-dimensional models and long-context training where HRA’s quadratic dependency on $d$ becomes a critical bottleneck.

\paragraph{BOFT:} Theoretically, BOFT parameterizes the orthogonal transformation using sparse butterfly factors, suggesting an optimal parameter memory complexity of $\mathcal{O}(m \cdot d \cdot b)$  and a theoretical total work (FLOPs) $\mathcal{O}(m \cdot N \cdot d \cdot b )$ due to the cost of $\mathcal{O}(db)$ for a single matrix multiplication. However, \textit{in practice}, preserving such fine-grained sparsity and sequential permutations leads to scattered memory accesses, which are inefficient on modern GPU architectures optimized for dense matrix multiplications~\cite{dao2022monarch}. Forced by this hardware limitation, the official implementation explicitly abandons sparse execution in favor of a weight-centric strategy. It sequentially accumulates these $m$ factors ($R = B_1 B_2 \dots B_m$) into a dense identity matrix to construct the full $d \times d$ rotation matrix before applying it to the input $x$. Consequently, the practical forward time complexity (Total Work) degrades to $\mathcal{O}(m \cdot d^2 \cdot b + N \cdot d^2)$. Specifically, the term $\mathcal{O}(m \cdot d^2 \cdot b)$ arises from $m$ sequential sparse-dense matrix multiplications to materialize the $d \times d$ matrix, while $\mathcal{O}(N \cdot d^2)$ accounts for the final dense projection of $N$ tokens. 
More critically, this dense materialization inflates the training memory footprint to $\mathcal{O}(m \cdot d^2 + N \cdot d)$. While $\mathcal{O}(N \cdot d)$ denotes the standard activation memory for the input features, the prohibitive $\mathcal{O}(m \cdot d^2)$ term arises from caching the intermediate dense matrices in the computational graph for backward gradients, imposing severe VRAM bottlenecks.

Furthermore, because modern Tensor Cores are strictly optimized for contiguous, large-block dense GEMMs, executing batched operations on exceptionally small block sizes ($b$) suffers from severe memory bandwidth bottlenecks and hardware padding overheads. Consequently, the theoretical computational savings of $b \ll d$ fail to translate into practical latency reductions, rendering the approach computationally intensive in real-world deployments~\cite{dao2022monarch}.

\paragraph{LoCO:} adopts input-centric transformation and circumvents prohibitive $d \times d$ dense matrices. Serial compositional adapters theoretically require a total computational work of $\mathcal{O}(n \cdot N \cdot d \cdot r + n \cdot r^3)$, where $\mathcal{O}(nNdr)$ and $\mathcal{O}(nr^3)$ denotes the cost of matrix-vector multiplication and matrix inversion, respectively (See~\Cref{eq:actual_adapter}). By virtue of full parallelism, LoCO reduces its sequential span to a depth of $\mathcal{O}(Ndr+r^3)$. Regarding activation memory, the static memory for parameter space is $\mathcal{O}(n \cdot d \cdot r)$. In addition, the framework only needs to cache $n$ independent projection tensors of size $N \times r$, without any $d \times d$ intermediate matrices, resulting in a memory footprint of $\mathcal{O}(n \cdot (d+N) \cdot r)$, where $\mathcal{O}(ndr)$ and $\mathcal{O}(nNr)$ represent gradients and activations memory, respectively.

In addition, in the naive parallel input-centric approach, we need to store $n$ intermidate tensors of size $N \times d$. The total memory for activations and gradients becomes: $\mathcal{O}(n \cdot N \cdot d) + \mathcal{O}(n \cdot d \cdot r)$. This significant increase in training overhead is \emph{eliminated} by serial compositional LoCO.

We further summarize the comparison results among 4 orthogonal fine-tuning methods in \Cref{tab:comprehensive_comparison}.

\begin{table*}[t]
\centering
\caption{Comprehensive Complexity and Structural Comparison of Orthogonal Fine-Tuning Methods, with respect to the analysis and notations above. Time complexity is decoupled into Total Work (FLOPs) and Sequential Span (Critical path length -- Latency). For sequential methods (HRA, BOFT), Span scales with the number of components, whereas LoCO achieve a constant Span relative to $n$ due its parallel additive structure. For HRA, we include WY-representation as the theoretically optimal strategy.}
\label{tab:comprehensive_comparison}
\renewcommand{\arraystretch}{1.8}
\resizebox{\textwidth}{!}{\begin{tabular}{@{} >{\raggedright\arraybackslash}m{1.2cm} >{\raggedright\arraybackslash}m{2.2cm} l l l l >{\raggedright\arraybackslash}m{5.7cm} @{}}
\toprule
\textbf{Method} & \textbf{Strategy} & \textbf{\#Params} & \textbf{Total Work (FLOPs)} & \textbf{Sequential Span} & \makecell[l]{\textbf{Memory} \\ \textbf{(Act. + Grad.)}} & \textbf{Remarks} \\ \midrule

\textbf{OFT} & Input-centric & $\mathcal{O}(db)$ & $\mathcal{O}(Ndb + db^2)$ & $\mathcal{O}(Ndb + db^2)$ & $\mathcal{O}(Nd + db)$ & Sparse block-diagonal; single-layer; local subspace interaction only. \\ \cmidrule(lr){1-7}

\textbf{BOFT} & \makecell[l]{Weight-centric \\ (Practical)} & $\mathcal{O}(mdb)$ & $\mathcal{O}(md^2b + Nd^2)$ & $\mathcal{O}(md^2b + Nd^2)$ & $\mathcal{O}(md^2 + Nd)$ & Sequential chain; $d^2$ bottleneck during dense materialization, negating theoretical FLOPs. \\ \cmidrule(lr){1-7}

\textbf{HRA} & WY-Refined & $\mathcal{O}(dr)$ & $\mathcal{O}(Ndr + dr^2 + Nr^2)$ & $\mathcal{O}(rd + Ndr + r^2)$ & $\mathcal{O}(Nr + dr)$ & High-depth Householder chain due to vector-based reflection; In their practical execution, sequential transformations impose $\mathcal{O}(d^2)$ overhead. \\ \cmidrule(lr){1-7}

\textbf{LoCO (ours)} & \makecell[l]{{Parallel} \\ {Input-centric}} & ${\mathcal{O}(ndr)}$ & $\mathcal{O}(nNdr + nr^3)$ & $\mathbf{\mathcal{O}(Ndr + r^3)}$ & ${\mathcal{O}(nr(N + d))}$ & Orthogonal approximation; parallel additive structure; shallow depth $\left( n \in \{1, 2, 4\} \right)$. \\ \bottomrule
\end{tabular}}
\end{table*}

\subsection{Theoretical analysis on approximation error}
\label{sec:app_error}
In this section, we provide detailed proofs of the upper bound of approximation error in \Cref{sec:model_method}.

\begin{theorem}[Upper Bound of Approximation Error]
    \label{thm:upper_bound_dev}
    Let $\rmZ_i = \Delta_i = 2\rmX_i(\rmI-\rmY_i^\top \rmX_i)^{-1}\rmY_i^\top$ denote the update components in \Cref{sec:model_method}. 
    
    Consider the generic expansion of a rotation chain $\mathbf{R} = \prod_{i=1}^{n} (\rmI + \rmZ_i)$, where $(\rmI + \rmZ_i) \in \textnormal{SO}(d)$. Its first-order parallel approximation is defined as $\tilde{\rmR} = \rmI + \sum_{i=1}^{n} \rmZ_i$. 
    
    If the Frobenius norm of each component is bounded by $\left\lVert \rmZ_i \right\rVert_F \leq \gamma$, then the orthogonal deviation of the parallel approximation is strictly bounded by:
    \begin{equation}
        \left\lVert \rmI - \tilde{\rmR}^\top \tilde{\rmR} \right\rVert_F \leq n(n-1) \gamma^2
    \end{equation}
\end{theorem}

\begin{proof}
    Consider the orthogonal deviation: $\Delta \rmR_O = \rmI - \tilde{\rmR}^\top \tilde{\rmR}$.
    
Since $\rmR_i = \rmI + \rmZ_i$ is orthogonal, we have $(\rmI + \rmZ_i)^\top (\rmI + \rmZ_i) = \rmI$, which yields:
    \begin{equation}
        \label{eq:zi_identity}
        \rmZ_i^\top + \rmZ_i = -\rmZ_i^\top \rmZ_i
    \end{equation}
    
Let $\rmS = \sum_{i=1}^n \rmZ_i$. The Gram matrix of the approximation expands as:
    \begin{gather}
        \tilde{\rmR}^\top \tilde{\rmR} = (\rmI +\rmS^\top)(\rmI +\rmS) = \rmI + \rmS^\top + \rmS + \rmS^\top \rmS \notag \\[1ex]
        \Rightarrow \Delta \rmR_O = -[\rmS^\top \rmS + (\rmS^\top + \rmS)]
    \end{gather}
    
By substituting \Cref{eq:zi_identity} into the linear term $(\rmS^\top + \rmS)$, we obtain:
    \begin{align}
        \Delta \rmR_O &= - \left( \sum_{i=1}^n \sum_{j=1}^n \rmZ_i^\top \rmZ_j - \sum_{i=1}^n\rmZ_i^\top \rmZ_i \right) \notag\\[1ex]
         &= - \sum_{i \neq j} \rmZ_i^\top \rmZ_j
    \end{align}
    
Applying the triangle inequality and the submultiplicativity of the Frobenius norm ($\lVert \rmA\rmB \rVert_F \leq \lVert \rmA \rVert_F \lVert \rmB \rVert_F$), we have:
    \begin{align}
        \left\lVert \Delta \rmR_O \right\rVert_F &= \left\lVert \sum_{i \neq j} \rmZ_i^\top \rmZ_j \right\rVert_F \leq \sum_{i \neq j} \left\lVert  \rmZ_i^\top \rmZ_j \right\rVert_F\notag \\[1ex]
        \Rightarrow \left\lVert \Delta \rmR_O \right\rVert_F &\leq \sum_{i \neq j} \left\lVert \rmZ_i^\top \right\rVert_F \left\lVert \rmZ_j \right\rVert_F 
    \end{align}
    
Since $\left\lVert \rmZ_i^\top \right\rVert_F = \left\lVert \rmZ_i \right\rVert_F \leq \gamma$, and there are exactly $n(n-1)$ terms where $i \neq j$, we conclude:
    \begin{equation}
        \left\lVert \rmI - \tilde{\rmR}^\top \tilde{\rmR} \right\rVert_F \leq n(n-1) \gamma^2
    \end{equation}
    This completes the proof.
\end{proof}
\Cref{thm:upper_bound_dev} shows the upper bound of deviation error in orthogonal approximation, providing the formal justification for LoCO's parallelized architecture. It mainly includes these properties below to advocate the effectiveness of the parallel mechanism:
\begin{itemize}
	\item \textbf{Mathematical Validation of Parallelism:} The theorem justifies replacing sequential rotation products with a parallel additive structure by proving the error remains bounded.
	\item \textbf{Guaranteed Orthogonality Approximation:} The upper bound $n^2\gamma^2$ ensures ``quasi-orthogonality'': for small update magnitude $\gamma$, the orthogonal properties are preserved.
	\item \textbf{Optimized Design Space:} The quadratic relationship with $n$ suggests that for small $n$, the deviation is minimal, providing a theoretical "safe zone" for the shallow configurations -- $n \in \{1,2,4\}$ -- of the LoCO architecture.
\end{itemize}

In addition, we provide empirical measurements of orthogonality deviation, and Frobenius norm of $\rmZ_i$ from the fine-tuned checkpoints in \Cref{fig:orthogonal_deviation}. The orthogonality deviation (~\Cref{fig:orthogonal_deviation}a) remains relatively small across all layers, and the perturbation norm $\|\Delta_i\|_F$ (~\Cref{fig:orthogonal_deviation}b) confirms that learned parameters stay within the regime where the approximation is valid. 

\begin{figure}[h]
	\centering
	\begin{tikzpicture}
	\node at (0,0) {\includegraphics[width=0.18\textwidth]{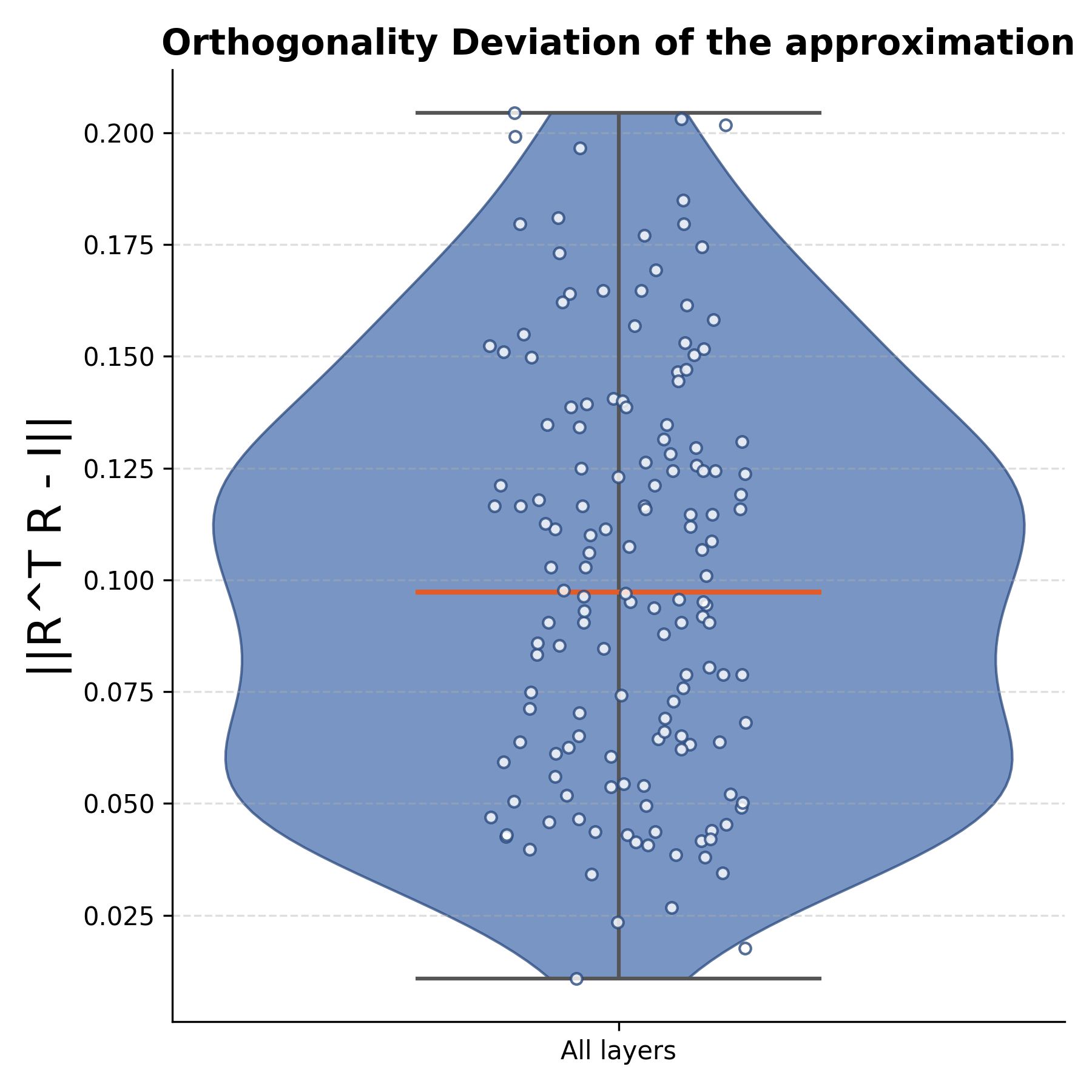}};
	\node at (4,0) {\includegraphics[width=0.18\textwidth]{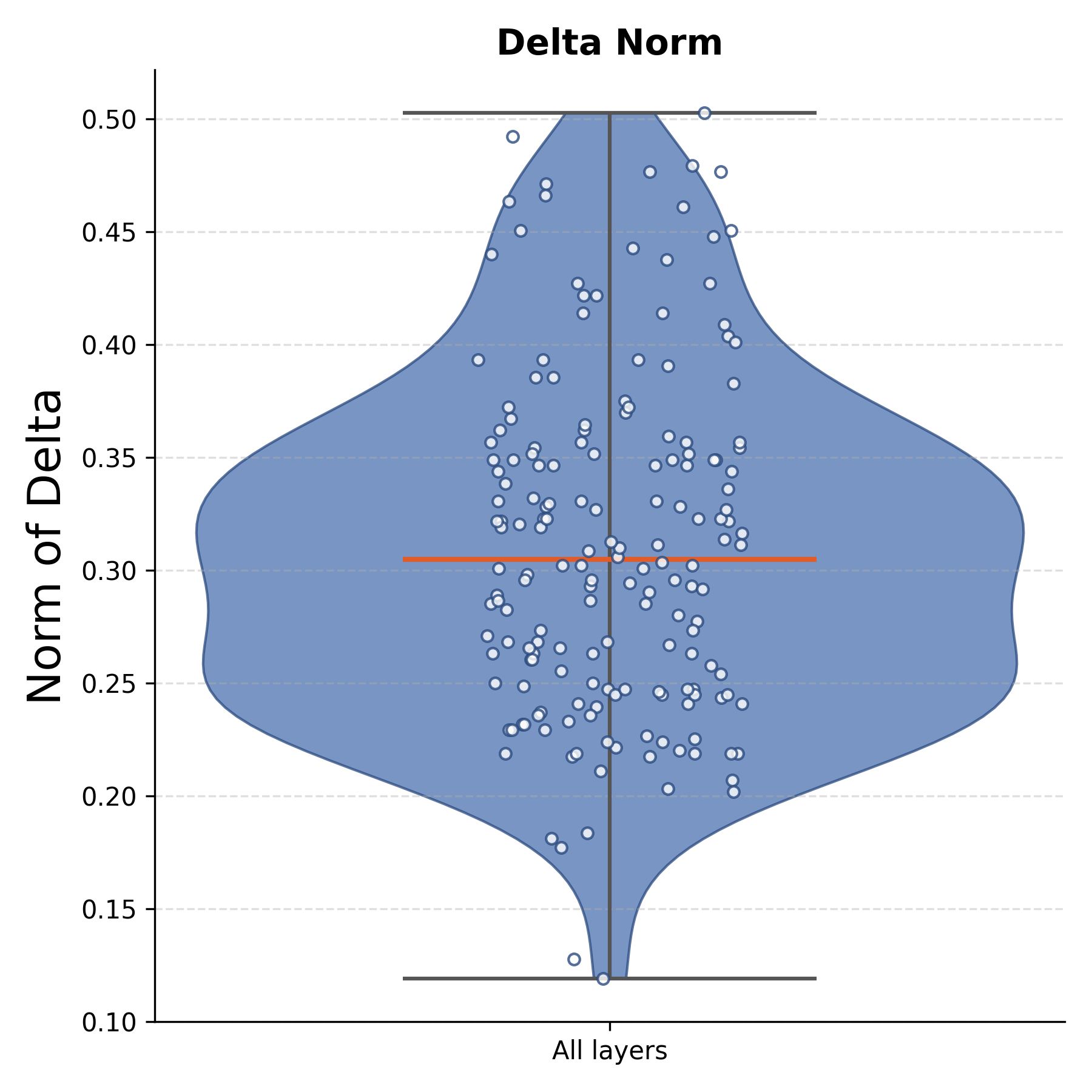}};
	\node at (0, -1.7) {(a)};
	\node at (4, -1.7) {(b)};
	\end{tikzpicture}
	\caption{(a) Orthogonality deviation $\|\tilde{\mathbf{R}}^\top \tilde{\mathbf{R}} - \mathbf{I}\|_F$ extracted from a trained diffusion model. (b) Distribution of perturbation norms $\left\lVert \rmZ_i \right\rVert_F$ or $\left\lVert \Delta_i \right\rVert_F$ from the same checkpoint. Here, $\tilde{\mathbf{R}} \in \mathbb{R}^{3072 \times 3072}$ is a high-dimensional matrix. The deviation norm remains low (mean $\approx 0.1$), indicating that the approximation roughly preserves vector norms under transformation.}
	\label{fig:orthogonal_deviation}
\end{figure}

\section{Experimental Details}

\subsection{Fine-tuning Large language models}
In this section regarding LLMs, we further describe experimental details and illustrate additional outputs on language tasks at~\ref{sec_llm}. They include experiments on natural language understanding/generation, computation efficiency evaluation, and ablation study on the impact of LoCO configurations.

\subsubsection{Natural Language Understanding}
\label{appendix:nlu_exps}

In regard to our experiments on the GLUE benchmark, we only tune the learning rate, and the number of training epochs. We adapt pre-trained DeBERTA-V3~\cite{he2021deberta} as the base model, then apply orthogonal methods to every linear module in each transformer block while freezing the pre-trained weights.

Particularly, we choose 8 datasets from the benchmark, including CoLA, SST-2, MRPC, m-MNLI, QNLI, RTE, STS-B, QQP.

For each specific task, the pooling layers and classification head are trained from scratch with individual classification learning rate. The settings for each dataset and other common factors are provided in Table~\ref{tab:app_llms:loco_nlu_common} and Table~\ref{tab:app_llms:loco_nlu_datasets}, respectively.

\begin{table}[H]
    \centering
    \resizebox{0.70\columnwidth}{!}{\begin{tabular}{l c}
        \toprule
        \textbf{Hyperparameter} & \textbf{Value} \\
        \midrule
        Optimizer & AdamW \\
        Optimizer ($\beta_1, \beta_2$) & (0.9, 0.99) \\
        Learning rate scheduler & Cosine \\
        Classifier learning rate & 2e-3 \\
        LoCO dropout & 0.0 \\
        \bottomrule
        \end{tabular}}
    \caption{Common hyperparameters shared across all GLUE tasks.}
    \label{tab:app_llms:loco_nlu_common}
\end{table} 
In order to do replication of HRA (taking regularization coefficient $\lambda{=}0$) and OFT/BOFT, we follow the settings in their papers~\cite{qiu2023controlling,liu2024boft}, and use a single learning rate value for all trainable parameters. The corresponding configurations are also mentioned as a part of Table~\ref{tab:glue_deberta}.

\subsubsection {Evaluation on mathematical reasoning tasks}
\label{appendix:nlg_math}

In the fine-tuning experiments on the LLaMA2 model, we fix the max sequence length as 512, which is sufficient for these tasks.
Following BOFT~\cite{liu2024boft}, we adapt the LLaMA2-7B model~\cite{touvron2023llama2openfoundation} on the first 512 tokens of MetaMath-40K~\cite{yu2023metamath}. The details are given in Table~\ref{tab:app_llms:metamath}

For the ablation study of different LoCO configurations in~\ref{exp:llm_abation} on LLMs, we only change combinations of $n,r$ and use the same values for other training/inference factors to ensure a fair comparison.
\begin{table}[H]
    \centering
    \setlength{\tabcolsep}{6pt}
    \resizebox{0.65\columnwidth}{!}{\begin{tabular}{l c}
        \toprule
        \textbf{Hyperparameter} & \textbf{Value} \\
        \midrule
        \multicolumn{2}{c}{\textbf{\textit{Training settings}}} \\
        \midrule
        Learning rate & 1e-3 \\
        Optimizer & AdamW \\
        Optimizer ($\beta_1, \beta_2$) & (0.9, 0.99) \\
        Learning rate scheduler & Cosine \\
        LoCO dropout & 0.1 \\
        Batch size & 32 \\
        Epoch & 3 \\
        Applied modules & \textit{Q, V} \\
        \midrule
        \multicolumn{2}{c}{\textbf{\textit{Inference settings}}} \\
        \midrule
        LLM max sequence length & default \\
        GSM new tokens & 1024 \\
        MATH new tokens & 1408 \\
        \bottomrule
        \end{tabular}}
    \caption{Hyperparameters for mathematical reasoning tasks.}
    \label{tab:app_llms:metamath}
\end{table} 
\subsubsection{Ablation study of LoCO configurations on LLMs}
\label{exp:llm_abation}
In this section, we explicitly evaluate influence of LoCo configurations on the final results, concentrating the two aspects: how rank $r$ of skew matrices affect LoCO performace, and what is the trade-off of parallel transformations $n$ in~\ref{model:Compositional Chain-of-Rotation}

\begin{table}[H]
    \centering
    \setlength{\tabcolsep}{6pt}
    \resizebox{0.80\columnwidth}{!}{
    \begin{tabular}{c c c c}
        \toprule
        \textbf{Params (\%)} & \textbf{Config ($n, r$)} & \textbf{GSM8K} & \textbf{MATH} \\
        \midrule
0.016 & $n{=}1, r{=}2$ & 41.75 & 5.80 \\
        0.031 & $n{=}1, r{=}4$ & 45.44 & 6.45 \\
        0.062 & $n{=}1, r{=}8$ & 48.40 & 7.37 \\
        \midrule
\multirow{3}{*}{0.124} 
              & $n{=}1, r{=}16$ & 50.19 & \textbf{8.40} \\
              & $n{=}2, r{=}8$  & \textbf{50.26} & 8.12 \\
              & $n{=}4, r{=}4$  & 49.46 & 7.85 \\
        \bottomrule
    \end{tabular}
    }
    \caption{Ablation study on LoCO configurations.  The first block shows the scaling effect by increasing rank $r$. The second block compares different factorization strategies ($n, r$) under a fixed parameter budget.}
    \label{tab:ablation_llm_nr}
\end{table} 
For the first question, we follow the experiment setting in~\ref{exp:llm_math}, vary the rank $r$. For the second, given a parameter budget equivalent to this experiment, $\approx$ 0.124\%, we apply and compare different configurations of $(n,r)$. The averaged results are reported in Table~\ref{tab:ablation_llm_nr}.

Our results indicate that increasing the rank $r$ consistently yields better performance. Furthermore, under a fixed parameter budget, approximating the transformation via parallel computation for low-rank skew matrices maintains competitive efficacy, demonstrating that computation efficiency benefits do not come at the cost of considerable degradation in performance.

\subsubsection{Computational Efficiency}
\label{appendix:computation_efficiency}

\begin{figure}[htbp]
	\centering
	\includegraphics[width=0.95\linewidth]{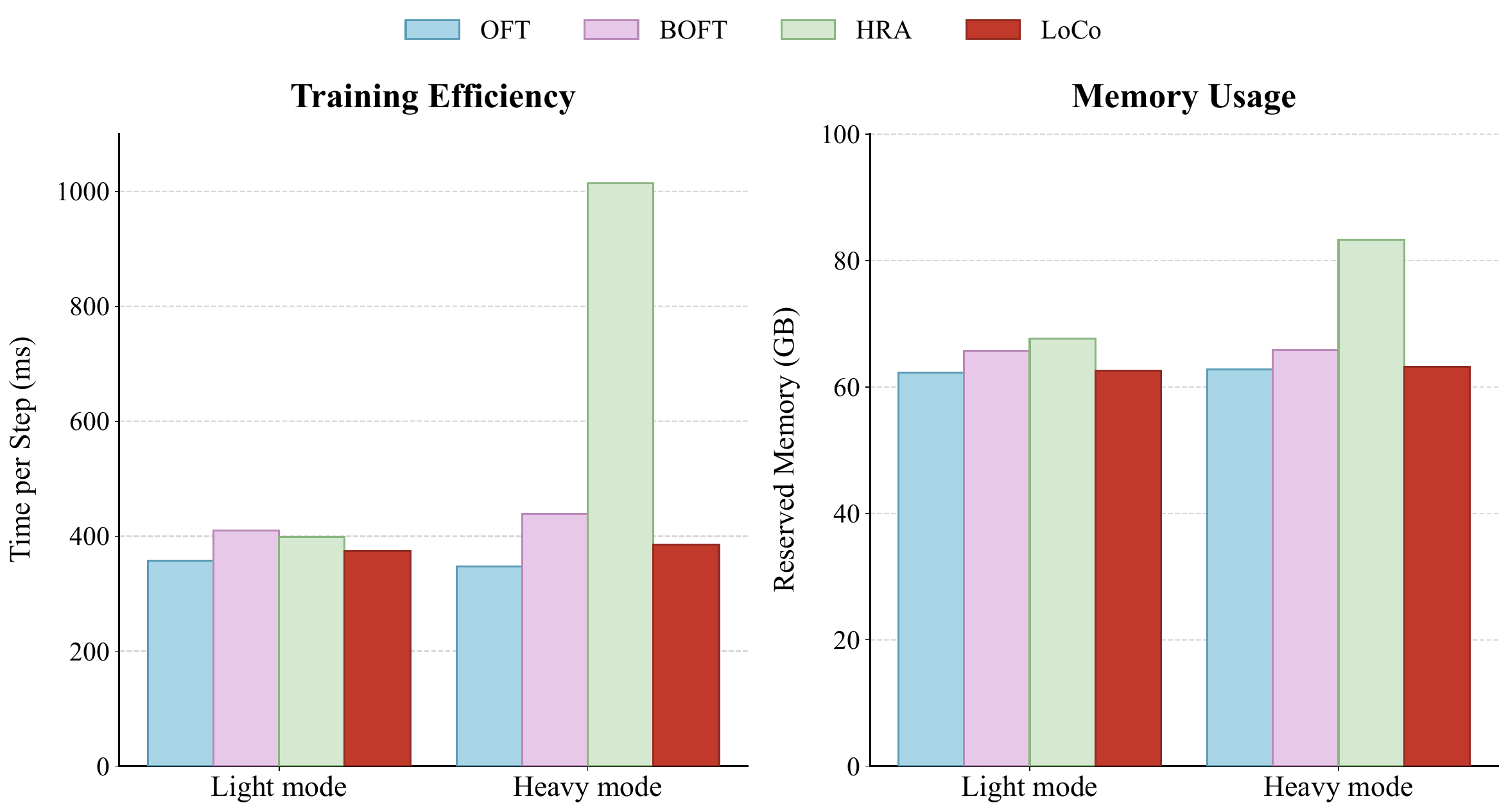}
	\caption{\label{fig:deberta64}Training efficiency comparison on DeBERTA-V3, batch size of 64.}
	
	\vspace{0.4cm}
	\centering
	\includegraphics[width=0.95\linewidth]{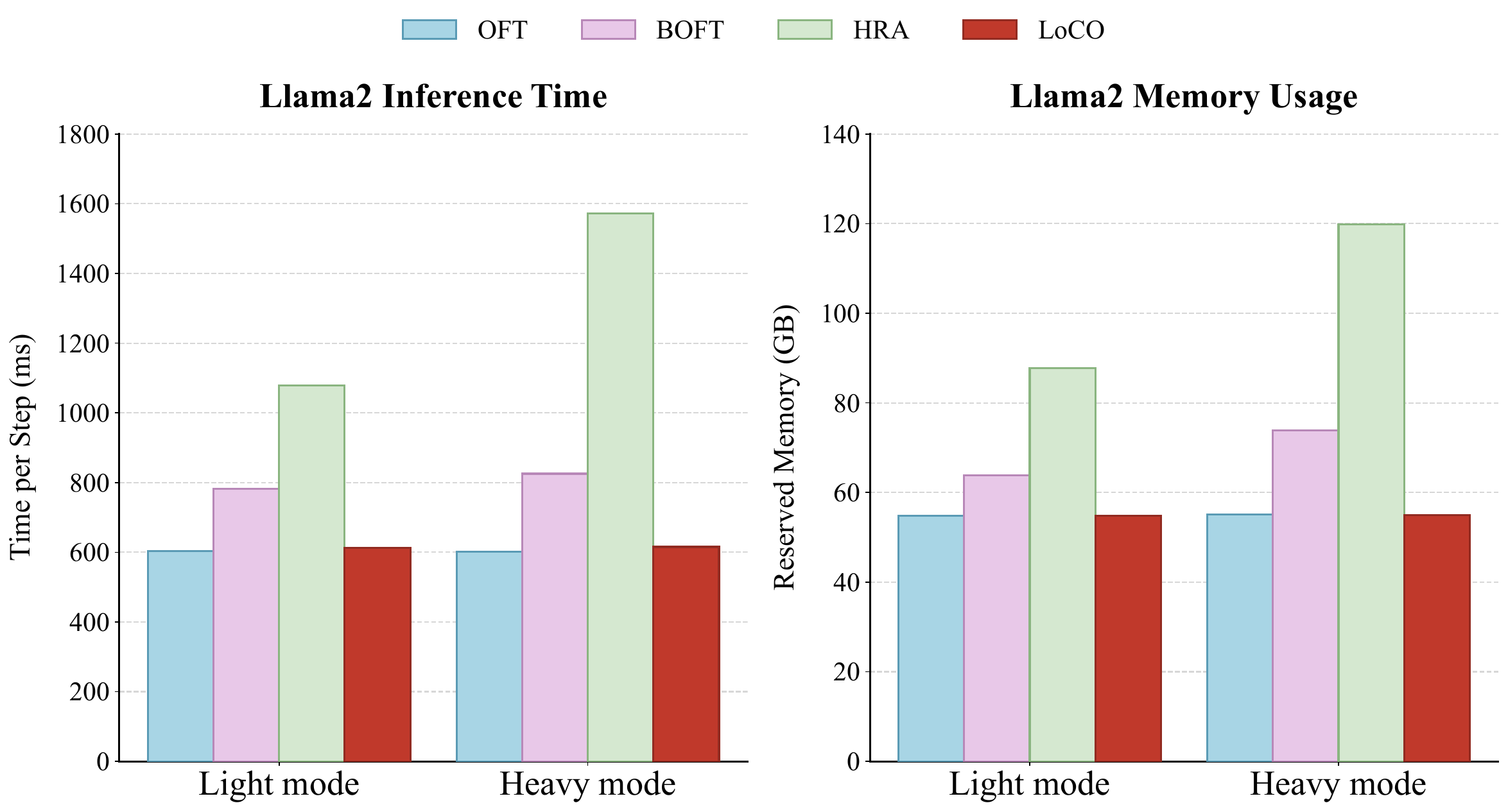}
	\caption{\label{fig:llama2b16}Training efficiency comparison on LLaMA2-7B batch size of 16.}

	\vspace{0.4cm}
	\centering
	\includegraphics[width=0.95\linewidth]{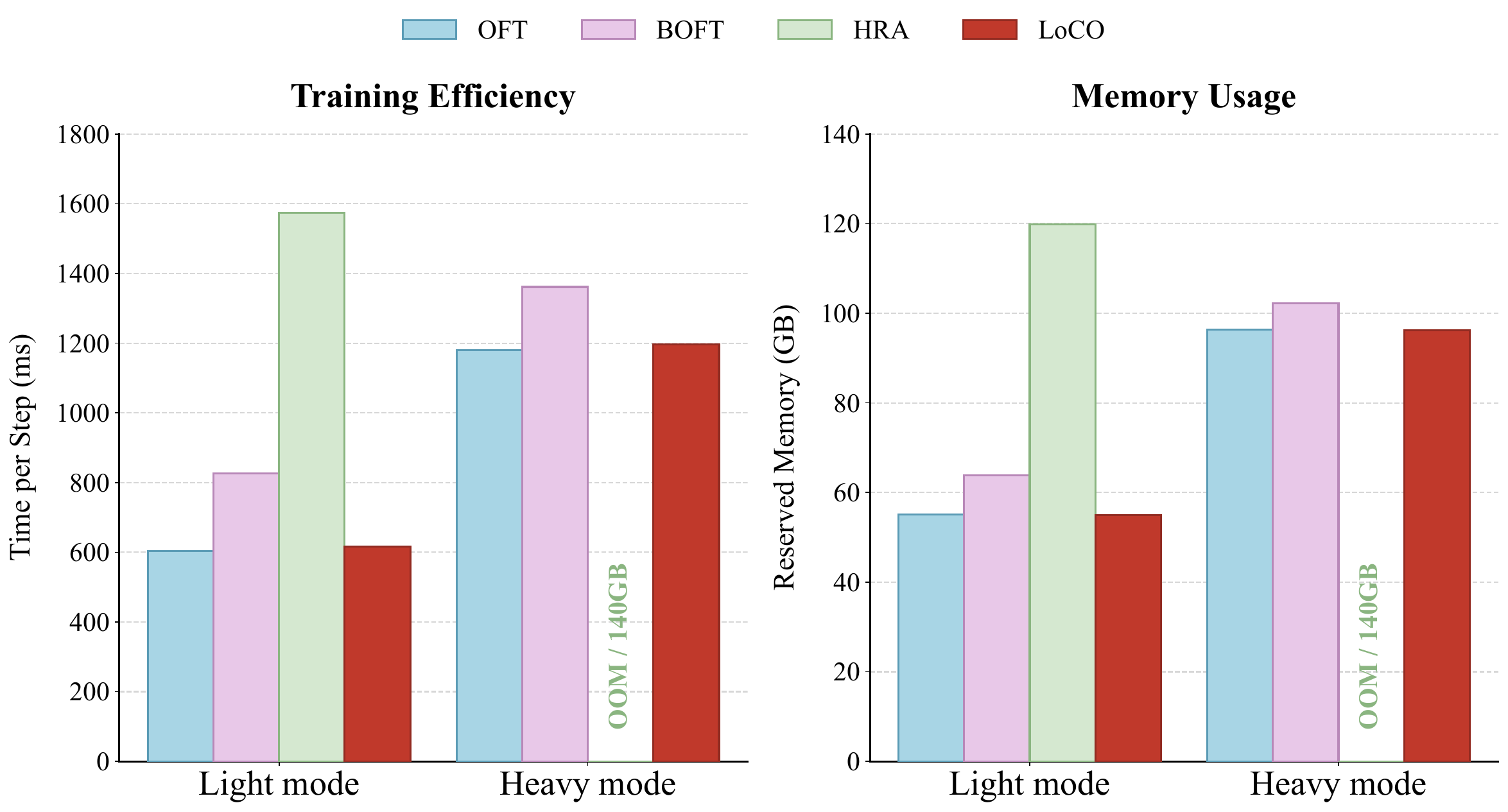}
	\caption{\label{fig:llama2b32} Efficiency comparison during training time on LLaMA2-7B batch size of 32. HRA suffers from Out of memory (OOM) issue.}
\end{figure}

To compare computational efficiency, we align the configurations of all methods to match the number of trainable parameters. We define two settings: \textit{light mode} and \textit{heavy mode}, as computational costs vary significantly depending on specific hyperparameter sets. For instance, the rank $r$ impacts the matrix inversion cost, while the number of Householder reflections determines the recursion depth in HRA.

In each mode, we test with two batch size values to get a comprehensive conclusion. For Natural language understanding tasks we use the QNLI dataset~\cite{rajpurkar2016squad} on the DeBERTA-V3~\cite{he2021deberta}, batch sizes of 16 and 64, with \textit{torch.compile()}; for Natural language generation tasks, we use the MetaMath-40K~\cite{yu2023metamath} dataset on the LLaMA2~\cite{touvron2023llama2openfoundation}, batch sizes of 16 and 32, without \textit{torch.compile()}. These models and datasets are similar to the setting of the primary experiments.

For each model, the detailed descriptions for \textit{light mode} and \textit{heavy mode} are provided in Table~\ref{tab:app_llms:eff_deberta} and Table~\ref{tab:app_llms:eff_llama}. All experiments are conducted on a single NVIDIA H200 GPU with 140GB VRAM.

\begin{table}[h]
    \centering
    \setlength{\tabcolsep}{6pt}
    \resizebox{0.85\columnwidth}{!}{\begin{tabular}{l l l c}
        \toprule
        \textbf{Setting} & \textbf{Method} & \textbf{Config} & \textbf{\#Params (M)} \\
        \midrule
        \multirow{4}{*}{Light mode}
          & OFT  & $b{=}16$      & 0.622 \\
          & BOFT & $m{=}2, b{=}4$  & 0.746 \\
          & HRA  & $r{=}8$       & 0.663 \\
          & LoCo & $n{=}1, r{=}4$  & 0.663 \\
        \midrule
        \multirow{4}{*}{Heavy mode}
          & OFT  & $b{=}64$      & 2.612 \\
          & BOFT & $m{=}2, b{=}1$  & 2.737 \\
          & HRA  & $r{=}32$      & 2.654 \\
          & LoCo & $n{=}1, r{=}16$ & 2.654 \\
        \bottomrule
        \end{tabular}}
    \caption{Detailed configurations for efficiency comparison on DeBERTa-v3.}
    \label{tab:app_llms:eff_deberta}
\end{table}

\begin{table}[h]
    \centering
    \setlength{\tabcolsep}{6pt}
    \resizebox{0.85\columnwidth}{!}{\begin{tabular}{l l l c}
        \toprule
        \textbf{Setting} & \textbf{Method} & \textbf{Config} & \textbf{\#Params (M)} \\
        \midrule
        \multirow{4}{*}{Light mode} 
          & OFT  & $b{=}32$      & 4.063 \\
          & BOFT & $m{=}2, b{=}8$  & 4.456 \\
          & HRA  & $r{=}16$      & 4.194 \\
          & LoCo & $n{=}1, r{=}8$  & 4.194 \\
        \midrule
        \multirow{4}{*}{Heavy mode} 
          & OFT  & $b{=}64$      & 8.257 \\
          & BOFT & $m{=}2, b{=}1$  & 8.651 \\
          & HRA  & $r{=}32$      & 8.388 \\
          & LoCo & $n{=}1, r{=}16$ & 8.388 \\
        \bottomrule
        \end{tabular}}
    \caption{Detailed configurations for efficiency comparison on LLaMA2-7B.}
    \label{tab:app_llms:eff_llama}
\end{table} 
\begin{table}[ht]
    \centering
    \resizebox{0.90\columnwidth}{!}{\begin{tabular}{l c c c c}
    \toprule
    \textbf{Method} & \textbf{Config} & \textbf{\#Params} & \textbf{Time (ms)} & \textbf{Mem (GB)} \\
    \midrule
    OFT & $b{=}128$ & 16.65 M & 151.3 & 15.97 \\
    \midrule
    
    \multirow{3}{*}{BOFT} 
        & $m{=}2,b{=}32$ & \multirow{3}{*}{17.04 M} & 303.7 & 57.93 \\
        & $m{=}4,b{=}16$ &                         & 466.0 & 37.68 \\
        & $m{=}8,b{=}8$  &                         & 879.8 & 98.37 \\
    \midrule
    
    HRA & $r{=}64$ & 16.78 M & \textbf{OOM} & \textbf{OOM} \\
    \midrule
    
    \multirow{3}{*}{\textbf{LoCO}} 
        & $n{=}1,r{=}32$ & \multirow{3}{*}{16.78M} & 173.3 & 15.65 \\
        & $n{=}4,r{=}8$  &                         & 171.0 & 15.71 \\
        & $n{=}2,r{=}16$ &                         & 175.3 & 15.74 \\
    \bottomrule
    \end{tabular}
    }
\caption{Efficiency comparison on an extreme configuration that is equivalent to $r=64$ vectors in $\mathbb{R}^d$. \textbf{OOM}: Out Of Memory.}
    \label{tab:efficiency_compact}
\end{table} 
Moreover, we also illustrate additional results beside Figure~\ref{fig:deberta}, in the setting of DeBERTA-V3 batch 64, LLaMA2-7B batch 16 and 32, respectively in Figure~\ref{fig:deberta64},~\ref{fig:llama2b16},~\ref{fig:llama2b32}.

Similar to Figure~\ref{fig:deberta} in~\ref{exp:llm:nlu}, we can observe consistency in requirements of adapters in regard to computational resources. LoCO and OFT are more reliable and robust compared to BOFT and HRA when the models are fine-tuned with more trainable parameters, correspondingly imposed by a large value of adapter configuration.

\begin{table*}[htbp]
    \centering
    \setlength{\tabcolsep}{6pt}
    \resizebox{0.75\textwidth}{!}{\begin{tabular}{l c c c c c c c c}
        \toprule
        \textbf{Hyperparameter} & \textbf{MNLI} & \textbf{SST-2} & \textbf{CoLA} & \textbf{QQP} & \textbf{QNLI} & \textbf{RTE} & \textbf{MRPC} & \textbf{STS-B} \\
        \midrule
        Learning Rate & 1e-4 & 1e-4 & 8e-4 & 2e-4 & 1e-4 & 5e-4 & 5e-4 & 5e-4 \\
        Batch Size & 64 & 32 & 32 & 64 & 64 & 32 & 32 & 32 \\
        Epochs & 10 & 8 & 12 & 10 & 11 & 30 & 30 & 13 \\
        Max Seq Length & 256 & 128 & 64 & 320 & 512 & 320 & 320 & 128 \\
        Eval Interval & 400 & 200 & 100 & 400 & 300 & 100 & 100 & 100 \\
        \bottomrule
        \end{tabular}}
    \caption{LoCO training hyperparameters for each dataset on GLUE benchmark.}
    \label{tab:app_llms:loco_nlu_datasets}
\end{table*}
 
This pattern is more obvious when we consider a small model, small batch sizes, then most of GPU memory is allocated to activation memory of the adapters. We also have another experiments for a more extreme configuration, equivalent to LoCO $n=1, r=32$ on LLaMA2-7B, batch size of $1$, demonstrated in ~Table~\ref{tab:efficiency_compact}.

\subsection{Fine-tuning Diffusion Transformers}
\label{appendix:diffusion}
This section provides additional details regarding our experiments on fine-tuning diffusion transformers for controllable image generation, as discussed in Section~\ref{sec:diffusion}.

\paragraph{Optimizer.} Similar to OminiControl setup, we adopt the Prodigy optimizer, an adaptive learning rate method that automatically adjusts step sizes based on gradient statistics. The optimizer is configured with:
\begin{itemize}
    \item Base learning rate: $\eta = 1$
    \item Weight decay: $\lambda = 0.01$
\end{itemize}

\paragraph*{Visualization examples} Figure~\ref{fig:canny_comparison},~\ref{fig:deblurring_comparison},~\ref{fig:depth_comparison} and~\ref{fig:fill_comparison} provide qualitative comparisons of generated images across different controllable generation tasks, including Canny edge-to-image, deblurring, depth-to-image, and impanting. Our method consistently produces high-quality images that effectively adhere to the provided control signals, demonstrating the efficacy of our method in preserving semantic alignment.

\subsection{Fine-tuning Vision Transformers (ViTs)}
\label{appendix:vit}

We provide the detailed information of VTAB-1k benchmark datasets and results in Table~\ref{tab:vtab_statistics} and Table~\ref{tab:vtab}.
\begin{table*}[h]
	\centering

	\small
	\setlength{\tabcolsep}{2pt}
	\scalebox{1}{
		
\begin{tabular}{l|ccccccc|cccc|cccccccc|cc}

			& \multicolumn{7}{c|}{\textbf{Natural}} & \multicolumn{4}{c|}{\textbf{Specialized}} & \multicolumn{8}{c}{\textbf{Structured}} \\
			Method & \rotatebox{90}{CIFAR-100} & \rotatebox{90}{Caltech101} & \rotatebox{90}{DTD} & \rotatebox{90}{Flowers102} & \rotatebox{90}{Pets} & \rotatebox{90}{SVHN} & \rotatebox{90}{Sun397} & \rotatebox{90}{Patch Camelyon} & \rotatebox{90}{EuroSAT} & \rotatebox{90}{Resisc45} & \rotatebox{90}{Retinopathy} & \rotatebox{90}{Clevr/count} & \rotatebox{90}{Clevr/distance} & \rotatebox{90}{DMLab} & \rotatebox{90}{KITTI/distance} & \rotatebox{90}{dSprites/loc} & \rotatebox{90}{dSprites/ori} & \rotatebox{90}{SmallNORB/azi} & \rotatebox{90}{SmallNORB/ele} & \rotatebox{90}{Mean} & \rotatebox{90}{Params (M)} \\
			\hline\hline
			Full fine-tuning & 68.9 & 87.7 & 64.3 & 97.2 & 86.9 & 87.4 & 38.8 & 79.7 & 95.7 & 84.2 & 73.9 & 56.3 & 58.6 & 41.7 & 65.5 & 57.5 & 46.7 & 25.7 & 29.1 & 65.57 & 85.84 \\
			Linear probing & 63.4 & 85.0 & 63.2 & 97.0 & 86.3 & 36.6 & 51.0 & 78.5 & 87.5 & 68.6 & 74.0 & 34.3 & 30.6 & 33.2 & 55.4 & 12.5 & 20.0 & 9.6 & 19.2 & 52.94 & 0.04 \\
			\hline
			Adapter & 74.1 & 86.1 & 63.2 & 97.7 & 87.0 & 34.6 & 50.8 & 76.3 & 88.0 & 73.1 & 70.5 & 45.7 & 37.4 & 31.2 & 53.2 & 30.3 & 25.4 & 13.8 & 22.1 & 55.82 & 0.27 \\
			VPT & 77.8 & 87.9 & 63.2 & 97.5 & 85.3 & 49.4 & 51.4 & 78.7 & 91.6 & 72.9 & 69.8 & 61.5 & 55.6 & 32.4 & 55.9 & 66.6 & 40.0 & 15.7 & 25.1 & 62.05 & 0.14 \\

			LoRA & \textbf{78.1} & 90.3 & 64.2 & 95.1 & 83.3 & 75.1 & 48.6 & 83.8 & 91.1 & 82.4 & 67.4 & 61.5 & 61.2 & 44.7 & 72.2 & 71.3 & \textbf{61.7} & 30.3 & 54.1 & 65.45 & 0.10 \\
			SSF & 69.0 & 92.6 & \textbf{75.1} & 99.4 & \textbf{91.8 }& 90.2 & 52.9 & 87.4 & 95.2 & 87.4 & 75.5 & 75.9 & 62.3 & 53.3 & 80.6 & 77.3 & 54.9 & 39.2 & \textbf{57.9} & 73.10 & 0.24 \\
\end{tabular}

	}
	\caption{Performance comparison across different methods and datasets}
	\label{tab:vtab}
\end{table*}

\begin{table*}[t]
	\centering
	\caption{Statistics of VTAB-1k benchmark datasets.}
	\label{tab:vtab_statistics}
	\begin{tabular}{l|l|c|c|c|c}
		\toprule
		\textbf{Dataset} & \textbf{Description} & \textbf{Classes} & \textbf{Train size} & \textbf{Val size} & \textbf{Test size} \\
		\midrule
		CIFAR-100 & \multirow{7}{*}{Natural} & 100 & \multirow{7}{*}{800/1,000} & \multirow{7}{*}{200} & 10,000 \\
		Caltech101 & & 102 & & & 6,084 \\
		DTD & & 47 & & & 1,880 \\
		Flowers102 & & 102 & & & 6,149 \\
		Pets & & 37 & & & 3,669 \\
		SVHN & & 10 & & & 26,032 \\
		Sun397 & & 397 & & & 21,750 \\
		\midrule
		Patch Camelyon & \multirow{4}{*}{Specialized} & 2 & \multirow{4}{*}{800/1,000} & \multirow{4}{*}{200} & 32,768 \\
		EuroSAT & & 10 & & & 5,400 \\
		Resisc45 & & 45 & & & 6,300 \\
		Retinopathy & & 5 & & & 42,670 \\
		\midrule
		Clevr/count & \multirow{8}{*}{Structured} & 8 & \multirow{8}{*}{800/1,000} & \multirow{8}{*}{200} & 15,000 \\
		Clevr/distance & & 6 & & & 15,000 \\
		DMLab & & 6 & & & 22,735 \\
		KITTI/distance & & 4 & & & 71,127 \\
		dSprites/location & & 16 & & & 73,728 \\
		dSprites/orientation & & 16 & & & 73,728 \\
		SmallNORB/azimuth & & 18 & & & 12,150 \\
		SmallNORB/elevation & & 9 & & & 12,150 \\
		\bottomrule
	\end{tabular}
\end{table*}

\subsection{Additional Figures}
\label{sec:appendix_additional_images}

\begin{figure*}[h]
	\centering
\newcommand{\imgspacing}{1.6}
	{
	%
	
}
	\caption{Effect of varying the temperature parameter $t$ for LoCo.}
	\label{fig:varying_t}
\end{figure*}

\begin{figure*}[h]
	\centering
\newcommand{\imgspacing}{1.6}
	%
	\caption{Effect of varying the temperature parameter $t$ on image generation quality within the range $[2, 3]$. The generated images start loosing visual quality. }
	\label{fig:varying_t_2_3}
\end{figure*}

\begin{figure*}[t]
	\centering
	%
	\caption{Qualitative comparison of Canny edge-to-image generation. Columns show: input canny edges, ground truth, and outputs from LoRA, BOFT, and our method (Ours).}
	\label{fig:canny_comparison}
\end{figure*}

\begin{figure*}[t]
	\centering
	%
	\caption{Qualitative comparison of deblurring image generation. Columns show: blurred input, ground truth, and outputs from LoRA, BOFT, and our method (Ours).}
	\label{fig:deblurring_comparison}
\end{figure*}

\begin{figure*}[t]
	\centering
	%
	\caption{Qualitative comparison of depth-to-image generation. Columns show: depth map input, ground truth, and outputs from LoRA, BOFT, and our method (Ours).}
	\label{fig:depth_comparison}
\end{figure*}

\begin{figure*}[t]
	\centering
	%
	\caption{Qualitative comparison of inpainting (fill) image generation. Columns show: masked input, ground truth, and outputs from LoRA, BOFT, and our method (Ours).}
	\label{fig:fill_comparison}
\end{figure*}

\clearpage{}
\fi

\end{document}